%% file: main.tex
\documentclass[sigconf]{acmart}
\AtBeginDocument{%
  }

\setcopyright{acmlicensed}
\copyrightyear{2018}
\acmYear{2018}
\acmDOI{XXXXXXX.XXXXXXX}

\input{math_commands.tex}

\usepackage{hyperref}
\usepackage{url}
\usepackage{graphicx}
\usepackage{amsmath}

\usepackage{amssymb}
\usepackage{enumitem}
\usepackage{bm}
\usepackage{xcolor}
\usepackage{dsfont}
\usepackage{booktabs}
\usepackage{todonotes}
\usepackage{cleveref}
\usepackage{siunitx}
\usepackage{wrapfig}
\usepackage{algorithmic}
\usepackage{multirow}
\usepackage{colortbl}
\usepackage{tcolorbox} 
\usepackage{listings}
\usepackage{subcaption}
\lstdefinestyle{mypython}{
  language=Python,
  basicstyle=\ttfamily\small,
  keywordstyle=\color{blue},
  commentstyle=\color{gray},
  stringstyle=\color{teal},
  showstringspaces=false,
  breaklines=true,
  frame=single,
  captionpos=b
}
\usepackage{algorithm}
\usepackage{algorithmic}

\newtheorem{lemma}{Lemma}
\newtheorem{assumption}{Assumption}

\definecolor{graybckgrnd}{gray}{0.975}

\acmConference[Conference acronym 'XX]{KDD'26}{June 03--05,
  2018}{Woodstock, NY}
\acmISBN{978-1-4503-XXXX-X/2018/06}

\title{Beyond Edge Deletion: A Comprehensive Approach to Counterfactual Explanation in Graph Neural Networks}

\author{Matteo De Sanctis$\text{*}$, Riccardo De Sanctis}
\authornote{Equal contribution}
\affiliation{%
\institution{Sapienza University of Rome}
\city{Rome}
\country{Italy}
}

\author{Stefano Faralli}
\orcid{0000-0003-3684-8815}
\affiliation{%
\institution{Sapienza University of Rome}
\city{Rome}
\country{Italy}
}

\author{Paola Velardi}
\orcid{0000-0003-0884-1499}
\affiliation{%
\institution{ISTC-CNR \& Sapienza University of Rome}
\city{Rome}
\country{Italy}
}
\author{Bardh Prenkaj}
\orcid{0000-0002-2991-2279}
\affiliation{%
 \institution{Technical University of Munich}
 \city{Munich}
 \country{Germany}}

\begin{document}

\begin{abstract}
Graph Neural Networks (GNNs) are increasingly adopted across domains such as molecular biology and social network analysis, yet their black-box nature hinders interpretability and trust. This is especially problematic in high-stakes applications, such as predicting molecule toxicity, drug discovery, or guiding financial fraud detections, where transparent explanations are essential. Counterfactual explanations -- minimal changes that flip a model's prediction -- offer a transparent lens into GNNs' behavior. In this work, we introduce XPlore, a novel technique that significantly broadens the counterfactual search space. It consists of gradient-guided perturbations to adjacency and node feature matrices. Unlike most prior methods, which focus solely on edge deletions, our approach belongs to the growing class of techniques that optimize edge insertions and node-feature perturbations, here jointly performed under a unified gradient-based framework, enabling a richer and more nuanced exploration of counterfactuals. To quantify both structural and semantic fidelity, we introduce a cosine similarity metric for learned graph embeddings that addresses a key limitation of traditional distance-based metrics, and demonstrate that XPlore produces more coherent and minimal counterfactuals. Empirical results on 13 real-world and 5 synthetic benchmarks show up to +56.3\% improvement in validity and +52.8\% in fidelity over state-of-the-art baselines, while retaining competitive runtime.
\end{abstract}
\maketitle

\begin{figure}[!t]
    \centering
    \includegraphics[width=\linewidth]{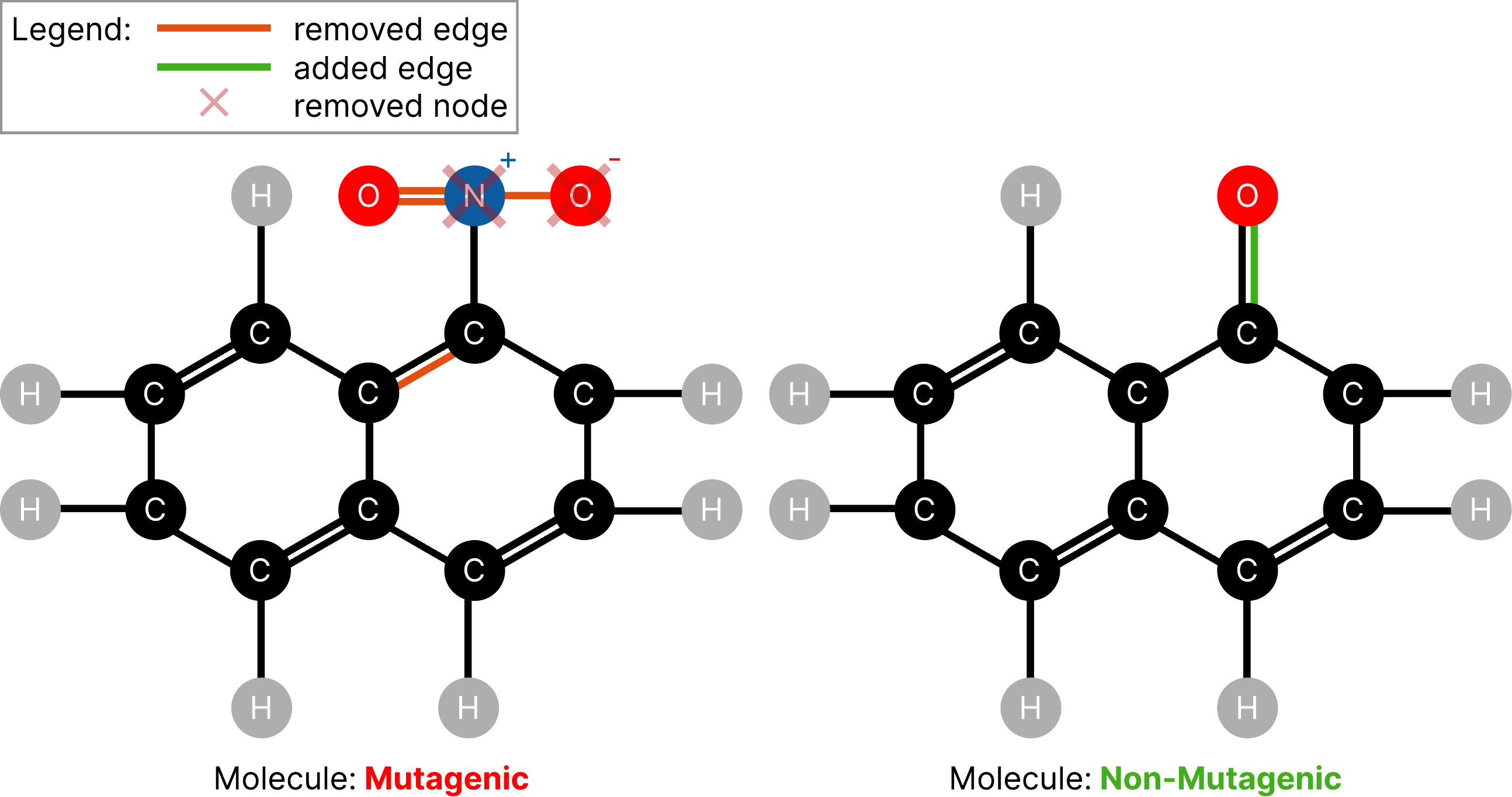}
    \caption{(left) original graph $G$ predicted mutagenic. (right) Counterfactual $G'$ predicted non-mutagenic with highlighted edge additions and removals, as well as potentially node additions to flip the prediction. Note that, although in this example, we illustrate valid chemical properties (see that the valence of atoms is respected), in practice, this might not happen as our method is domain-independent.}
    \label{fig:cf_example}
\end{figure}

\section{Introduction}

Explainability is crucial in fields such as healthcare and finance, where transparent and accountable decisions are necessary \citep{guidotti2018survey}. While deep neural networks are powerful, \citet{petch2021opening} point out that their black-box nature limits interpretability, hindering trust in sensitive applications. White-box models are more interpretable \citep{8882211} and, as \citet{verenich2019predicting} discusses, are often better suited for regulated settings; however, they frequently underperform compared to black-box models on complex, high-dimensional tasks \citep{aragona2021coronna,ding2019effective}.

GNNs have gained significant attention \citep{10.1145/3554981}, but like other deep learning models, they lack interpretability. To mitigate this, post-hoc methods, particularly counterfactual (CF) explanations, aim to reveal how changes to the input affect predictions. Recently, Graph Counterfactual Explainability (GCE) has emerged as a key area \citep{prado2022survey}. As an example, we illustrate a graph-level counterfactual on the MUTAG dataset \citep{debnath1991structure}, where each graph represents a molecule and the label indicates its mutagenicity.\footnote{Note that, in this dataset, the molecules can also be synthetic and not real compounds.} Given an input graph $G$ predicted as mutagenic, we want to produce a valid counterfactual $G'$  that flips the prediction while minimizing a graph distance that jointly accounts for structural (edge) and attribute (node-feature) edits. The resulting counterfactual highlights the smallest set of modifications that changes the model decision--see~\Cref{fig:cf_example}.

GCE techniques can be broadly categorized into search-based, heuristic-based, and learning-based approaches \citep{prado2022survey}. Search-based methods identify counterfactuals by searching within the existing data distribution. Heuristic-based approaches modify the input graph $G$ to create a perturbed version $G'$, ensuring that a prediction model $\Phi$ produces different classifications ($\Phi(G) \neq \Phi(G')$). These methods, however, require carefully designed heuristics, which often rely on domain expertise. For instance, generating chemically valid counterfactuals for molecular graphs necessitates an understanding of atomic valences and bonding rules. Learning-based strategies, in contrast, automate the discovery of meaningful perturbations by training on data, allowing for the generation of counterfactual examples at inference time.

In this work, we introduce XPlore, a significant advancement over the framework by \citet{lucic2022cf} for generating GNN counterfactual explanations. While prior methods often restrict modifications to edge deletions, XPlore substantially broadens the search space by enabling edge additions and node-feature perturbations within a unified framework. This expansion enables a more comprehensive exploration of the graph's structure and attributes, resulting in higher-quality, more diverse counterfactuals.
Central to our approach is a specialized loss function that treats the target GNN as a fixed ``oracle''. By directly leveraging oracle gradients -- the directional signals from the model's internal loss -- XPlore autonomously guides minimal perturbations to identify the closest counterfactual instance. This design increases transparency by construction. Given access to oracle gradients, the explainer introduces no auxiliary surrogate or generative model and performs no additional learning. As a result, the explanation pipeline remains lightweight, faithful to the classifier's behavior, and straightforward to inspect and audit.

Our contributions are as follows.

\noindent\textbf{(1) Comprehensive Input Perturbation:} Our framework supports edge additions, edge removals, and node feature perturbations, enabling perturbation of nearly all input degrees of freedom. This richer perturbation space allows us to explore how structural and attribute-level changes jointly influence the oracle's prediction process and to uncover richer, truly counterfactual motifs that edge-only methods cannot express. 

\noindent\textbf{(2) Gradient-Guided Optimization:} By exploiting the properties of the loss, we find counterfactuals that are locally closest in the loss-optimization landscape, as guided by directed gradient-based modifications (\Cref{app:loc_min_convergence}). This guarantees that the counterfactual explanation is not only minimal under the objective but also consistent with the oracle's learned decision boundaries. 

\noindent\textbf{(3) Mitigation of Out-of-Distribution Issues:} We introduce a cosine similarity metric to quantify both structural and semantic fidelity of counterfactuals. Together with the use of edge additions and feature perturbations, this mitigates---although not fully solves---out-of-distribution artifacts \citep{chen2023d4explainerindistributiongnnexplanations} and preserves class-relevant discriminative cues. Discussion follows in \Cref{sec:ood_effect}. 

\noindent\textbf{(4) Domain coverage and Strong Empirical Performance:} We evaluate our method on a diverse suite of real-world and synthetic graph benchmarks and show that it consistently outperforms state-of-the-art methods in validity and fidelity while maintaining competitive runtime and semantic coherence.

\section{Related Work}
We distinguish between inherently explainable and black-box methods, focusing on counterfactual explanations for graph classification. While counterfactuals are well-established in images and text \citep{vermeire2022explainable,Zemni_2023_CVPR}, their application to graphs is less common \citep{liu2021multi,ma2022clear,nguyen2022explaining,numeroso2021meg,tancf2}. As categorized in \citet{prado2022survey}, graph counterfactual explanation methods fall into heuristic search and learning-based approaches. This work focuses on instance-level, learning-based explainers that identify minimal perturbations to flip a prediction, providing actionable, data-point-specific insights.

Learning-based strategies often use perturbation matrices \citep{tancf2}, reinforcement learning (RL) \citep{numeroso2021meg}, or generative models \citep{ma2022clear,prado2024robust}. Some notable methods include CF-GNNExp. \citep{lucic2022cf}, which learns a binary perturbation matrix; CF$^\text{2}$ \citep{tancf2}, which balances factual and counterfactual reasoning through multi-objective optimization; and CLEAR \citep{ma2022clear}, which uses a Variational Autoencoder (VAE) to generate counterfactuals. Other approaches, such as RSGG-CE \citep{prado2024robust} and D4Explainer \citep{chen2023d4explainerindistributiongnnexplanations}, rely on advanced techniques like Generative Adversarial Networks (GANs) and discrete denoising diffusion, respectively. {Additional methods are INDUCE \citep{verma2024induce} which is based on a RL-based inductive approach; COMBINEX \citep{giorgi2025combinex} and C2Explainer \citep{ma2025c2explainer}.} Emerging research on time-related graph counterfactuality \citep{prenkaj2024unifying, qu2024greedy} is outside the scope of our work.

Unlike SoTA, we provide graph- and node-level explanations. In the former, we identify changes needed to alter the prediction for an entire graph. In the latter, we leverage loss information from individual nodes, enabling us to perturb both the local graph structure and node features to achieve the targeted objective of changing a single node's label. Here, an oracle performs node-wise classification, and our method focuses on modifying the inputs that directly influence the prediction of the targeted node.

\section{Method} \label{sec:Method}
\subsection{Problem Formulation}

\paragraph{Graph Counterfactual Explanation}
Suppose we have a well-trained GNN serving as an oracle $\Phi: \mathcal{G}\rightarrow\mathcal{Y}$ and an original graph instance $G \in \mathcal{G}$, with predicted label $\Phi(G)$. The objective of counterfactual explanation is to find a perturbed graph $G'$, obtained by a counterfactual model $\mathcal{E}: \mathcal{G} \rightarrow \mathcal{G}$, that minimally deviates from $G$ while ensuring that the oracle's prediction changes, i.e., $\Phi(G') \neq \Phi(G)$ \citep{prado2022survey}.  Let $\Delta(G,G')$ denote the distance function measuring the difference between $G$ and $G'$; then the problem can be stated as:
\begin{equation}
G^* = \underset{G' \in \mathcal{G}'}{\arg\min} \;\Delta(G, G')
\text{ s.t }  \Phi(G) \neq \Phi(G').
\end{equation}%

\begin{table}[!t]
\centering
\caption{Notation used in the Section~\ref{sec:Method}.}
\label{tab:notation}
\resizebox{\linewidth}{!}{
\begin{tabular}{ll}
\toprule
\textbf{Symbol} & \textbf{Description} \\
\midrule
$\mathcal{E}, \Phi$& Counterfactual and oracle model\\
$G, G'$ & Original and counterfactual graphs \\
$A, \hat{A}$ & Adjacency matrices (original and counterfactual) \\
$P, \bar{P}$ & Original and new perturbation matrices for edges \\
$N, \bar{N}$ & Binary and continuous node feature perturb. matrices \\
$X, W$ & Node feature and weight matrices  \\
$\bar{D}$ & Degree matrix \\
$\Gamma$ & Probability weight matrix with values $\gamma_{i,j} \in [0,1]$\\
$\sigma(\cdot)$ & Sigmoid activation function\\
 softmax($\cdot$) & Softmax function\\
  $\mathcal{T}_\alpha (\cdot)$& Entry-wise threshold: $\mathcal{T}_{\tau}(X) = \mathds{1} \{\sigma(X)\geq \alpha \}$\\
$\odot$ & Element-wise (Hadamard) product \\
 $\mathbf{1}\{\cdot\}$& Indicator function\\
$K$ & Number of optimization iterations (e.g. 50, see \ref{app:hyperparams})\\
$L_{pred}, L_{dist}$ & Prediction and distance losses
\\
\bottomrule
\end{tabular}%
}
\end{table}
As done in \citet{lucic2022cf}, we generate counterfactuals by reformulating the above hard-constrained formulation into a soft-unconstrained optimization by minimizing the loss function: 
\begin{equation}\label{eq:loss}
L(G, G^\prime) = L_{\text{pred}}(G, G^\prime \mid \Phi)+\beta \, L_{\text{dist}}(G, G^\prime).
\end{equation}%
Here, $L_{pred}$ is a prediction loss that encourages $\Phi(G^\prime)\ne \Phi(G)$ and $L_{dist}$ is a distance loss that promotes similarity between $G^\prime$ and $G$, with the trade-off controlled by $\beta$. Thus, the optimal counterfactual example for $G$ is obtained by solving \Cref{eq:loss}.

\paragraph{Node Counterfactual Explanation}
Our method can be extended to node-level counterfactual explanations. In this scenario, rather than modifying only the target node, we perturb the entire graph, including all node features. Moreover, the loss formulation can be adapted to encompass a set of nodes by imposing a soft constraint that preserves the class labels of nodes other than the target, using binary cross-entropy with logits as the loss function. This framework naturally generalizes to multi-node classification tasks, where the objective is to generate counterfactual explanations for a set of nodes rather than a single node.

\begin{figure*}[!t]
    \centering    
    \begin{subfigure}{0.165\linewidth}
        \centering
        \includegraphics[width=\linewidth]{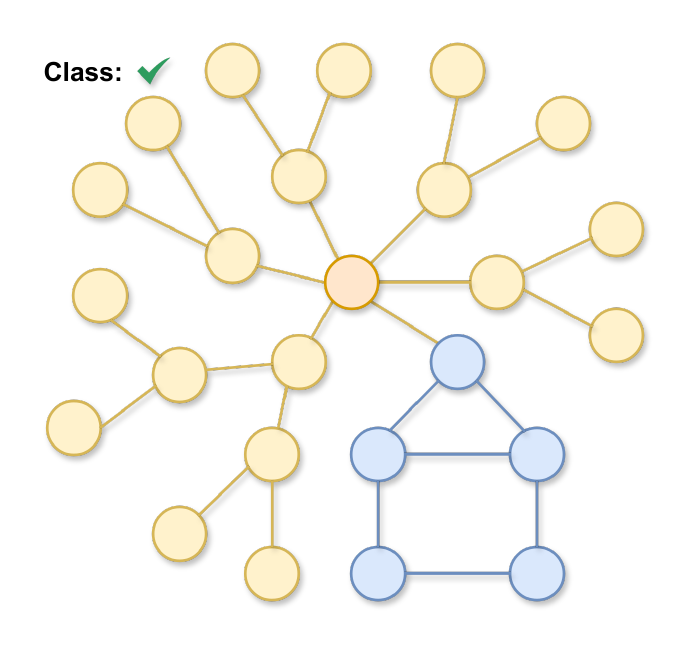}
        \caption{BAS\\house motif.}
    \end{subfigure}%
    \hfill
    \begin{subfigure}{0.165\linewidth}
        \centering
        \includegraphics[width=\linewidth]{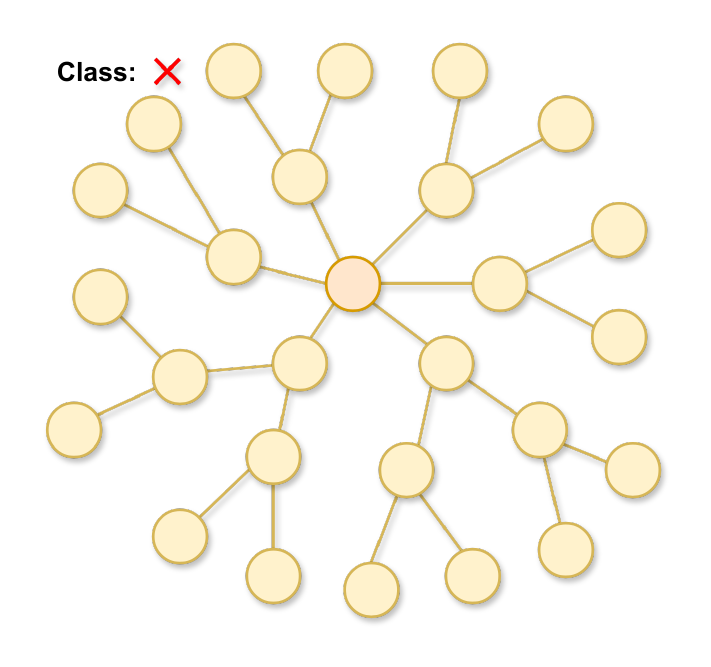}
        \caption{TCR\\No cycle.}
    \end{subfigure}%
    \hfill
    \begin{subfigure}{0.165\linewidth}
        \centering
        \includegraphics[width=\linewidth]{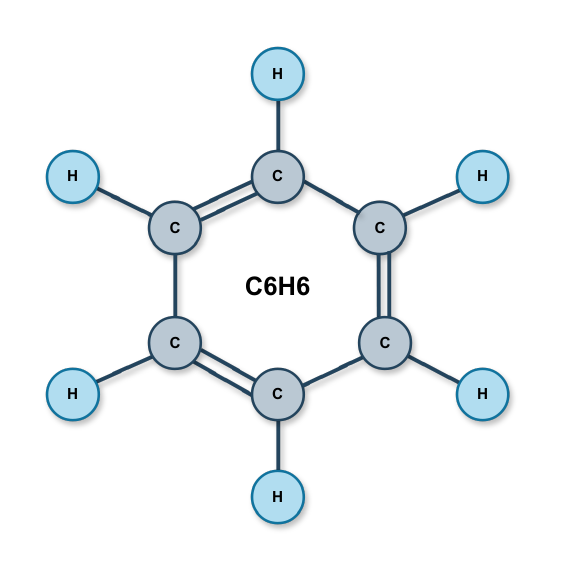}
        \caption{{Benzene\\neutral charge.}}
    \end{subfigure}%
    \hfill
    \begin{subfigure}{0.165\linewidth}
        \centering
        \includegraphics[width=\linewidth]{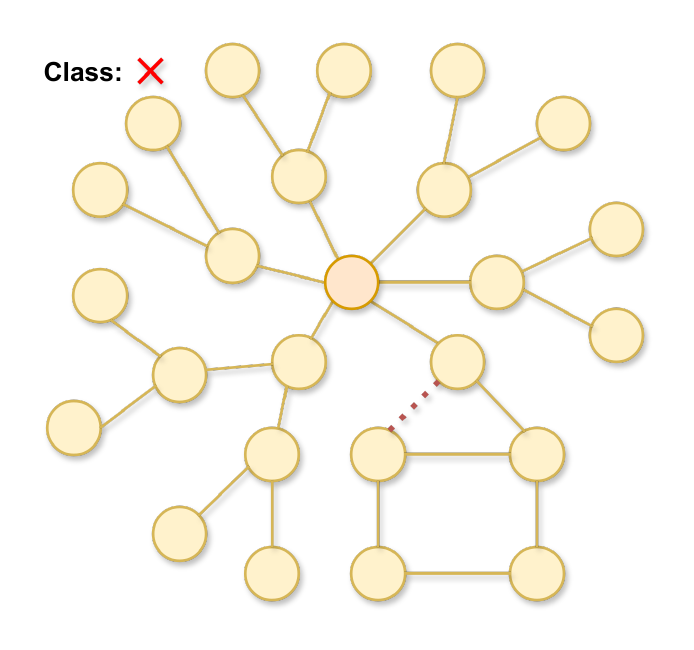}
        \caption{BAS\\edge deletion.}
    \end{subfigure}%
    \hfill
    \begin{subfigure}{0.165\linewidth}
        \centering
        \includegraphics[width=\linewidth]{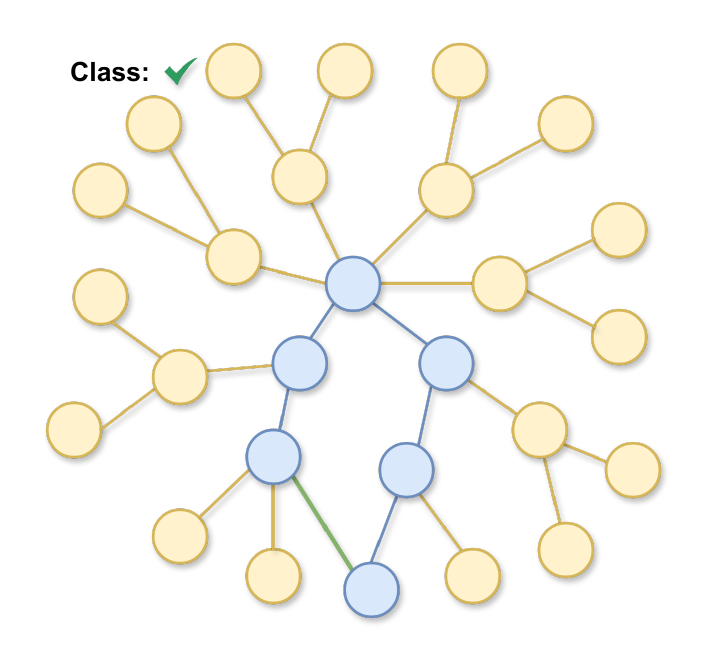}
        \caption{TCR\\edge insertion.}
    \end{subfigure}%
    \hfill
    \begin{subfigure}{0.165\linewidth}
        \centering
        \includegraphics[width=\linewidth]{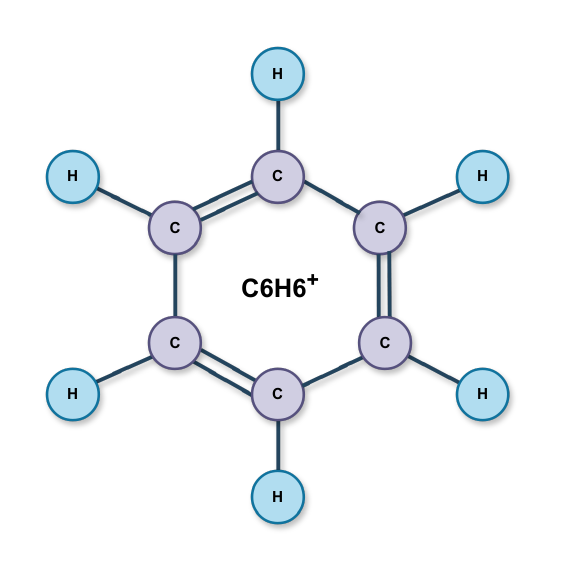}
        \caption{{Benzene\\pos. charge.}}
    \end{subfigure}
    
    \caption{Illustration of counterfactual explanations in graph classification. \textit{a} and \textit{d}: The class changes only after an \textbf{edge deletion} (red, dashed).  \textit{b} and \textit{e}: the predicted class changes only after an \textbf{edge addition} (green). {\textit{c} and \textit{f}: Feature perturbation of the carbon-group illustrates benzene converting from neutral to cationic by electron loss.} XPlore enables these types of perturbations, alone or in combination, offering a broader search space than edge deletion-only methods.}
    \label{fig:grid}
\end{figure*}

\subsection{Proposed Method}
We propose XPlore to sieve the less-constrained counterfactual search space compared to the base model \citep{lucic2022cf}, aiming at better explanations while reducing the out-of-distribution effect shown in \citet{chen2023d4explainerindistributiongnnexplanations}. XPlore not only drops instance edges but also adds them, enabling the discovery of counterfactual explanations for which a class change is unattainable by edge removal alone (e.g., TCR in~\Cref{sec:datasets}; see  ~\Cref{fig:grid}). Moreover, by modifying node features, our approach can identify counterfactuals in which a class change is unachievable solely through edge modifications (or any amount thereof), but rather through feature adjustments. We explored two modalities for handling features: i.e., \textbf{(1)} applying a gating mechanism to retain or discard node features, similarly to the edge-dropping process introduced in \citet{lucic2022cf}; \textbf{(2)} allowing for continuous adjustment in node features, enabling their values to increase/decrease smoothly. We provide formal guarantees on convergence, on {$\ell_1$}-minimality of perturbations, and semantic fidelity in Appendix ~\ref{app:appendix}.

\paragraph{Original Adjacency matrix perturbations} Let $A~\in~\{0,1\}^{n\times n}$ be the adjacency matrix -- a square matrix whose elements indicate whether pairs of vertices are adjacent or not in the graph. We initially define the perturbation similarly to \citet{lucic2022cf} , having the CF generated matrix $\hat{A} = P \odot A$, with $P = \mathbf{1}_{n\times n}
$, a binary perturbation matrix full of ones, with same dimensionality of $A$; and $\odot$ denoting element-wise (Hadamard) product (a real-valued $\bar{P}$ is first generated, then a threshold with a entry-wise sigmoid transformation $\sigma(\cdot)$ is performed to obtain a binary ${P}$, either discarding or retaining edges: $\mathcal{T}_{0.5}(\bar{P}) = \mathds{1} \{\sigma(\bar{P})\geq 0.5\}$). To include self-message passing, we add self-loops in the form of the identity matrix during CF search. We define $X$ as the node feature matrix (so that $G = (A, X)$), $W$ as the weight matrix, and $\bar{D}$ as the degree matrix, a diagonal matrix containing the number of edges attached to each vertex, based on $P \odot A+I$. The original goal is to remove edges by zeroing out entries in the adjacency matrix, so find $P$ (which acts as a gate) that minimally perturbs $A$ and use it to compute $\hat{A}$. Hence, the original counterfactual generating model, parameterized by $P$, is:
\begin{equation}\label{eq:original_formulation}
    \mathcal{E}(A, X, W; {P})\!=\!\mathrm{softmax}\big[\!\bar{D}^{-\tfrac{1}{2}}\!({P} \odot A\! +\!I)\bar{D}^{-\tfrac{1}{2}}\!XW\big].
\end{equation}%
Note that here, differently from \citet{lucic2022cf}, we perform the CF search not on a subgraph neighbourhood of a node, but rather on the whole graph.\footnote{This does not entail an add-on on the computational complexity as shown in \Cref{sec:complexity}.} For undirected graphs, only the upper triangular part of $\bar{P}$ is parametrized, and it is symmetrized at each step to obtain $\hat{A}$. This ensures that adjacency perturbations remain symmetric.

\paragraph{Edge-masking with Free Insertions}
The original formulation of the framework has some downsides when it comes to updating $P$: when the oracle prediction is obtained with the associated loss, and new edges need to be added. Essentially, gradients for adjacency and perturbation matrix entries are computed and flow back, but edges not present in $A$ {($A_{ij} = 0$)} cause respective gradients of $P$ to be zeroed out {($\frac{\partial \mathcal{L}}{\partial P_{ij}} 
= \frac{\partial \mathcal{L}}{\partial \hat{A}_{ij}} \cdot \frac{\partial \hat{A}_{ij}}{\partial P_{ij}}, \;
\hat{A}_{ij} = A_{ij} \cdot P_{ij} 
\Rightarrow 
\frac{\partial \hat{A}_{ij}}{\partial P_{ij}} = A_{ij}.$)} -- \emph{As a results edges not initially present in $A$ cannot be inserted during CF search by design choice, and gradients of missing edges indicating the direction of optimal change, even if are back-propagated until $P$,  will not lead to any update to it, thus introducing confusion for other updates that may be co-adapted/correlated to them}. To enable the inclusion of any edge, we swap the roles of $A$ and $P$:  initializing $\bar{P} \leftarrow A$, hence making it store the original adjacencies, but making it able to be updated by the iterative gradients; and making $\bar{A} = \mathbf{1}_{n\times n}$ be a matrix full of ones, fully connecting the graph. We account for the original missing edges by removing them from $\bar{P}$ (setting their values to zero), now $\hat{A}$ is computed as $\hat{A} = \bar{P}\odot \bar{A}$ plus self-loops. With this swap, we give the explainer the ability to freely add edges by updating the entire graph accordingly to the loss. $\bar{A}$ now becomes a (redundant) matrix of ones whereas $\bar{P}$ stores the graph connections, ready to be perturbed. We add to $P$ a small Gaussian noise ($std: \sigma=0.1; \mu = 0$) to break the symmetry in the gradient updates.

\begin{table}[!t]
\centering
\caption{XPlore steps to perturb edges and node features.}
\resizebox{\linewidth}{!}{
\begin{tabular}{@{}cp{0.8\linewidth}@{}}
\toprule
\textbf{Step} & \textbf{Action}\\ \midrule
\textbf{1}    & Initialize $\bar{P} \leftarrow A+\mathcal{N}(0, 0.1^2)$ and optionally add $\Gamma$.\\
\textbf{2}    & Initialize $\bar{N} \gets \mathbf{1}_{n \times f}$\\
\textbf{3}    & Compute $\hat{A} = \textit{max}(\mathcal{T}_{0.5}(\bar{P}), I) $\\
\textbf{4}    & Compute contribution of $\hat{A}$ and $\bar{N}$ to $L_{pred}$ \newline through the oracle's prediction (\Cref{eq:final_formulation,eq:loss_pred}).\\
\textbf{5}    & Update $\bar{P}$ and $\bar{N}$ via gradient descent. \\ \bottomrule
\end{tabular}%
}
\end{table}

We also assign a probability weight to missing edges, denoted by $\bar{P}$: we define a matrix $\Gamma \in [-1,1]^{n\times n}$. $\Gamma$ is added to $\bar{P}$ at initialization: $\bar{P} \leftarrow \bar{P}+\Gamma$. The idea is to manually adjust each $\gamma_{ij} \in [-1,1]$ so that by adding $\Gamma$ it is possible to induce edge presence or absence. For example, we may be interested in a counterfactual explanation with specific connections; we can directly encode them in $\Gamma$.
Note that we swapped $P$ and $A$ to be consistent with the original formulation: since $A$ now consists of a matrix of ones, it is redundant for updating $\bar{P}$ and can be omitted. Now, the explainer is as in~\Cref{eq:new_formulation}:

\begin{algorithm}[!t]
\caption{XPlore.}\label{alg:XPlore}

\resizebox{0.9\columnwidth}{!}{
\begin{minipage}{\columnwidth}
\begin{algorithmic}[1]
  \STATE \textbf{Input} graph $G=(A,X)$, trained oracle $\Phi$, explainer $\mathcal{E}$, loss function $L$, learning rate $\alpha$, number of iterations $K$, matrix $\Gamma$, node features $gate$ flag, distance function $d$.
  \STATE $y\gets\Phi(G)$ 
  \STATE $\bar{P} \gets A+\mathcal{N}(0,0.01)+\Gamma$
  \STATE $\bar{N} \gets \mathbf{1}_{n \times f}$
  \STATE $G^*= [ \ ]$   
  \STATE \textbf{for} $K$ iterations \textbf{do}
    \STATE \hspace{\algorithmicindent}  $G', G^* =\; \text{\texttt{GET\_CF\_EXAMPLE()}}$
    \STATE \hspace{\algorithmicindent}  $L \gets L(G, G', \Phi)$  \COMMENT{\textcolor{gray}{$\blacktriangleright$ Eqns. \ref{eq:loss_pred}-\ref{eq:loss_dist}}}
    \STATE \hspace{\algorithmicindent}  $\bar{P} \gets \bar{P} - \alpha \nabla_{\bar{P}}L$  \COMMENT{\textcolor{gray}{$\blacktriangleright$ Update $\bar{P}$}}
    \STATE \hspace{\algorithmicindent} $\bar{N} \gets \bar{N} - \alpha \nabla_{\bar{N}}L$ \COMMENT{\textcolor{gray}{$\blacktriangleright$  Update $\bar{N}$}}
  \STATE \textbf{return} $G^*$
  \STATE
  \STATE \textbf{func} \texttt{GET\_CF\_EXAMPLE()}
  \STATE \hspace{\algorithmicindent} $\hat{A} \gets max(\mathcal{T}_{0.5}(\bar{P}),I)$ \COMMENT{\textcolor{gray}{$\blacktriangleright$ Perturbed adj.}}
  \STATE \hspace{\algorithmicindent} \textbf{if} $gate$ \textbf{then}  \COMMENT{\textcolor{gray}{$\blacktriangleright$ Gate node feat.}}
  \STATE \hspace{\algorithmicindent} \hspace{\algorithmicindent} ${N} \gets \mathcal{T}_{0.5}(\sigma(\bar{N}))$
  \STATE \hspace{\algorithmicindent} \textbf{else} \COMMENT{\textcolor{gray}{$\blacktriangleright$ Freely perturb node feat.}}
    \STATE \hspace{\algorithmicindent} \hspace{\algorithmicindent} ${N} \gets \bar{N}$
    \STATE \hspace{\algorithmicindent} $N \leftarrow N \odot X$
  \STATE \hspace{\algorithmicindent} ${G}'_{cand} \gets (\hat{A}, {N})$
  \STATE \hspace{\algorithmicindent} \textbf{if} $\Phi(G) \neq \Phi({G}'_{cand})$ \textbf{then}
    \STATE \hspace{\algorithmicindent} \hspace{\algorithmicindent} ${G}' \gets {G}'_{cand}$
    \STATE \hspace{\algorithmicindent} \hspace{\algorithmicindent} \textbf{if} not ${G}^*$ \textbf{then}
      \STATE \hspace{\algorithmicindent} \hspace{\algorithmicindent}\hspace{\algorithmicindent} ${G}^* \gets {G}'$  \COMMENT{\textcolor{gray}{$\blacktriangleright$ First CF}}
    \STATE \hspace{\algorithmicindent} \hspace{\algorithmicindent} \textbf{else if} $d(G, {G}') \leq d(G, {G}^*)$ \textbf{then}
      \STATE \hspace{\algorithmicindent} \hspace{\algorithmicindent}\hspace{\algorithmicindent} ${G}^* \gets {G}'$  \COMMENT{\textcolor{gray}{$\blacktriangleright$ Closer CF}}
  \STATE \hspace{\algorithmicindent} \textbf{return} $G', {G}^*$
\end{algorithmic}
\end{minipage}
}
\end{algorithm}

\begin{equation}\label{eq:new_formulation}
    \mathcal{E}(X, W; \bar{P}) = \mathcal{S} \big[\bar{D}^{-1/2} \hat{A} \bar{D}^{-1/2}XW\big],
\end{equation}%
where $\mathcal{S}$ is either the element-wise sigmoid or the softmax function, depending on whether the task is binary or multi-class; and $\bar{D}$ is based on $\hat{A}$, slightly different from \Cref{eq:original_formulation}, ensuring that values in the diagonal are ones, specifically when self-loop are already present -- due to $\hat{A}$ having no more the diagonal entries fixed at zero as it depends on $\bar{P}$  for storing perturbed adjacencies and not on $A$ any more.

\paragraph{Node Features Perturbation} Similarly to what has been done for the edge perturbation matrix $\bar{P}$, we can introduce a perturbation matrix $\bar{N} = \mathbf{1}_{n\times f}$, $f$ number of node features, to perform continuous perturbation on the node features. We can apply two different mechanisms (\Cref{alg:XPlore}): i.e., \textbf{(1)} gate the feature values with a sigmoid, by applying the sigmoid and then the threshold to have a binary mask, and \textbf{(2)} let them freely change, updating them proportionally to the gradient information of the loss. In results, these mechanisms added on top of the edges addition, are defined as XPlore w/ gating and XPlore w/ freedom. Hence, the final formulation of our explainer, considering all perturbations, is: 
\begin{equation}
N  = \mathcal{T}_{0.5}(\bar{N}) \text{ if } gate \text{ else } \bar{N}
\end{equation}
\begin{equation}\label{eq:final_formulation}
    \mathcal{E}(X, W; \bar{P}, \bar{N}) = \mathcal{S} \big[\bar{D}^{-1/2}\hat{A}\bar{D}^{-1/2}(X\odot {N})W\big].
\end{equation}%
\paragraph{Loss function optimization}
We generate $\bar{P}$ by minimizing ~\Cref{eq:loss}, defining the prediction loss $L_{\text{pred}}$ as in ~~\Cref{eq:loss_pred}: for $L_{logits}$ we employ the cross-entropy (CE) loss for single-label classification tasks, encompassing binary and multi-class scenarios, and the binary cross-entropy with logits loss for multi-label classification tasks:
\begin{equation}\label{eq:loss_pred}
\begin{gathered}
    {L}_{\mathrm{pred}}(G, {G}' \;|\; \Phi) = - \mathbf{1} \big[\Phi(G) = \Phi({G}')\big]\\ \cdot {L}_\mathrm{logits}(\Phi_{\ell-1}(G), \Phi_{\ell-1}(G')),
\end{gathered}
\end{equation}%
where $\Phi_{\ell-1}(G)$ outputs the logits given the input $G$, and $\ell$ is the number of layers. $G'$ is $\mathcal{E}(G)$.
The $L_{dist}$ in~\Cref{eq:loss} is the element wise distance between $G$ and $G'$, corresponding to the sums of the two L\textsubscript{p}-norms of $A$ and $A'$, and $X$ and $X'$:
\begin{equation}\label{eq:loss_dist}
    L_\mathrm{dist}(G,{G}') = \|A - {A}'\|+\|X - {X}'\|.
\end{equation}

\begin{table*}[!t]
\centering
\caption{Comparison of XPlore with SoTA methods (validity $\uparrow$ -- up; fidelity $\uparrow$ -- down). \textbf{Bold} is best-performing; \underline{underline} is second-best. \textbf{XPlore is best in 17/18 on validity, and 17/18 on fidelity (of which 1/17 on par with RSGG-CE).} \Cref{tab:results_extended,tab:results_extended_new_row} show a detailed comparison of other metrics. $\dag$ indicates the second-best explainer.}
\label{tab:results}
\setlength{\tabcolsep}{1mm}
{\fontsize{9}{11}\selectfont
\resizebox{\textwidth}{!}{
\begin{tabular}{@{}lcccccccccccccccccc@{}}
\toprule
\multicolumn{1}{c}{\textbf{}} & \rotatebox{90}{TCR} & \rotatebox{90}{BAS} & \rotatebox{90}{BZR} & \rotatebox{90}{AIDS} & \rotatebox{90}{ENZYMES} & \rotatebox{90}{Fingerprint} & \rotatebox{90}{COLORS-3} & \rotatebox{90}{TG} & \rotatebox{90}{MUTAG} & \rotatebox{90}{COX2} & \rotatebox{90}{BBBP} & \rotatebox{90}{PROTEINS} & \rotatebox{90}{COLLAB} & \rotatebox{90}{TRIANGLES}  &\rotatebox{90}{{DBLP}} &\rotatebox{90}{{IMDB}} &\rotatebox{90}{{TWITTER}}&\rotatebox{90}{{MSRC}}\\ \midrule
iRand & 27.92 & 50.70 & 27.16 & 0.00 & 26.67 & 0.09 & 42.99 & 36.16 & 2.66 & \underline{69.81} & 19.76 & 18.87 & 4.60 & 6.38  & 1.17& 3.70& 0.00&\underline{95.74}\\
CF$^\text{2}$ & 50.04 & 45.78 & 19.75 & 0.10 & 68.33 & 24.52 & 52.07 & 49.86 & 0.00& 24.20 & \underline{25.26} & 16.35 & \underline{52.66} & 37.13  & 5.76& 50.60& 38.82&90.94\\
CLEAR & 50.68 & 50.96 & \underline{60.49} & 16.75& 83.17 & 72.73 & 0.00 & 58.40& 35.11 & 22.06 & 22.90 & 0.00 & 0.00 & 89.99  & 0.68& 56.20& 7.65&23.98\\
RSGG-CE $\dag$ & \underline{67.90}& \underline{91.04} & 21.23 & \underline{19.80} & \underline{98.33} & \underline{90.46} & \underline{94.57}& \underline{89.28} & \underline{56.91}& \textbf{99.36} & 22.90 & \underline{58.67} & 0.00 & \underline{99.84}  & \underline{57.51}& \textbf{86.10}& \underline{48.36}&\textbf{100.0}\\
D4Explainer & 44.82 & 35.16 & 20.00 & 0.10 & 68.00 & 24.52 & 0.00 & 49.86 & 9.57& 22.06& 21.24 & 0.00 & 0.00 & 31.11  & 7.02& 49.60& 23.45&90.94\\ \midrule
CF-GNNExpl & 50.04& 44.18 & 19.75 & 0.10 & 68.33 & 24.52 & 52.07 & 49.86 & 10.11& 22.06 & 22.31 & 16.35 & \underline{52.66} & 37.13  & 5.76& 50.60& 38.82&90.94\\
XPlore & \textbf{100.0}& \textbf{100.0}& \textbf{100.0}& \textbf{32.30}& \textbf{100.0}& \textbf{100.0}& \textbf{100.0}& \textbf{100.0}& \textbf{67.55}& \textbf{99.36} & \textbf{81.51}& \textbf{65.41}& \textbf{100.0} & \textbf{100.0}  & \textbf{91.84}& \underline{78.80}& \textbf{50.71}&\textbf{100.0}\\ \bottomrule
\end{tabular}%
}
}
\setlength{\tabcolsep}{1mm}
{\fontsize{9}{11}\selectfont
\resizebox{\textwidth}{!}{
\begin{tabular}{@{}lcccccccccccccccccc@{}}
& & & & & & & & & & & & & & & & & & \\
\toprule
iRand & 0.279 & 0.507 & 0.262 & 0.000& 0.238 & 0.001& 0.300 & 0.356 & 0.005 & \underline{0.677} & \underline{0.275} & 0.183 & 0.006 & 0.053  & 0.06& 0.026& 0.000&\underline{0.391}\\
CF$^\text{2}$ & 0.500 & 0.457 & 0.188 & 0.001 & 0.633& 0.133 & 0.398 & 0.499 & 0.005 & 0.229 & 0.253 & 0.151 & \underline{0.262} & 0.364  & 0.027& 0.366& 0.345&0.371\\
CLEAR & 0.507 & 0.510 & \underline{0.595} & 0.168& 0.797 & 0.515 & 0.000& 0.584& 0.309 & 0.208 & 0.229 & 0.000& 0.000& 0.892  & 0.005& 0.420& 0.039&0.012\\
RSGG-CE $\dag$ & \underline{0.679}& \underline{0.910} & 0.202 & \underline{0.198} & \underline{0.930} & \underline{0.653} & \underline{0.824} & \underline{0.888} & \underline{0.516}& \textbf{0.968} & 0.229 & \underline{0.556} & 0.000& \underline{0.987}  & \underline{0.450}& \textbf{0.664}& \underline{0.434}&\textbf{0.423}\\
D4Explainer & 0.448 & 0.446 & 0.190 & 0.001 & 0.632 & 0.133 & 0.000& 0.499 & 0.021& 0.208& 0.212 & 0.000& 0.000& 0.302  & 0.031& 0.358& 0.161&0.371\\ \midrule
CF-GNNExpl & 0.500& 0.442 & 0.188 & 0.001 & 0.633 & 0.133 & 0.398 & 0.499 & 0.005& 0.208 & 0.223 & 0.151 & \underline{0.262} & 0.364  & 0.027& 0.366& 0.345&0.371\\
XPlore & \textbf{1.000} & \textbf{1.000} & \textbf{0.980} & \textbf{0.323}& \textbf{0.942} & \textbf{0.730}& \textbf{0.896}& \textbf{0.994} & \textbf{0.548} & \textbf{0.968} & \textbf{0.803}& \textbf{0.627} & \textbf{0.570}& \textbf{0.988}  & \textbf{0.748}& \underline{0.606}& \textbf{0.613}&\textbf{0.423}\\ \bottomrule
\end{tabular}%
}
}
\end{table*}

\subsection{Algorithmic Implementation}
We summarize the details of our method XPlore\footnote{\url{https://anonymous.4open.science/r/XPlore}} in \Cref{alg:XPlore}.  Given a graph in the test set G, its prediction is obtained from the GNN oracle $\Phi$. $\bar{P}$ is initialized as the adjacency matrix $A$ and summed with $\Gamma$, giving missing edges a probability value of $\gamma_{i,j} \in [-1,1]$. 
$\bar{N}$ is assigned to a matrix of ones. XPlore is run for $K$ iterations (\Cref{app:hyperparams}), at each one, an optimization step is performed to find a valid counterfactual and improve it, there is no stopping criterion as the same CF is iteratively modified and potentially improved across successive iterations, as conceived by \citet{lucic2022cf}.~\Cref{eq:final_formulation} is used to find a counterfactual example. ${P}$ is computed by applying a sigmoid transformation on $\bar{P}$ and then a threshold to obtain a binary matrix. 
Although a hard threshold $\mathcal{T}$ is applied in the forward pass,  gradients are back-propagated through this non-differentiable step via the straight-through estimator (STE) \citep{bengio2013estimatingpropagatinggradientsstochastic}, which treats thresholding as identity during backpropagation (\Cref{sec:A.2 Thresholding}). We add self-loops by summing the identity matrix. According to the node perturbation mechanism, node features are either gated or simply multiplied by the node features perturbation matrix ${\bar{N}}$. The candidate graph produced ${G}'_{cand}$ is fed to the GNN oracle $\Phi$. If the output prediction is different from the initial node, a valid counterfactual is found. Upon success, the closest counterfactual is returned as the optimal CF example ${G}^*$ after $K$ iterations. $\bar{P}$ and $\bar{N}$ are updated based on the loss, which is calculated according to \Cref{eq:loss,eq:final_formulation,eq:loss_pred,eq:loss_dist}.  This setting allows for performing updates freely by perturbing edges and node features and adding missing ones. The optimal explanation is retrieved as $\Delta_{G}^* = G - {G}^*$. The algorithm is linear in the number of edges and in the multiplication of features with the number of nodes (i.e. $O\bigl(|E|\,d+n\,d\,f\bigr)$,  \Cref{sec:complexity}).

\section{Experiments}

\subsection{Experimental Setup}
We train separate GCN oracles for each dataset using an 80-20 train-test split, with architectural and hyperparameter details provided in~\Cref{app:oracles} and~\ref{app:hyperparams}, respectively. We compare XPlore against CF-GNNExpl. \citep{lucic2022cf}, CF$^\text{2}$~\citep{tancf2}, CLEAR~\citep{ma2022clear}, RSGG-CE~\citep{prado2024robust}, D4Explainer~\citep{chen2023d4explainerindistributiongnnexplanations}, and iRand~\citep{prado2023generative}. For iRand, we set the edge perturbation probability ($p$) to 0.01 and the number of iterations ($t$) to 3.

Following the protocol from~\cite{prado2022survey}, we use a comprehensive evaluation suite, including \textit{Oracle Accuracy}, \textit{Validity}, \textit{Fidelity}, \textit{Sparsity}, \textit{Graph Edit Distance (GED)}, \textit{Oracle Calls} and \textit{Runtime}. Additionally, we use the cosine similarity (CS), which measures the semantic similarity between original and counterfactual graphs. Unlike structural metrics such as GED or sparsity, CS captures graph meaning alignment by comparing embeddings.

Given a set of graph embedders $E$ ($|E| = M$), we compute vector embeddings $\mathbf{e_{ji}} = E_j(G_i)$ and $\mathbf{e'_{ji}} = E_j(G'_i)$ for each graph $G_i$ ($|G|=N$) and its counterfactual $G'_i$, using $E_j \in E$. CS is defined as:
\begin{equation}
\text{CS}(\mathbf{e_{ji}}, \mathbf{e'_{ji}}) = \frac{1}{MN} \sum_{j = 1}^{M} \sum_{i = 1}^{N} \frac{\mathbf{e_{ji}} \cdot \mathbf{e'_{ji}}}{|\mathbf{e_{ji}}| |\mathbf{e'_{ji}}|} \in [-1, 1]
\end{equation}
The set $E$ includes following embedders: Feather-G~\citep{rozemberczki2020characteristicfunctionsgraphsbirds}, Graph2Vec~\citep{narayanan2017graph2veclearningdistributedrepresentations}, NetLSD~\citep{Tsitsulin_2018}, WaveletCharacteristic~\citep{Wang_2021}, IGE~\citep{galland:hal-02947290}, LDP~\citep{cai2022simpleeffectivebaselinenonattributed}, GeoScattering~\citep{pmlr-v97-gao19e}, GL2Vec~\citep{10.1007/978-3-030-36718-3_1}, SF~\citep{delara2018simplebaselinealgorithmgraph}, and FGDS~\citep{NIPS2017_d2ddea18}. Definitions of standard metrics are detailed in Appendix~\ref{app:metrics}.

\subsection{Experimental Results} \label{sec:results}

\begin{figure}
    \centering
    \includegraphics[width=\linewidth]{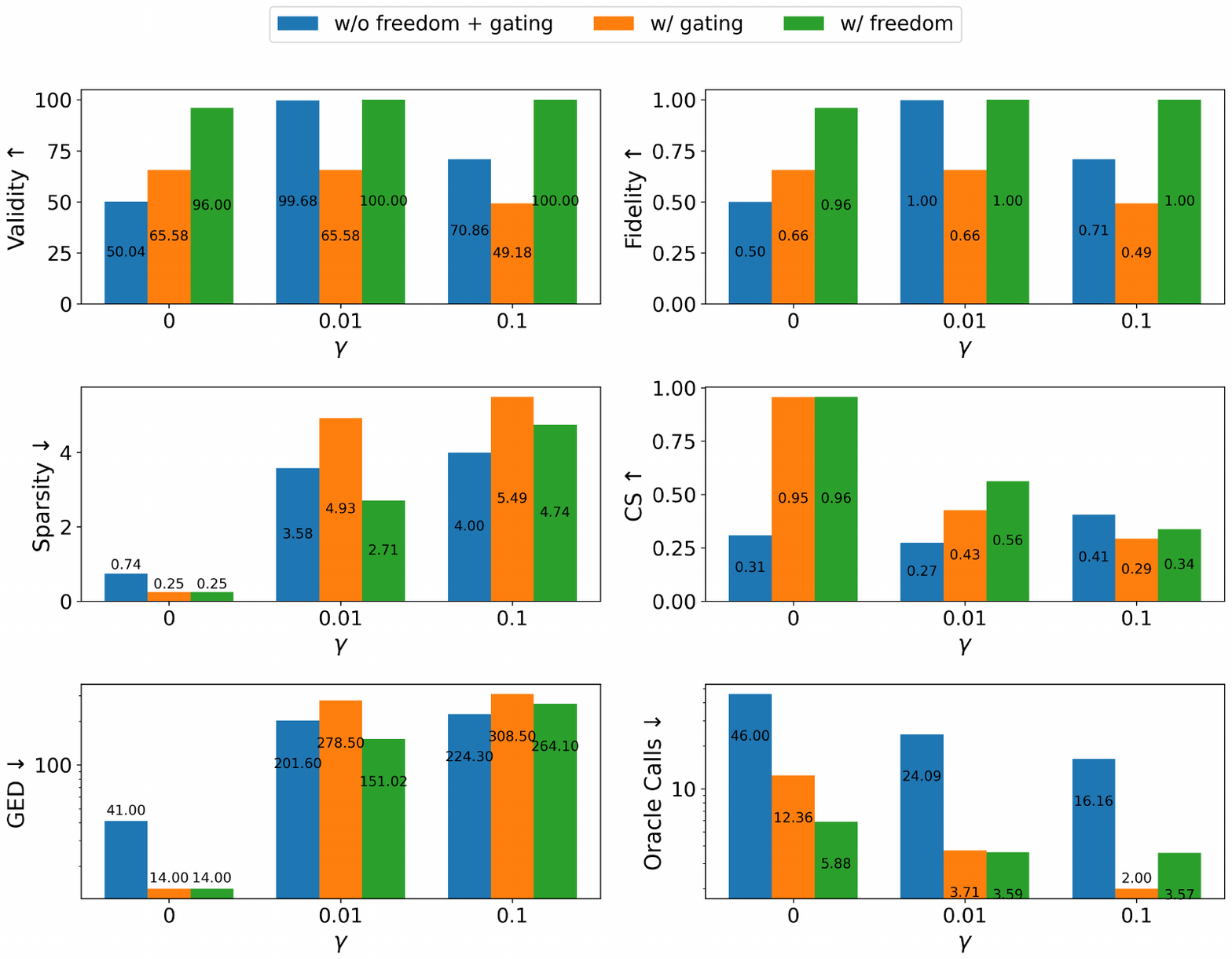}
    \caption{Metrics comparison over TCR dataset for different values of hyperparameter $\gamma$.}
    \label{fig:TCR_metrics}
\end{figure}




\begin{table}[!t]
\centering
\caption{Validity (\%) for different variants of XPlore ($\gamma=0$.01). Bold values are best.}
\label{fig:Accuracy_ablation}
\resizebox{\linewidth}{!}{%
\begin{tabular}{@{}lcccc@{}}
\toprule
\multicolumn{1}{l}{} & \multicolumn{1}{l}{w/o freedom + gating} & \multicolumn{1}{l}{w/ gating} & \multicolumn{1}{l}{w/ freedom} & \multicolumn{1}{l}{CF-GNNExpl}\\ \midrule
TCR   & 99.680          & 65.580          & \textbf{100.00}  &50.040\\
TG    & 99.660          & \textbf{100.00} & \textbf{100.00}  &49.860\\
BAS   & 94.340& \textbf{100.00}  & \textbf{100.00}  &44.180\\
MUTAG & 45.745& 60.106& \textbf{67.553} &10.106\\ 
BZR   & 98.519& \textbf{100.00} & \textbf{100.00}  &19.753\\
COX2  & \textbf{99.358}& 97.859& 98.073 &22.056\\
AIDS  & 0.300& 19.850& \textbf{30.300}&0.100\\
BBBP  & 38.892& 38.794& \textbf{79.794}&22.315\\
 ENZYMES& 76.500& \textbf{99.667}&98.833 &68.333\\
 PROTEINS& 17.071& 64.960&\textbf{65.409} &16.352\\
 Fingerprint& 26.105& \textbf{98.232}&94.742&24.523\\
 COLLAB& 55.640& \textbf{100.00}&93.380&52.660\\
 COLORS-3& 67.943& \textbf{100.00}&94.476 &52.067\\
 TRIANGLES& 39.873& 90.436&\textbf{100.00} &37.127\\ \bottomrule
\end{tabular}%
}
\end{table}

\noindent\textbf{XPlore on average achieves +15.1\% percentage improvement in validity, and +14.0\% on fidelity across the board vs. the second-best in graph classification.}~\Cref{tab:results} reports the performance of XPlore and all state-of-the-art baselines across ten runs on each dataset. See \Cref{tab:results_extended,tab:results_extended_new_row} for a broader analysis. We show the validity, fidelity, sparsity, and oracle calls, as they describe the explainer's ability to find valid counterfactuals that are also cheap to query the underlying predictor. Note that XPlore is the best across 17/18 datasets in terms of validity and 17/18 in terms of fidelity. This shows that a simple modification to the learning objective and node-feature manipulation enhances the explainer's faithfulness. XPlore maintains a relatively low number of oracle calls, indicating that in at most 14.692 (see TG) iterations, the optimal counterfactual is found (see line 6 of~\Cref{alg:XPlore}). 
Additionally, we show the relationship between GED and CS in~\Cref{fig:ged_vs_cs}. 

\begin{figure}
    \centering
    \includegraphics[width=\linewidth]{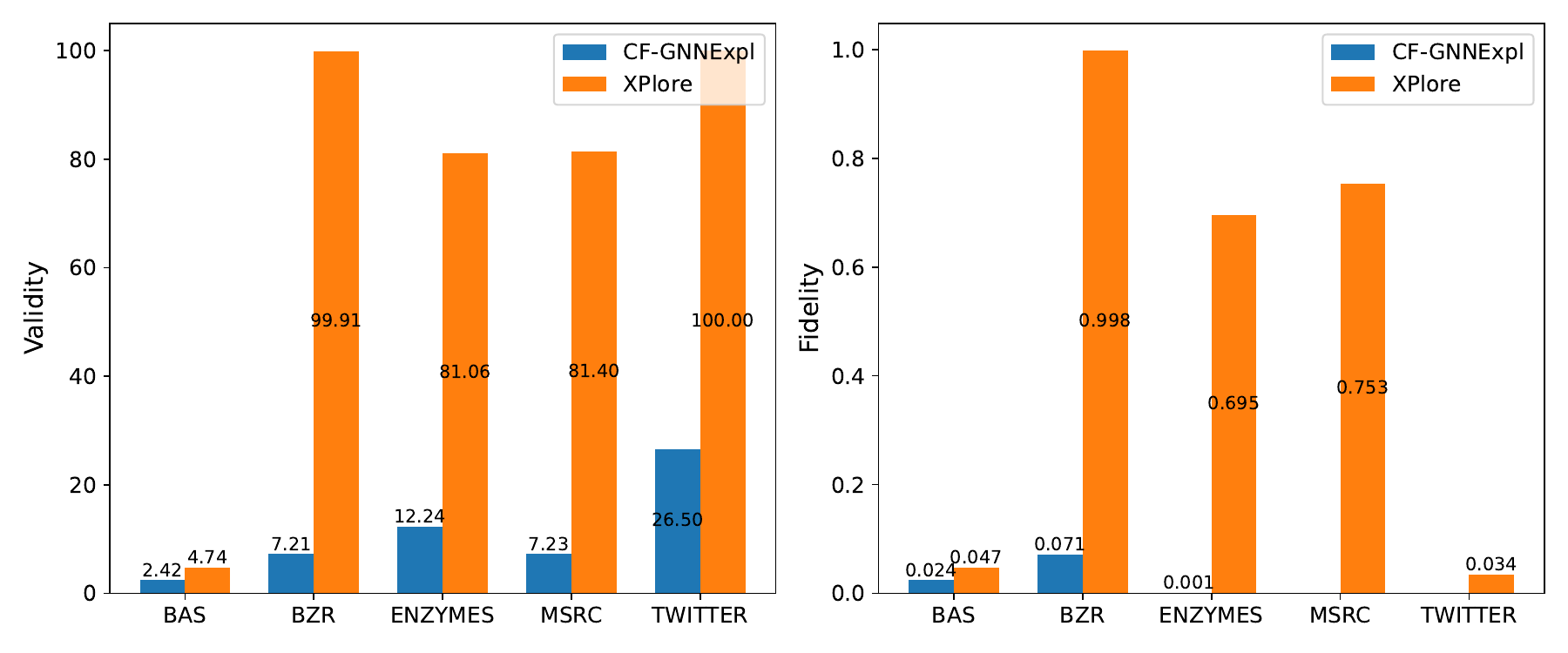}
    \caption{Validity and Fidelity results for node explanation.}
    \label{tab:NodeResults}
\end{figure}

However, we are interested in those explainers that might induce more edit changes to the counterfactual while maintaining high semantic similarity. 
Note that XPlore generally exhibits high GED values, while consistently achieving strong CS scores, underscoring our intuition that semanticity plays an important role in counterfactuality. We reserve more thorough experimentation for future work. Lastly, we report the inference-time runtime (s) for all methods to show that XPlore is at least on par with SoTA explainers (see~\Cref{fig:runtime} in the Appendix).

\noindent\textbf{XPlore outperforms CF-GNNExpl on node classification, improving per-class explainability by +62.30 validity points on average.} We compare both methods on five datasets, selecting one per domain. Despite prior claims~\citep{lucic2022cf}, generating effective node-level counterfactuals remains difficult, particularly in class-specific scenarios. As summarized in~\Cref{tab:NodeResults}, XPlore shows consistent improvements, reflecting its ability to finely manipulate node-level features.

\subsection{Residual Out-of-distribution Influence}\label{sec:ood_effect}
\begin{figure}[!t]
    \centering
        \includegraphics[width=\linewidth]{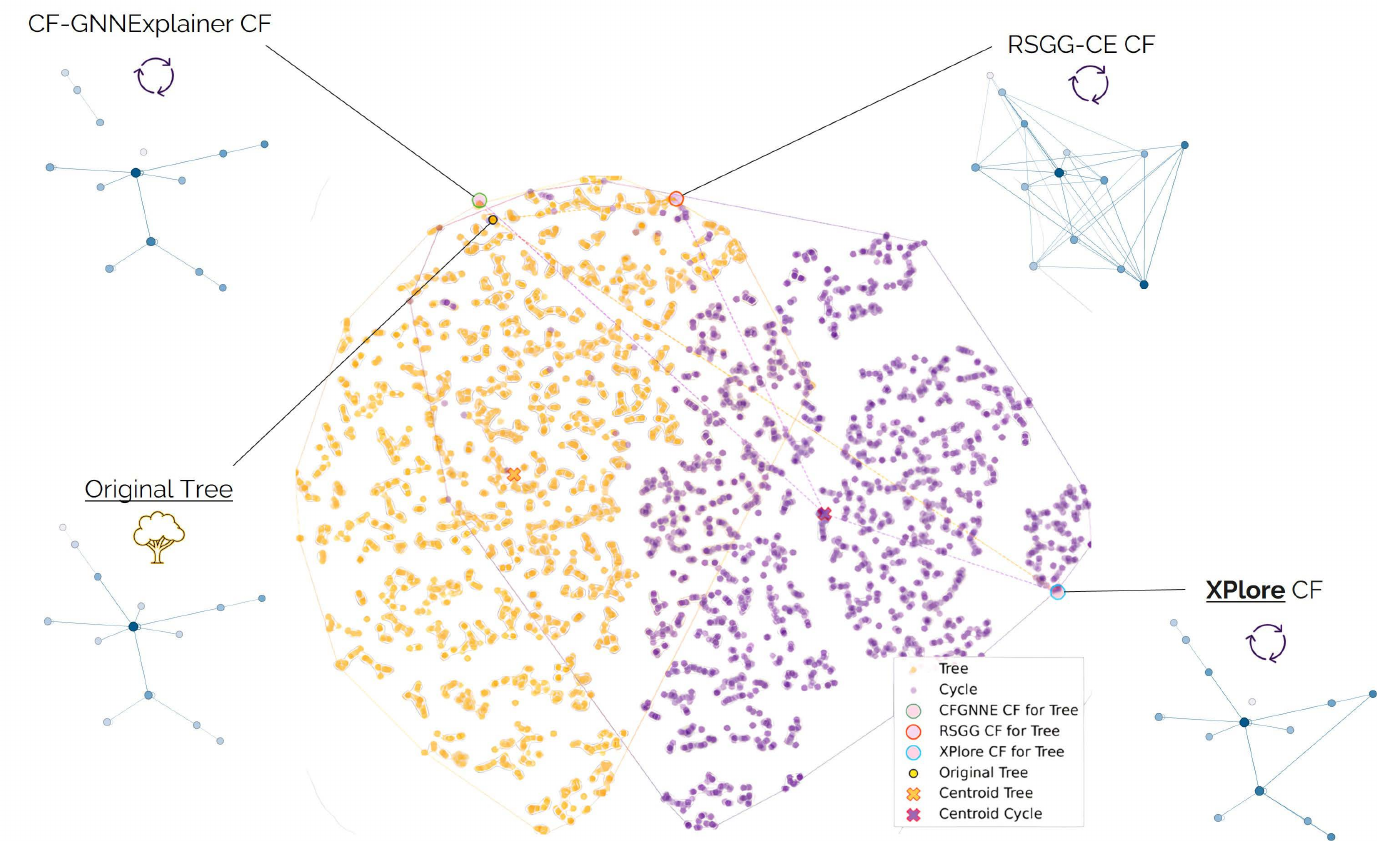}
    \caption{t-SNE projection of Wavelet Characteristic embeddings for TCR, comparing CFs generated by CF-GNNExpl, XPlore, and RSGG for the Tree motif. CF-GNNExpl and RSGG find a close CF but fail to land in the Cycle distribution, while XPlore achieves this correctly.}
    \label{fig:tsne}
\end{figure}
\begin{figure*}[!t]
    \centering
\includegraphics[width=\linewidth]{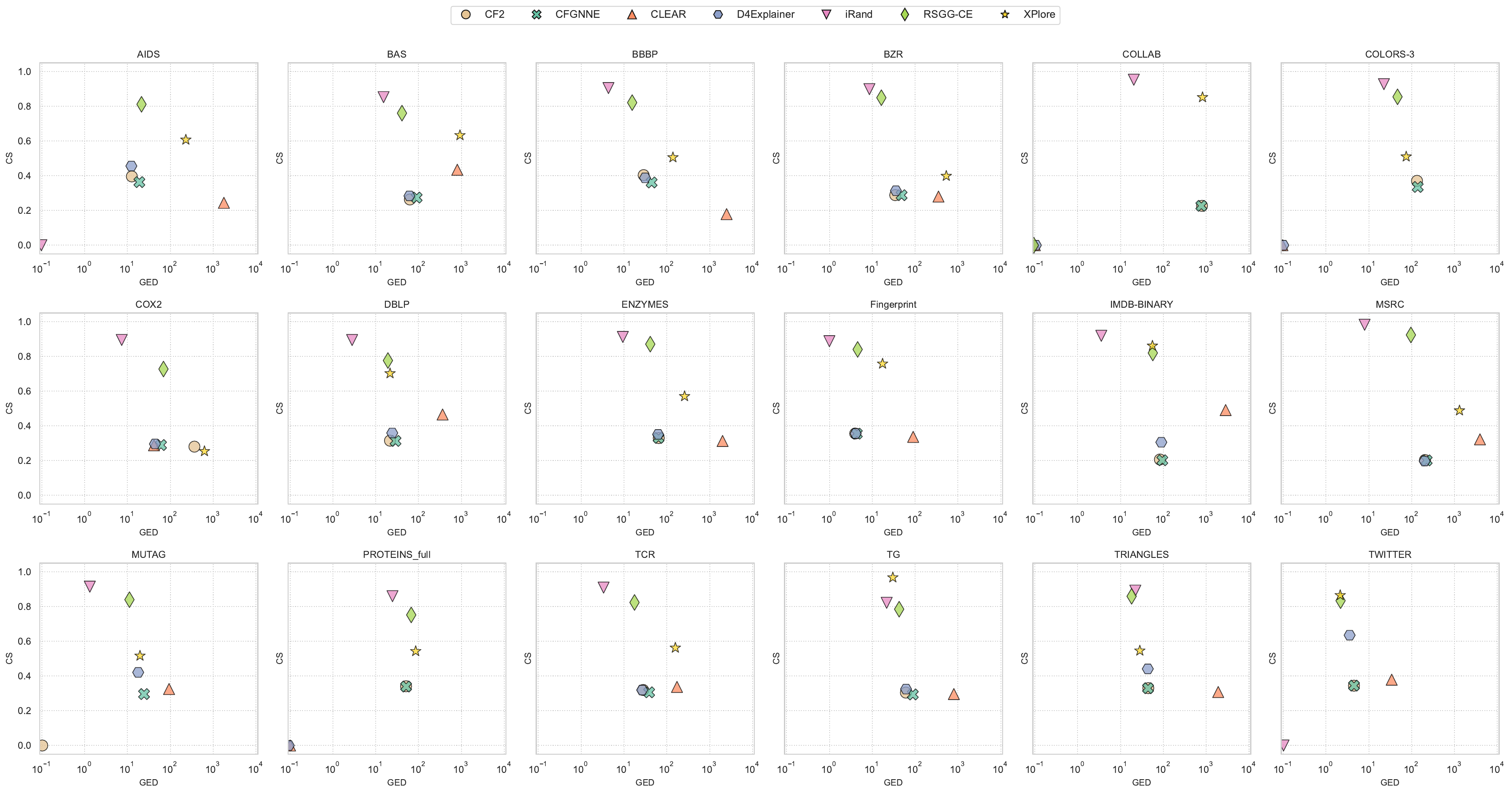}
    \caption{Relationship between GED and CS.    XPlore consistently achieves strong CS scores w.r.t. the competitors along with the same GED values in log-scale, indicating that counterfactuals retain meaningful semantic relations to original graphs, partially mitigating OOD.}
    \label{fig:ged_vs_cs}
\end{figure*}
\citet{chen2023d4explainerindistributiongnnexplanations} observed that CF-GNNExpl leverages the out-of-distribution (OOD) effect that influences the oracle's prediction, by deriving explanatory subgraphs while omitting additional potential edges.

Consequently, the extracted explanation lacks discriminative information for the CF class but is still classified as CF due to this OOD effect, ultimately misleading the oracle and compromising reliability.  As shown in~\Cref{fig:ged_vs_cs}, we reduce OOD effect exploitation by increasing cosine similarity: XPlore semantic similarity is higher w.r.t. competitors with similar GED scores and operating at the same perturbation level, thus relying less on OOD effect. \Cref{fig:violin_plots} shows that by identifying more CFs, XPlore also captures harder instances, yielding elongated or bimodal CS distributions, reflecting competitive performance on both easy and hard instances. t-SNE projection of the Wavelet Characteristic embedding space reveals that XPlore's CF exhibits substantial semantic overlap with the target class, landing correctly in the Cycle distribution embedding space, which integrates both topological and node attribute information (\Cref{fig:tsne}). Other explainers do not achieve this behavior, remaining within the Tree distribution and exploiting the OOD effect. Although XPlore's ability leads to more comprehensive, and hence more robust, explanations, the OOD effect persists (see \Cref{app:OOD}). This highlights the oracle's limited expressive power and its lack of training to mitigate adversarial instances. We refer the reader to~\cite{leemann2024towards} for non-adversarial CF explanations.

\subsection{Ablation Studies}\label{sec:ablation}
\textbf{Free node perturbation leads to better counterfactual validity.} Since XPlore performs free node perturbations, we study the effect of this freedom by constraining it and assessing whether simpler mechanisms work in the same manner. Therefore, we test \textbf{(1)} XPlore w/ gating that performs edge additions and node feature manipulation by gating them, similarly to the original edge drop mechanism \citep{lucic2022cf}, either retaining or discarding them; and \textbf{(2)} a variant that does not have the gating mechanism and cannot freely manipulate node features, called XPlore w/o freedom \& gating. To be consistent with the previous nomenclature, we rename XPlore in XPlore w/ freedom. We show the validity of these variants in \Cref{fig:Accuracy_ablation}. Thanks to the better manipulation of graph elements, all three variants have higher validity w.r.t. the original CF-GNNExplainer. The additional ability to alter node features generally leads to even better results. The flexibility of this modulation enables XPlore to better meet the explainability criteria and identify more effective counterfactuals.

In~\Cref{fig:TCR_metrics}, we report metrics for $\gamma = 0$ and $\gamma = 0.01$ respectively, where $\gamma$ is the value used to populate the $\Gamma$ matrix entries that correspond to missing edges of $A$. $\Gamma$ is added to $P$ at initialization. Recall that an edge is present if and only if $\sigma(v_i,v_j) > 0.5$ s.t. $v_i, v_j$ are nodes. A positive $\gamma$ increases the likelihood that edges are present at initialization, since $\sigma(\gamma+\epsilon)>0.5$. However, due to the added Gaussian noise, the presence is not guaranteed. This behavior may result in a higher number of counterfactuals found, but at the cost of higher GED, greater sparsity, and lower CS. This functionality may be needed if an edge must be present in the CF explanation; this prior knowledge is injected into the matrix $\Gamma$ as edge probabilities.

\section{Conclusion}
We introduced XPlore, a novel counterfactual explainer for Graph Neural Networks (GNNs) that explores a more complete search space, enabling both edge deletions/insertions and node-feature perturbations. Our method is not a black box itself; it relies on a basic gradient-based optimization building block to perform an intuitive counterfactual search. Once an oracle is trained, our explainer is fast and lightweight, requiring no further training or intense computation. XPlore identifies minimal yet impactful modifications, ensuring high-quality counterfactual explanations while avoiding heuristic-driven biases. Through extensive benchmarking across multiple datasets, XPlore consistently outperformed state-of-the-art explainers in terms of validity and fidelity. Our empirical evaluation highlights that XPlore achieves an average percentage gain of +17.3\% in validity and +15.0\% in fidelity over the second-best method across multiple datasets. We also rely on cosine similarity of graph embeddings as a complementary metric, demonstrating that XPlore captures structural modifications while maintaining semantic fidelity. This confirms the effectiveness of incorporating a more flexible perturbation space and a comprehensive loss function.

\paragraph{Future work} We acknowledge that mitigating out-of-distribution (OOD) effects and optimizing node feature perturbations remain key challenges. While our method has shown promise in this area, the ability to reduce these effects depends on how effectively the oracle captures the data distribution. We posit that less robust oracles are more prone to OOD explanations, which gives us hope that a more robust oracle, such as a diffusion model-based one, could significantly reduce XPlore's OOD effects. We will also investigate the link between counterfactual explanations and model robustness, aiming to create explainability methods that are both interpretable and resistant to adversarial manipulation. Another avenue would be to refine the loss function to better balance minimal perturbations with explanation quality, and to develop more precise control over feature modifications.

\bibliographystyle{ACM-Reference-Format}
\bibliography{bibliography}
\clearpage
\appendix
\section{Theoretical Foundations of Gradient-Guided Counterfactual Perturbations} \label{app:appendix}

This section is meant to give a mathematical justification for why our gradient‑based procedure converges and finds small ($\ell_1$‑minimal) local perturbations to the graph and features.

\noindent\textbf{Appendix Roadmap.}
We first establish convergence of projected gradient descent (A.1), then handle non‑smooth thresholding (A.2), prove the \(\ell_1\)‑minimality bound (A.3), derive the edge‑insertion/deletion condition (A.4), show our prediction loss is \(L\)‑smooth (A.5), choose sparsity–accuracy trade‑offs (A.6), analyze computational complexity (A.7), extend to weighted and directed graphs (A.8), and finally discuss search‑space expressivity (A.9).\\
\\
In \Cref{sec:Method}, we described a soft, differentiable objective $L(\cdot)$ (~\Cref{eq:loss}) that trades off between (a) changing the model’s prediction and (b) paying an $\ell_p$ “cost” for every edit. What remains is to show that, if we optimize this objective with a natural algorithm, we end up with a valid counterfactual that doesn’t make too large edits. Hence, we here justify why the joint gradient-based optimization over edge masks ${P}$ and feature perturbations ${N}$ (cf. ~\Cref{alg:XPlore,eq:loss,eq:final_formulation}) converges to a meaningful counterfactual and yields nearly $\ell_\mathrm{1}$-minimal changes.

\subsection{Convergence to Local Minimizers}\label{app:loc_min_convergence} 

Let \(\theta = (\mathrm{vec}(P), \mathrm{vec}(N))\) and define the soft objective \(L(\theta)\) as in~\Cref{eq:loss}, with projected gradient descent (PGD) updates
\[
\theta^{t+1} = \Pi_\Theta \big(\theta^t - \eta \nabla L(\theta^t)\big),
\]
where the projection \(\Pi_\Theta\) enforces \(P \in [0,1]^{n\times n}\) and leaves \(N \in \mathbb{R}^{n\times f}\) unconstrained.  

\begin{assumption}
The loss \(L\) is differentiable except for a convex nonsmooth part \(L_\mathrm{dist}\), \(L\)-smooth on the smooth part, and coercive on \(\Theta\).
\end{assumption}

Under these standard conditions and for step-size \(\eta \le 1/L\), classical PGD convergence results \citep[Ch.~2.3]{bertsekas1997nonlinear} guarantee that the iterates satisfy \(\theta^{t+1}-\theta^t \to 0\), and any limit point \(\theta^*\) is Clarke-stationary:
\[
0 \in \nabla L(\theta^*) + NC_\Theta(\theta^*),
\]
i.e., no feasible small perturbation can locally decrease \(L\). Hence, PGD converges to a stationary point of the soft objective under the box constraint, as required for our method.

\subsection{Handling non-smooth Thresholding via Surrogate Gradients} \label{sec:A.2 Thresholding}
Let $P \in [0,1]^{n\times n}$ be our continuous edge-importance matrix. We obtain a hard adjacency
$$\hat{A} = 1[P >0.5] \in \{0,1\}^{n\times n}$$
for the forward GNN pass. To propagate gradients back to $P$, we use the common PyTorch "detach" trick (an instance of the straight-through estimator \citep{bengio2013estimatingpropagatinggradientsstochastic}) 
\begin{lstlisting}[style=mypython, caption={STE-based hard‐thresholding in PyTorch}]
mask   = (P >= 0.5).float() # no gradient
P_hard = P+(mask - P).detach() # forward uses mask; backward flows into P
outputs = GNN(P_hard, features)
\end{lstlisting}%
In effect, the backward pass treats the binarization as if it were the identity function. This can be seen as a limiting case of using a steep sigmoid surrogate $\sigma_\alpha(x) =  1/(1+e^{-\alpha(x-0.5)})$ with $\alpha \to \infty$. 
Under this STE, the composite loss
$$L(P) = L_\mathrm{pred}(\hat{A}, X)+\beta||P-P_0||_1$$
is differentiable almost everywhere in $P$, and convergence to a Clarke-stationary point follows from standard projected-gradient arguments. $P_0$ is the original adjacency mask.\\
Note that STE is a heuristic used in practice, and the previous convergence proofs refer to the continuous/relaxed problem, not the exact gradients through hard thresholding.

Note that we replace the non‐differentiable hard threshold by the identity in the backward pass (i.e.\ STE). Using the identity in the backward pass induces a bias (dependent on the size of the downstream gradient) by propagating nonzero gradients where the true hard‐threshold has zero derivative, whereas zeroing those gradients would irreversibly freeze masked entries -- hence, to mitigate this, one may use a temperature‐controlled sigmoid $\sigma_\alpha$ (or even an annealed $\alpha$) that smoothly trades off bias and trainability.

\subsection{$\ell_1$-Minimality Bound of Perturbations}
Let again,
$$\theta = \big(\mathrm{vec}(P), \mathrm{vec}(N)\big), \quad \theta_0=\big(\mathrm{vec}(A), \mathrm{vec}(1_{n\times f})\big),$$
and the well-known soft objective
$$
L(G; \theta) = L_\mathrm{pred}\big(\Phi(G), E(G;\theta)\big) +\beta ||\theta-\theta_0||_1,
$$
\begin{lemma}
Assume $L_\mathrm{pred}$ is bounded below and set $\Delta = L_\mathrm{pred} - \mathrm{inf}_\Theta L_\mathrm{pred}$, the difference of $L_\mathrm{pred}$ with the infimum (greatest lower bound) of the prediction loss $L_\mathrm{pred}$ over all feasible $\theta\in \Theta = [0,1]^{n\times n}\times\mathbb{R}^{n\times f}$, the best possible point in $\Theta$ achievable by $\theta^*$. 
Then any stationary point $\theta^*$ of $L$ satisfies 
$$||\theta^*-\theta_0||_1 \leq \frac{\Delta}{\beta},$$
where $\theta_0$ is the original unperturbed $(P,N)$-vector. 
\end{lemma}
\begin{proof}
Since $\theta_0 \in \Theta$, stationarity (or just minimality) of $\theta^*$ gives
$$L(\theta^*) \leq L(\theta_0)  \Longrightarrow  L_\mathrm{pred}(\theta^*)+\beta||\theta^* - \theta_0||_1 \leq L_\mathrm{pred}(\theta_0).$$
Rearranging,
\begin{equation*}
\begin{gathered}
    \beta||\theta^* - \theta_0||_1  \leq   L_\mathrm{pred}(\theta_0) -  L_\mathrm{pred}(\theta^*) \\\leq  L_\mathrm{pred}(\theta_0) - \mathrm{inf}_\Theta L_\mathrm{pred} = \Delta
\end{gathered}
\end{equation*}
Dividing by $\beta$, we get the clean bound
$$||\theta^*-\theta_0||_1 \leq \frac{\Delta}{\beta}.$$ Thus, choosing $\beta \geq \Delta/ \epsilon$ forces $||\theta^*-\theta_0||_1 \leq \epsilon$, (i.e. the size of our edits (in $\ell_1$) is at most the ratio of “how much we can lower the prediction loss” over “how costly each unit of edit is.” 

\end{proof} 
\textbf{Practical choice of \(\beta\).}
In theory, the bound 
\[
  \|\theta^*-\theta_0\|_1 \;\le\;\frac{\Delta}{\beta},\quad\Delta=L(\theta_0)-L(\theta^*)
\]
guides \(\beta\) to achieve a target \(\ell_1\)‐norm (sparsity).  In practice \(\Delta\) is unknown, so we select \(\beta\) empirically -- e.g.\ via grid search or cross‐validation -- by monitoring the resulting \(\|\theta^*\|_1\) (or edge count) and choosing the smallest \(\beta\) that attains the desired sparsity level \(\epsilon\).



\subsection{Edge-Insertion and Edge-Deletion Condition }
We now derive a simple, quantitative criterion under which a zero-entry $p_{ij} = 0$ will be driven strictly positive, and a one-entry $p_{ij} = 1$ will be driven to zero, by projected gradient descent on the soft objective in~\Cref{eq:app_loss}.\\
Recall that $L_\mathrm{dist}$ penalizes deviations of both edges and node features, but here we focus on its effect on the edge variable $p_{ij}$.\\
\\
We consider the soft objective
\begin{equation}\label{eq:app_loss}
L(P,N) \;=\; L_{\rm pred}(P,N)\;+\;\beta\,L_{\rm dist}(P,N),
\end{equation}%
where
\begin{itemize}
  \item $P=[p_{ij}]\in\{0,1\}^{n\times n}$ is the \emph{final} binary adjacency mask,  
  \item $N$ is the node‐feature matrix,  
  \item $\bar P=[\bar p_{ij}]\in[0,1]^{n\times n}$ is the \emph{original} continuous adjacency (here we assume no noise, in practice we add a small Gaussian noise to $\bar{P}$, which does not materially change the $\pm1$ subgradients at the boundaries),  
  \item $L_{\rm dist}(P,N)$ includes the edge penalty $\sum_{i,j}|p_{ij}-\bar p_{ij}|$, ~\Cref{eq:loss_dist} (plus any feature‐penalty terms).
\end{itemize}%
To decide whether an edge coordinate \(p_{ij}\) will switch its value under projected gradient descent, we inspect the one‐dimensional update
\[
p_{ij}^{(t+1)}
=\Pi_{\{0,1\}}\Bigl(
  p_{ij}^{(t)}
  -\eta\Bigl(
    g_{ij}
    +\beta\,d_{ij}
  \Bigr)
\Bigr),
\]
where
\[
g_{ij}
=\frac{\partial L_{\rm pred}}{\partial p_{ij}}
\,,\quad
d_{ij}
=\frac{\partial}{\partial p_{ij}}\bigl|p_{ij}-\bar p_{ij}\bigr|
=\mathrm{sign}\bigl(p_{ij}-\bar p_{ij}\bigr).
\]
By the sign‐definition,
\[
d_{ij} =
\begin{cases}
+1, & p_{ij}>\bar p_{ij},\\
-1, & p_{ij}<\bar p_{ij},
\end{cases}
\quad
d_{ij}\in[-1,1]\ \text{if }p_{ij}=\bar p_{ij}.
\]

\paragraph{Insertion.} 
At the insertion boundary \(p_{ij}=0\) with continuous \(\bar p_{ij}>0\), we have \(d_{ij}= -1\) if \(\bar p_{ij}>0\), but to capture the switch from 0 to 1 we consider the subgradient at the boundary:  
\[
p_{ij}^{(t+1)}=1
\Longleftrightarrow
0 - \eta\bigl(g_{ij}+\beta\,d_{ij}\bigr)\le -1
\Longleftrightarrow
-g_{ij} > \beta\,d_{ij}.
\]
Since at the 0‐boundary we take \(d_{ij}=+1\) for the “hardest” subgradient,
\[
-g_{ij}>\beta
\]
is required for insertion.

\paragraph{Deletion.}
At the deletion boundary \(p_{ij}=1\) with \(\bar p_{ij}<1\), we similarly take \(d_{ij}=-1\) and find
\[
p_{ij}^{(t+1)}=0
\Longleftrightarrow
1 - \eta\bigl(g_{ij}+\beta\,d_{ij}\bigr)\ge 1
\Longleftrightarrow
g_{ij}>\beta.
\]
\noindent In both cases, an edge flips exactly when its \emph{marginal benefit} or \emph{cost} in the prediction loss exceeds the fixed penalty \(\beta\), yielding a transparent sparsity threshold.

\subsubsection*{Role of the \(\Gamma\) Matrix}
Rather than using a uniform $\ell_1$ penalty 
\(
  \beta\sum_{i,j}\bigl|\bar P_{ij}-A_{ij}\bigr|,
\)
we shift the target by \(\Gamma\) and write
\[
  \min_{\bar P}\;L_{\rm pred}(\bar P)
  \;+\;\beta\sum_{i,j}\bigl|\bar P_{ij} - (A_{ij}+\gamma_{ij})\bigr|.
\]
By simple algebraic manipulation, this is equivalent (up to an additive constant) to
\[
  \min_{\bar P}\;L_{\rm pred}(\bar P)
  \;+\;\beta\sum_{i,j}w_{ij}\,\bigl|\bar P_{ij}-A_{ij}\bigr|,
  \quad
  w_{ij} \;\equiv\; \frac{1}{1+\gamma_{ij}}.
\]
In the latter form, the KKT subgradient for each $(i,j)$ is
\[
  \underbrace{\frac{\partial L_{\rm pred}}{\partial \bar P_{ij}}}_{\text{model term}}
  +\;
  \beta\,w_{ij}\,\mathrm{sign}\bigl(\bar P_{ij}-A_{ij}\bigr),
\]
so a larger \(\gamma_{ij}\) $\Rightarrow$ smaller \(w_{ij}\) $\Rightarrow$ smaller magnitude of the regularizer $\Rightarrow$ cheaper to flip edge \((i,j)\), all \textbf{without} changing the global~\(\beta\).

\subsection{Smoothness of the Prediction Loss}

We next show, at a high level, that our prediction loss \(L_{\rm pred}(P,N)\) admits an \(L\)-Lipschitz continuous gradient (i.e.\ is \(L\)-smooth) in the perturbation variables \((P,N)\).

\begin{lemma}\label{lem:smoothness}
Let \(f_\theta\) be a GNN with fixed weights \(\theta\), built from layers that are each Lipschitz continuous (e.g.\ linear transforms, neighbor‐aggregation, and elementwise activations such as ReLU).  Define
\[
  L_{\rm pred}(P,N)=L\bigl(f_\theta(A \odot P,\,X+N),\,y\bigr),
\]
where \(L\) is a twice‐differentiable loss (e.g.\ cross‐entropy).  Then there exists a constant \(L>0\) depending only on \(\theta\), the GNN architecture, and \(L\), such that for all \((P,N)\) and \((P',N')\),
\begin{equation*}
    \begin{gathered}
          \bigl\|\nabla L_{\rm pred}(P,N) -\nabla L_{\rm pred}(P',N')\bigr\|
\\  \le L\,\bigl\|\,(P,N)-(P',N')\bigr\|.
    \end{gathered}
\end{equation*}
Hence \(L_{\rm pred}\) is \(L\)-smooth.
\end{lemma}

\begin{proof}[High-Level]
Note that each GNN layer can be written as the composition of:
\begin{itemize}
  \item A linear map in \((P,N)\) (adjacency mask enters via \(A\odot P\), features via \(X+N\)), which is Lipschitz.
  \item An elementwise activation (e.g.\ ReLU or smooth ReLU), which is 1‑Lipschitz.
  \item A final differentiable loss \(L\), whose gradient is Lipschitz in the model’s output.
\end{itemize}
By the chain rule, the gradient of the overall map \((P,N)\mapsto L\bigl(f_\theta(\cdot),y\bigr)\) is Lipschitz, with constant
\[
  L \;\le\; L_{L}\,\prod_{\ell=1}^L L_{\text{layer},\ell},
\]
where \(L_{L}\) is the loss’s Lipschitz constant and each \(L_{\text{layer},\ell}\) upper‐bounds the Lipschitz constant of layer \(\ell\).  
\end{proof}

\paragraph{Remark.} In practice, we do not require the exact value of \(L\), only the existence of such a bound to invoke standard convergence results for projected gradient methods on smooth\,+\; non-smooth objectives.

\subsection{Choosing the Trade-off}
In practice, $\beta$ controls sparsity versus fidelity. A grid search typically selects $\beta \in [0.1,1]$ to yield few edits while guaranteeing $\Phi(G')\ne \Phi(G).$  The step size ($\alpha$ is $\eta$ in ~\Cref{alg:XPlore}) is as well selected via a grid search in $\{10^{-3}, 10^{-2}, 10^{-1}, 1\}$. 

\subsection{Computational Complexity.}\label{sec:complexity}
Each iteration of \Cref{alg:XPlore} performs:
\begin{itemize}
  \item one forward+backward pass through the GNN, which on a graph with $n$ nodes, $|E|$ edges, hidden‐dimensionality $d$ and feature‐dimensionality $f$ takes
  \[
    O\bigl(|E|\,d+n\,d\,f\bigr),
  \]
  since message‐passing scales linearly in edges and feature multiplications scale in $n\,f\times d$;
  \item elementwise thresholding of the perturbation masks $P$ and $N$, costing
  \[
    O\bigl(|E|+n\,f\bigr),
  \]
  as we only inspect each edge entry and each feature entry once. And the STE back-pass adds no asymptotic overhead. 
\end{itemize}
Thus, the dominant per‐iteration cost is 
\[
  O\bigl(|E|\,d+n\,d\,f\bigr),
\]
i.e.,  linear in the graph size under standard GNN architectures, ensuring the method scales to large graphs. Recall that $K$ total iterations are performed in ~\Cref{alg:XPlore}.
{Treating $d$ as a small constant, when the graph is dense $|E| \sim n^2 $, the complexity reduces to $O\bigl( n^2 \bigr)$.}

\subsection{Extension to Weighted and Directed Graphs}
Our theoretical developments extend straightforwardly to weighted or directed graphs.  For weighted graphs, replace the binary mask $P\in\{0,1\}^{n\times n}$ by a continuous mask $P\in[a,b]^{n\times n}$ (e.g.\ $[0,1]$), and substitute the unweighted $\ell_1$ distance 
\[
  \sum_{i,j}\bigl|p_{ij}-\bar p_{ij}\bigr|
\]
with a weighted version 
\[
  \sum_{i,j}w_{ij}\,\bigl|p_{ij}-\bar p_{ij}\bigr|,
\]
where $w_{ij}>0$ can reflect edge‐specific costs.  All projected gradient and subgradient arguments carry over by projecting onto the box $[a,b]$ and using the weighted sign subgradient for the $\ell_1$ term.  For directed graphs, simply treat $(i,j)$ and $(j,i)$ as distinct entries in $P$, and the same “flip when $|\partial L_{\rm pred}/\partial p_{ij}|>\beta\,w_{ij}$” threshold holds.  Convergence to a KKT‐stationary point and the $\ell_1$‐minimality bound remain valid with these modifications, since they rely only on box‐constraints and separable $\ell_1$ penalties.

\subsection{Search‐Space Expressivity}
By jointly parameterizing the adjacency perturbation $\bar P \in [0,1]^{n\times n}$ and feature perturbation $\bar N \in \mathbb{R}^{n\times f}$ (see ~\Cref{sec:Method,alg:XPlore}), our gradient‐based explainer can, in principle, reach \emph{any} discrete graph–feature configuration (once binarized via sigmoid and thresholding). In practice, however, the non-convex loss landscape can trap plain gradient descent in a local minimum within the basin around the original graph, favoring small, local edits. To counteract this, we add a small amount of Gaussian noise to $\bar P$ at initialization. These random perturbations diversify the initial gradient directions, helping the optimizer “hop” across low-gradient regions and leading to a richer exploration of structurally distinct counterfactuals without sacrificing convergence.\\
Hence, even when the optimization converges, the explainer may fail to reach a counterfactual due to either local minima or an ill-conditioned decision boundary in the oracle -- e.g., the true label region may be disjoint, vanishingly thin, or entirely absent near the input. This is reflected in our experiments: some generated counterfactuals either fail to flip the label or result in out-of-distribution inputs. These issues could be mitigated by improving the oracle’s inductive bias and generalization capacity. 

\section{Datasets}\label{sec:datasets}

To evaluate our approach against state-of-the-art GCE methods, we conducted experiments on 18 datasets (13 real and 5 synthetic) spanning diverse domains and multi-class graph classification tasks. \Cref{tab:datasetdescription,tab:datasetnodedescription} illustrate the characteristics of the datasets used in this paper for graph and node classification, respectively.

\begin{table*}[!h]
    \centering
    \caption{Graph explanation datasets characteristics.}
    \label{tab:datasetdescription}
    \resizebox{\textwidth}{!}{%
    \begin{tabular}{l c ccccccccccccccccc} \toprule
         & \quad TCR & TG & BAS &  MUTAG&BZR & COX2 & AIDS & BBBP   & ENZYMES& PROTEINS& Fingerprint& COLLAB& COLOR-3&TRIANGLES &  {MSRC}& {DBLP} & {IMDB} & {TWITTER} \\ \midrule
         Avg \# of Nodes & \quad 28 & 64 & 64 &   17.93&35.75&  41.22&  15.69&  25.95 & 32.63& 39.06& 7.06& 74.49& 61.31&20.85 &  77.52&10.48& 19.77& 4.03\\
         \rowcolor{graybckgrnd}Node Attr.& \quad -& 1& 4& $-$& 3& 3& 4& $-$& 18& 29& 2& $-$& 4&$-$ &  $-$ &$-$ & $-$ & $-$ \\ 
         Avg \# of Edges & \quad 27.75 & 65.01 & 64.01 &  19.79&38.36 &    43.45 &     16.20 &  24.06 & 62.14& 72.82& 5.76& 2457.78& 91.03&32.74 &  198.32&19.65& 96.53& 4.98\\
         \rowcolor{graybckgrnd}\# of Graphs & \quad 5000 & 5000 & 5000 &  188&405 & 467 &  2000&  2039 & 600& 1113& 2149& 5000& 10500&45000 &  563&19456& 1000& 144033\\
         \# of Classes & \quad 2 & 2 & 2 &  2&2 & 2 & 2 & 2  & 6& 2& 15& 3& 11&10 &  20&2& 2& 2\\
         \rowcolor{graybckgrnd}Motif & \quad Cycle & Grid & House &       $-$&$-$& $-$& $-$& $-$& $-$& $-$& $-$& $-$& Color&Triangle &  $-$ &$-$ & $-$ & $-$ \\
 Category& Synthetic& Synthetic& Synthetic& Molecular& Molecular& Molecular& Molecular& Molecular& Bioinf.& Bioinf.& Comp.Vis.& Social net.& Synthetic&Synthetic &  Comp.Vis.&Social net.& Social net.& Social net.\\
         \midrule
         Oracle Test. Acc. &  \quad 100.00\%&  99.72\%&  100.00\%&   88.30\%&99.01\%&  98.72\%&  99.80\%&  99.41\% & 95.67\%& 98.02\%& 74.97\%& 69.60\%& 91.68\%&99.18\% &  42.27\%&80.83\%& 77.80\%& 91.32\%\\
         \bottomrule
    \end{tabular}%
    }
\end{table*}

\begin{table}[!h]
    \centering
    \caption{{Node explanation datasets characteristics. }}
    \label{tab:datasetnodedescription}
    \resizebox{\linewidth}{!}{%
    \begin{tabular}{lccccc} \toprule
         & BAS &  BZR&ENZYMES& MSRC& TWITTER\\
         \# of Node labels& 2 &  10&3& 22& 1323\\
         \midrule
         Oracle Test. Acc. &  100.00\%&   99.90\%&87.81\%&  92.81\%&  3.41\%\\
         \bottomrule
    \end{tabular}%
    }
\end{table}

\textbf{AIDS} \citep{riesen2008iam} consists of graphs representing molecular compounds. These graphs are derived from the AIDS Antiviral Screen Database of Active Compounds. This data set consists of two classes (active and inactive) that represent molecules with or without HIV activity. Molecules are converted into graphs straightforwardly by representing atoms as nodes and covalent bonds as edges. Nodes are labeled with the number of the corresponding chemical symbol and edges by the valence of the linkage. There are 2,000 elements in total (1,600 inactive elements and 400 active elements).

\textbf{BAShapes} \citep{ying2019gnnexplainer} is a synthetic dataset consisting of a base graph and motifs connected to the base. The base graph is a Barabasi-Albert (BA) graph with a house-shaped motif attached. The resulting graph is further perturbed by adding $0.1N$ random edges. Following the generation described in \cite{lucic2022cf}, the base graph has 8 nodes and 5 edges. Each base graph has 7 motives connected to it.

\textbf{BBBP} (Blood-Brain Barrier Penetration) \citep{martins2012bayesian} is a widely used dataset in drug discovery and neurological research for developing machine learning models to predict blood-brain barrier permeability. The blood-brain barrier is a protective membrane that shields the central nervous system by regulating the passage of solutes. Its presence is a critical consideration in drug development, whether for designing molecules that target the central nervous system or for identifying compounds that should not cross the barrier. BBBP contains binary labels for 2,053 curated molecules, indicating whether a compound can penetrate the blood-brain barrier. Specifically, 1,570 molecules can penetrate the barrier, while 483 cannot.

\textbf{BZR} and \textbf{COX2} \citep{doi:10.1021/ci034143r} are distinct molecular targets: 467 cyclooxygenase-2 (COX-2) inhibitors and 405 benzodiazepine receptor (BZR) ligands, respectively. These datasets are widely used in quantitative structure-activity relationships (QSAR) studies, which aim to correlate the biological activities of compounds with their structural attributes to elucidate the mechanisms by which they act and predict the activities of novel derivatives.

\textbf{COLLAB} \citep{yanardag2015deep} is a social network dataset comprising 5000 scientific collaborations derived from 3 public collaboration datasets \citep{leskovec2005graphs}, namely, High Energy Physics, Condensed Matter Physics, and Astro Physics. Ego networks of researchers from each field were generated, and each graph was labeled with the researcher's field. The task is then to determine whether a researcher's egocollaboration graph belongs to the High Energy, Condensed Matter, or Astrophysics field.  

\textbf{COLORS-3} \citep{knyazev2019understandingattentiongeneralizationgraph} is a synthetic dataset consisting of 10500 random graphs with 11 classes where features of each node are assigned to one of the three colors (red, green, or blue), $\mathbf{p} \in \mathbb{R}^3$. The dataset is extended to higher $n$-dimensional cases $\mathbf{p} \in \mathbb{R}^n$.

\textbf{DBLP} \citep{pan2013graph} dataset consists of bibliography data in computer science. Each record is associated with a set of attributes, including abstract, authors, year, venue, title, and reference ID. The graph stream is built by selecting the list of conferences and using the papers published in these conferences (in chronological order) to form a binary-class graph stream. The classification task is to predict whether a paper belongs to the DBDM (database and data mining) or the CVPR (computer vision and pattern recognition) field using each paper's references and title.

\textbf{ENZYMES} \citep{borgwardt2005protein} is a bio-informatics dataset consisting of a dataset of 600 enzymes, constructed from the Protein Data Bank \citep{berman2000protein} and labeled with their corresponding enzyme class labels from the BRENDA enzyme database \citep{schomburg2004brenda}. It includes 100 proteins from each of six classes (EC 1, EC 2, EC 3, EC 4, EC 5, EC 6), which represent the six enzyme commission top-level hierarchy (EC classes). Proteins are converted into graphs by representing their secondary structure elements as nodes and edges in an attributed graph. Nodes are labeled with their type (helix, sheet, or loop) and their amino acid sequence. Every node is connected to its three nearest neighbors in space by an edge. Edges are labeled with their type and the distance they represent in angstroms. 

\textbf{Fingerprint} \citep{riesen2008iam} is a computer vision dataset consisting of 2149 fingerprints that are converted into graphs by filtering the images and extracting regions that are relevant \citep{neuhaus2005graph}. To obtain graphs from fingerprint images, the relevant regions are binarized, and a noise-removal and thinning procedure is applied. This results in a skeletonized representation of the extracted regions. Ending points and bifurcation points of the skeletonized regions are represented by nodes. Additional nodes are inserted in regular intervals between the ending points and the bifurcation points. Finally, undirected edges are inserted to link nodes that are directly connected through a ridge in the skeleton. Each node is labeled with a two-dimensional attribute giving its position. Each edge is assigned an angle indicating its orientation relative to the horizontal direction.

\textbf{IMDB} is a movie collaboration dataset in which actors/actresses and genre information of different movies on IMDB are collected. For each graph, nodes represent actors/actresses, and there is an edge between them if they appear in the same movie. Collaboration graphs on Action and Romance genres were generated, and ego-networks for each actor/actress were derived. Note that a movie can belong to both genres at the same time; therefore, movies from the Romance genre already included in the Action genre were discarded. Similarly to the COLLAB dataset, each ego-network was labeled with the genre graph to which it belongs. The task is then simply to identify the genre to which an ego-network graph belongs.

\textbf{MSRC21} \citep{neumann2016propagation} is derived from the a state-of-the-art dataset in semantic image processing originally introduced by \citet{winn2005object}. Each image is represented by a conditional Markov random field graph. The nodes of each graph are derived by oversegmenting the images using the QuickShift algorithm \citep{vedaldi2008quick}, yielding a single graph per image. Nodes are connected if the superpixels are adjacent, and each node can further be annotated with a semantic label. Semantic (ground-truth) node labels are derived by taking the mode ground-truth label of all pixels in the corresponding superpixel. Semantic labels are object names plus a label void to handle objects that do not fall into one of the given classes. Images consisting solely of a single semantic label were removed, leaving a classification task across 20 classes for the MSRC21 dataset.

\textbf{MUTAG} \citep{debnath1991structure} is a widely used dataset consisting of 188 nitroaromatic chemical compounds divided into two classes according to their mutagenic effect on a Salmonella typhimurium bacterium. Input graphs represent chemical compounds, vertices stand for atoms and are labeled by the atom type (represented by one-hot encoding), while edges between vertices represent bonds between the corresponding atoms.

\textbf{PROTEINS} \citep{borgwardt2005protein} is a bio-informatics dataset comprising 1113 proteins from the dataset of enzymes (59\%) and non-enzymes (41\%) \citep{dobson2003distinguishing}. Proteins are modeled as feature vectors which indicate for each amino acid its fraction among all residues, its fraction of the surface area, the existence of ligands, the size of the largest surface pocket, and the number of disulfide bonds. 

\textbf{Tree-Cycles (TCR)} \citep{ying2019gnnexplainer}  is an emblematic synthetic dataset.  Each instance constitutes a graph comprising a central tree motif and multiple cycle motifs connected through singular edges.  The dataset comprises two distinct classes: graphs without cycles (0) and graphs with cycles (1). The TC also allows control over the number of nodes, the number of cycles, and the number of nodes per cycle.

\textbf{Tree-Grid (TG)} is a synthetic dataset similar to \textit{Tree-Cycles}, in which n-by-n grid motifs are attached to the main tree motif in place of cycle motifs. We used 5000 graphs with 64 nodes and randomly attached a 3-by-3 grid. 

\textbf{TRIANGLES} \citep{knyazev2019understandingattentiongeneralizationgraph} is a synthetic dataset comprising 45000 graphs. The task is to count the number of triangles in the graph. Node degree features are added as one-hot vectors to all graphs, so the oracle model can exploit both graph structure and features. 

{\textbf{TWITTER} \citep{pan2014graph}\citep{pan2015cogboost} dataset is extracted from twitter sentiment classification. Because of Twitter's inherently short and sparse nature, Twitter sentiment analysis (i.e., predicting whether a tweet reflects a positive or negative sentiment) is a difficult task. To build a graph dataset, each tweet is represented as a graph using its content, with nodes denoting terms and/or smiley symbols and edges indicating co-occurrence between two words or symbols in each tweet. To ensure graph quality, only tweets containing 20 or more words were used. Tweets from April 6 to June 16 were selected to generate 140,949 graphs (in chronological order).}

\section{Detailed Experimental Setup}

\subsection{Oracles and training}\label{app:oracles}
Oracle Graph Convolutional Network (GCN) models for different datasets were trained with varying architectures and hyperparameters, all using RMSprop and CrossEntropyLoss, with an 80\%-20\% train-test split. Most datasets (BAS, AIDS, BZR, COX2, TCR, TG, and ENZYMES) used 3 convolutional layers and 1 dense layer, BBBP (5 conv, 3 dense), COX2, MUTAG (2 conv, 2 dense), TRIANGLES (5 conv, 2 dense), COLORS-3 (2 conv, 1 dense), COLLAB, PROTEINS (3 conv, 2 dense), Fingerprint (5 conv, 5 dense). Learning rates varied between 0.001 and 0.01, with training epochs ranging from 20 to 1000, and batch sizes typically 32, except BBBP, ENZYMES, PROTEINS (64), Fingerprint (128), COLLAB (256), TRIANGLES (1024). The experiments were conducted on an 8GB NVIDIA RTX 4060 GPU with an Intel Core i9-13900HX and 32GB RAM, while node explanation experiments used a 4GB NVIDIA RTX 3050 GPU with an Intel Core i7-11800H and 16GB RAM.

\subsection{Hyperparameters search}\label{app:hyperparams} 
We perform a hyperparameter search according to these combination of parameters: $\alpha \in \{0.001, 0.01, 0.1, 1\}, \beta \in \{0, 0.5, 1\}, K \in \{50, 100, 500, 5000\}, \gamma \in \{0, 0.1, 0.01, 0.001\} $. The optimal configuration found was $\alpha = 0.1, \beta = 0.5, K=50$ and $\gamma = 0$; if accuracy is preferred over other metrics (\Cref{sec:ablation}), then $\gamma = 0.01$ performed best over the other choices.

\subsection{Metrics}\label{app:metrics}

\textbf{Oracle Accuracy} tells us how accurate the model $\Phi$ is at predicting the ground truth labels, $\frac{1}{N} \sum_{i=1}^{N} \mathbf{1}[ \Phi(G_i) = y_i \bigr]$.
 
\noindent\textbf{Validity} illustrates the capability of the explainer to cross the decision boundary of $\Phi$ given an input $G_i$,  $\frac{1}{N}\sum_{i=1}^{N} \mathbf{1}[\Phi(G'_i) \neq \Phi(G_i)]$.

\noindent\textbf{Fidelity} Fidelity measures how faithful the explainer’s generated counterfactuals are to the original instance’s true labels, not just to the oracle’s learned decision boundary. $\frac{1}{N}\sum_{i=1}^{N} \max(\mathbf{1}(\Phi(G_i) = y_i) - \mathbf{1}[\Phi(G'_i) = y_i],0)$.

Note that if the oracle misclassifies an instance, the explainer may produce a counterfactual that changes the predicted class to the original class; in such cases, accuracy would indicate success even though the counterfactual is not a true solution, whereas fidelity correctly reflects this.

\noindent\textbf{Sparsity} measures the ratio between the number of structural features modified to obtain a valid counterfactual explanation and the number of structural features in the original instance.

\noindent\textbf{Graph Edit Distance (GED)} tracks the modification applied to the graph of the valid counterfactual explanation, i.e., the number of added/removed nodes/edges.

\noindent\textbf{Oracle Calls} is the number of times the explainer calls $\Phi$ to generate a valid counterfactual explanation.

\noindent \textbf{Runtime to counterfactual}: the time in seconds required by the explainer to generate a valid counterfactual explanation.
\Cref{tab:results_extended,tab:results_extended_new_row} shows all the performances of XPlore against the SoTA methods .~\Cref{fig:runtime} shows the runtime of all explainers, showcasing that XPlore is lightweight and only second to CLEAR and CF\textsuperscript{2}. However, it shows that XPlore is more efficient than its baseline CF-GNNExpl. which is, notice, more restrictive in perturbing the input graph (i.e., only allowing edge deletions). Conversely, XPlore allows both edge additions and removals, as well as node perturbations. {\Cref{fig:violin_plots} presents the CS and GED distributions, showing that XPlore retains high semantic meaning in the CFs it generates. XPlore performs better because it can also identify CFs in hard instances where competitors fail. Consequently, it achieves CS scores comparable to or better than those of other baselines on easy samples, while maintaining strong semantic meaning for hard instances that competitors miss.}

\begin{table*}[!t]
    \centering
    \caption{Extended comparison of XPlore with SoTA methods on the used datasets. Standard deviation is reported. Bold is best-performing; underline is second-best. Highlighted columns show the key metrics used to evaluate the explainer's faithfulness in generating counterfactuals. When an explainer cannot produce valid counterfactuals in a dataset, it does not make sense to evaluate the other metrics (see $-$).}
    \label{tab:results_extended}
\begin{minipage}{.48\linewidth}
\resizebox{\linewidth}{!}{%
        \begin{tabular}{@{}ll
>{\columncolor[HTML]{EFEFEF}}c 
>{\columncolor[HTML]{EFEFEF}}c cc@{}}
\toprule
Dataset &
  Method &
  \cellcolor[HTML]{EFEFEF}Validity$\uparrow$ &
  \cellcolor[HTML]{EFEFEF}Fidelity$\uparrow$ &
  Sparsity$\downarrow$ &
  Oracle Calls$\downarrow$ \\ \midrule
                        & iRand         & \cellcolor[HTML]{EFEFEF}27.92\% ± 44.86\% & \cellcolor[HTML]{EFEFEF}0.279 ± 0.449 & \textbf{0.062} ± 0.023 & 10.666 ± 0.084 \\
                        & CF$^\text{2}$ & \cellcolor[HTML]{EFEFEF}50.04\% ± 50.00\% & \cellcolor[HTML]{EFEFEF}0.500 ± 0.500  & 0.491 ± 0.000  & \textbf{0.000} ± 0.000 \\
                        & CLEAR         & \cellcolor[HTML]{EFEFEF}50.68\% ± 50.00\% & \cellcolor[HTML]{EFEFEF}0.507 ± 0.500 & 3.037 ± 0.193 & \textbf{0.000} ± 0.004 \\
                        & RSGG-CE       & \cellcolor[HTML]{EFEFEF}\underline{67.90\%} ± 46.69\% & \cellcolor[HTML]{EFEFEF}\underline{0.679} ± 0.467 & \underline{0.305} ± 0.057 & 8.890 ± 0.169 \\
                        & D4Explainer   & \cellcolor[HTML]{EFEFEF}44.82\% ± 49.73\% & \cellcolor[HTML]{EFEFEF}0.448 ± 0.497 & 0.491 ± 0.000 & \textbf{0.000} 0.012 \\ \cmidrule(l){2-6} 
                        & CF-GNNExpl    & \cellcolor[HTML]{EFEFEF}50.04\% ± 50.00\% & \cellcolor[HTML]{EFEFEF}0.500 ± 0.500  & 0.745 ± 0.000 & 46.000 ± 0.399\\
    \multirow{-7}{*}{\color{black}{TCR}} &
      XPlore & \cellcolor[HTML]{EFEFEF}\textbf{100.00\%} ± 0.00\% & \cellcolor[HTML]{EFEFEF}\textbf{1.000} ± 0.000 & 2.708 ± 0.906 & \underline{3.593} ± 0.009 \\ \bottomrule
    \end{tabular}
}
\end{minipage}%
    \hfill
\begin{minipage}{.48\linewidth}
\resizebox{\linewidth}{!}{%
        \begin{tabular}{@{}ll
>{\columncolor[HTML]{EFEFEF}}c 
>{\columncolor[HTML]{EFEFEF}}c cc@{}}
\toprule
Dataset &
  Method &
  \cellcolor[HTML]{EFEFEF}Validity$\uparrow$ &
  \cellcolor[HTML]{EFEFEF}Fidelity$\uparrow$ &
  Sparsity$\downarrow$ &
  Oracle Calls$\downarrow$ \\ \midrule
                        & iRand         & \cellcolor[HTML]{EFEFEF}36.16\% ± 48.05\% & \cellcolor[HTML]{EFEFEF}0.356 ± 0.485 & \textbf{0.175} ± 0.086 & 69.177 ± 0.0542 \\
                        & CF$^\text{2}$ & \cellcolor[HTML]{EFEFEF}49.86\% ± 50.00\% & \cellcolor[HTML]{EFEFEF}0.499 ± 0.500 & 0.496 ± 0.000 & \textbf{0.000} ± 0.001 \\
                        & CLEAR         & \cellcolor[HTML]{EFEFEF}58.40\% ± 49.29\% & \cellcolor[HTML]{EFEFEF}0.584 ± 0.493 & 6.353 ± 0.400 & \textbf{0.000} ± 0.003 \\
                        & RSGG-CE       & \cellcolor[HTML]{EFEFEF}\underline{89.28\%} ± 30.94\% & \cellcolor[HTML]{EFEFEF}\underline{0.888} ± 0.324 & 0.316 ± 0.091 & 19.152 ± 0.911 \\
                        & D4Explainer   & \cellcolor[HTML]{EFEFEF}49.86\% ± 50.00\% & \cellcolor[HTML]{EFEFEF}0.499 ± 0.500 & 0.500 ± 0.001 & \textbf{0.000} ± 0.013 \\ \cmidrule(l){2-6} 
                        & CF-GNNExpl    & \cellcolor[HTML]{EFEFEF}49.86\% ± 50.00\% & \cellcolor[HTML]{EFEFEF}0.499 ± 0.500 & 0.784 ± 0.000 & 46.000 ± 0.414 \\
\multirow{-7}{*}{TG} &
  XPlore & \cellcolor[HTML]{EFEFEF}\textbf{100.00\%} ± 0.00\% & \cellcolor[HTML]{EFEFEF}\textbf{0.994} ± 0.106 & \underline{0.248} ± 0.004 & \underline{14.692} ± 0.236 \\ \bottomrule
    \end{tabular}
}
\end{minipage}%
\hfill
\begin{minipage}{.48\linewidth}
\resizebox{\linewidth}{!}{%
        \begin{tabular}{@{}ll
>{\columncolor[HTML]{EFEFEF}}c 
>{\columncolor[HTML]{EFEFEF}}c cc@{}}
\toprule
Dataset &
  Method &
  \cellcolor[HTML]{EFEFEF}Validity$\uparrow$ &
  \cellcolor[HTML]{EFEFEF}Fidelity$\uparrow$ &
  Sparsity$\downarrow$ &
  Oracle Calls$\downarrow$ \\ \midrule
                        & iRand         & \cellcolor[HTML]{EFEFEF}50.70\% ± 50.00\% & \cellcolor[HTML]{EFEFEF}0.507 ± 0.500 & \textbf{0.112} ± 0.044 & 43.600 ± 0.293 \\
                        & CF$^\text{2}$ & \cellcolor[HTML]{EFEFEF}45.78\% ± 49.82\% & \cellcolor[HTML]{EFEFEF}0.457 ± 0.499 & 0.496 ± 0.000 & \textbf{0.000} ± 0.001 \\
                        & CLEAR         & \cellcolor[HTML]{EFEFEF}50.96\% ± 49.99\% & \cellcolor[HTML]{EFEFEF}0.510 ± 0.500 & 6.007 ± 0.325 & \textbf{0.000} ± 0.035 \\
                        & RSGG-CE       & \cellcolor[HTML]{EFEFEF}\underline{91.04\%} ± 28.56\% & \cellcolor[HTML]{EFEFEF}\underline{0.910} ± 0.286 & \underline{0.331} ± 0.108 & 23.585 ± 1.069 \\
                        & D4Explainer   & \cellcolor[HTML]{EFEFEF}44.56\% ± 49.70\%           & \cellcolor[HTML]{EFEFEF}0.446 ± 0.497 & 0.499 ± 0.002 & \textbf{0.000} ± 0.014 \\ \cmidrule(l){2-6} 
                        & CF-GNNExpl    & \cellcolor[HTML]{EFEFEF}44.18\% ± 49.66\% & \cellcolor[HTML]{EFEFEF}0.442 ± 0.497 & 0.748 ± 0.000 & 46.000 ± 0.283 \\
\multirow{-7}{*}{BAS} &
  XPlore &
  \cellcolor[HTML]{EFEFEF}\textbf{100.00\%} ± 0.00\% & \cellcolor[HTML]{EFEFEF}\textbf{1.000} ± 0.000 & 7.600 ± 7.557 & \underline{2.514} ± 0.047 \\ \bottomrule
    \end{tabular}
}
\end{minipage}%
\hfill
\begin{minipage}{.48\linewidth}
\resizebox{\linewidth}{!}{%
        \begin{tabular}{@{}ll
>{\columncolor[HTML]{EFEFEF}}c 
>{\columncolor[HTML]{EFEFEF}}c cc@{}}
\toprule
Dataset &
  Method &
  \cellcolor[HTML]{EFEFEF}Validity$\uparrow$ &
  \cellcolor[HTML]{EFEFEF}Fidelity$\uparrow$ &
  Sparsity$\downarrow$ &
  Oracle Calls$\downarrow$ \\ \midrule
                        & iRand         & 2.66\% ± 16.09 & 0.005 ± 0.163 & \textbf{0.035} ± 0.011 & 4.200 ± 0.004 \\
                        & CF$^\text{2}$ & 0.00\% ± 0.00\% & 0.005 ± 0.318 & $-$& $-$\\
                        & CLEAR         & \cellcolor[HTML]{EFEFEF}35.11\% ± 47.73\% & \cellcolor[HTML]{EFEFEF}0.309 ± 0.506 & 2.438 ± 0.375 & \textbf{0.000} ± 0.001 \\
                        & RSGG-CE       & \cellcolor[HTML]{EFEFEF}\underline{56.91\%} ± 49.52\% & \cellcolor[HTML]{EFEFEF}\underline{0.516} ± 0.550 & \underline{0.264} ± 0.004 & \underline{2.000} ± 0.002 \\
                        & D4Explainer   & 9.57\% ± 29.42\% & 0.021 ± 0.309 & 0.526 ± 0.016 & \textbf{0.000} ± 0.007 \\ \cmidrule(l){2-6} 
                        & CF-GNNExpl    & 10.11\% ± 30.14\% & 0.005 ± 0.318 & 0.758 ± 0.005 & 46.000 ± 0.147 \\
\multirow{-7}{*}{MUTAG} & XPlore   & \textbf{67.55\%} ± 46.82\% & \textbf{0.548} ± 0.613 & 0.459 ± 0.036 & 2.866 ± 0.089 \\ \bottomrule
    \end{tabular}
}
\end{minipage}%
\hfill
\begin{minipage}{.48\linewidth}
\resizebox{\linewidth}{!}{%
        \begin{tabular}{@{}ll
>{\columncolor[HTML]{EFEFEF}}c 
>{\columncolor[HTML]{EFEFEF}}c cc@{}}
\toprule
Dataset &
  Method &
  \cellcolor[HTML]{EFEFEF}Validity$\uparrow$ &
  \cellcolor[HTML]{EFEFEF}Fidelity$\uparrow$ &
  Sparsity$\downarrow$ &
  Oracle Calls$\downarrow$ \\ \midrule
                        & iRand         & 27.16\% ± 44.48\% & 0.262 ± 0.451 & \textbf{0.104} ± 0.045 & 25.373 ± 0.213 \\
                        & CF$^\text{2}$ & 19.75\% ± 39.81\% & 0.188 ± 0.403 & 0.516 ± 0.007 & \textbf{0.000} ± 0.002 \\
                        & CLEAR         & \underline{60.49\%} ± 48.89\% & \underline{0.595} ± 0.501 & 4.937 ± 0.843 & \textbf{0.000} ± 0.012 \\
                        & RSGG-CE       & 21.23\% ± 40.90\% & 0.202 ± 0.414 & \underline{0.258} ± 0.003 & \underline{2.000} ± 0.015 \\
                        & D4Explainer   & 20.00\% ± 40.00\% & 0.190 ± 0.405 & 0.523 ± 0.006 & \textbf{0.000} ± 0.015 \\ \cmidrule(l){2-6} 
                        & CF-GNNExpl    & 19.75\% ± 39.81\% & 0.188 ± 0.403 & 0.758 ± 0.003 & 46.000 ± 0.224 \\
\multirow{-7}{*}{BZR}   & XPlore   & \textbf{100.00\%} ± 0.00\% & \textbf{0.980} ± 0.198 & 6.595 ± 3.616 & 2.244 ± 0.029  \\ \bottomrule
    \end{tabular}
}
    \end{minipage}%
    \hfill
\begin{minipage}{.48\linewidth}
\resizebox{\linewidth}{!}{%
        \begin{tabular}{@{}ll
>{\columncolor[HTML]{EFEFEF}}c 
>{\columncolor[HTML]{EFEFEF}}c cc@{}}
\toprule
Dataset &
  Method &
  \cellcolor[HTML]{EFEFEF}Validity$\uparrow$ &
  \cellcolor[HTML]{EFEFEF}Fidelity$\uparrow$ &
  Sparsity$\downarrow$ &
  Oracle Calls$\downarrow$ \\ \midrule
                        & iRand         & \underline{69.81}\% ± 45.91\% & \underline{0.677} ± 0.490 & \textbf{0.087} ± 0.049 & 22.482 ± 0.192 \\
                        & CF$^\text{2}$ & 24.20\% ± 42.83\% & 0.229 ± 0.435 & 4.354 ± 0.235 & \textbf{0.000} ± 0.005 \\
                        & CLEAR         & 22.06\% ± 41.46\% & 0.208 ± 0.421 & \underline{0.512} ± 0.001 & \textbf{0.000} ± 0.002 \\
                        & RSGG-CE       & \textbf{99.36\%} ± 7.99\% & \textbf{0.968} ± 0.238 & 0.761 0.366 & 92.914 ± 1.102 \\
                        & D4Explainer   & 22.06\% ± 41.46\% & 0.208 ± 0.421 & 0.518 ± 0.003 & \textbf{0.000} ± 0.025 \\ \cmidrule(l){2-6} 
                        & CF-GNNExpl    & 22.06\% ± 41.46\% & 0.208 ± 0.421 & 0.756 ± 0.001 & 46.000 0.198 \\
\multirow{-7}{*}{COX2}  & XPlore   & \textbf{99.36\%} ± 7.99\% & \textbf{0.968} ± 0.238 & 7.682 3.704 & \underline{11.388} ± 0.568 \\ \bottomrule
    \end{tabular}
}
    \end{minipage}%
    \hfill
\begin{minipage}{.48\linewidth}
\resizebox{\linewidth}{!}{%
        \begin{tabular}{@{}ll
>{\columncolor[HTML]{EFEFEF}}c 
>{\columncolor[HTML]{EFEFEF}}c cc@{}}
\toprule
Dataset &
  Method &
  \cellcolor[HTML]{EFEFEF}Validity$\uparrow$ &
  \cellcolor[HTML]{EFEFEF}Fidelity$\uparrow$ &
  Sparsity$\downarrow$ &
  Oracle Calls$\downarrow$ \\ \midrule
                        & iRand         & 0.00\% ± 0.00\% & 0.000 ± 0.000                                   & $-$              & $-$              \\
                        & CF$^\text{2}$ & 0.10\% ± 3.16\%                               & 0.001 ± 0.032                      & 3.870 ± 0.011     & \textbf{0.000} ± 0.004 \\
                        & CLEAR         & \cellcolor[HTML]{EFEFEF}16.75\% ± 37.34\% & \cellcolor[HTML]{EFEFEF}0.168 ± 0.373    & 13.442 ± 1.834  & \textbf{0.000} ± 0.008 \\
                        & RSGG-CE       & \underline{19.80\%} ± 39.85\%   & \underline{0.198} ± 0.398   & \textbf{0.258} ± 0.004& \underline{2.000} ± 0.027  \\
                        & D4Explainer   & 0.10\% ± 0.032\%   & 0.001 ± 0.032  & \underline{0.488} ± 0.031& \textbf{0.000} ± 0.001 \\ \cmidrule(l){2-6} 
                        & CF-GNNExpl    & 0.10\% ± 3.16\%                            & 0.001 ± 0.032                  & 0.745 ± 0.005      & 46.000 ± 0.059        \\
                       \multirow{-7}{*}{AIDS} & XPlore   & \textbf{32.30\%} ± 45.96\%& \textbf{0.323} ± 0.460& 2.445 ± 2.002& 5.851± 0.058\\ \bottomrule
    \end{tabular}
}
    \end{minipage}%
    \hfill
    \begin{minipage}{.48\linewidth}
\resizebox{\linewidth}{!}{%
        \begin{tabular}{@{}ll
>{\columncolor[HTML]{EFEFEF}}c 
>{\columncolor[HTML]{EFEFEF}}c cc@{}}
\toprule
Dataset &
  Method &
  \cellcolor[HTML]{EFEFEF}Validity$\uparrow$ &
  \cellcolor[HTML]{EFEFEF}Fidelity$\uparrow$ &
  Sparsity$\downarrow$ &
  Oracle Calls$\downarrow$ \\ \midrule
                        & iRand         & 19.76\% ± 39.82\% & \underline{0.275} ± 0.460 & \textbf{0.059} ± 0.037 & 12.928 ± 0.190 \\
                        & CF$^\text{2}$ & \underline{25.26\%} ± 43.45\% & 0.253 ± 0.434 & 0.510 ± 0.024  & \textbf{0.000} ± 0.002 \\
                        & CLEAR         & \cellcolor[HTML]{EFEFEF}22.90\% ± 42.02\% & \cellcolor[HTML]{EFEFEF}0.229 ± 0.420 & 51.889 ± 32.182 & \textbf{0.000} ± 0.022 \\
                        & RSGG-CE       & \cellcolor[HTML]{EFEFEF}22.90\% ± 42.02\% & \cellcolor[HTML]{EFEFEF}0.229 ± 0.420 & \underline{0.258} ± 0.006 & \underline{2.000} 8.131    \\
                        & D4Explainer   & 21.24\% ± 40.90\% & 0.212 ± 0.409 & 0.524 ± 0.010 & \textbf{0.000} ± 0.038 \\ \cmidrule(l){2-6} 
                        & CF-GNNExpl    & 22.31\% ± 41.64\% & 0.223 ± 0.416 & 0.758 ± 0.005 & 46.000 ± 0.185 \\
\multirow{-7}{*}{BBBP}  & XPlore   & \textbf{81.51\%} \% 38.82\% & \textbf{0.803} ± 0.412 & 2.217 ± 1.206 & 4.055 ± 0.030 \\ \bottomrule
    \end{tabular}
}
    \end{minipage}
\begin{minipage}{.48\linewidth}
\resizebox{\linewidth}{!}{%
        \begin{tabular}{@{}ll
>{\columncolor[HTML]{EFEFEF}}c 
>{\columncolor[HTML]{EFEFEF}}c cc@{}}
\toprule
Dataset &
  Method &
  \cellcolor[HTML]{EFEFEF}Validity$\uparrow$ &
  \cellcolor[HTML]{EFEFEF}Fidelity$\uparrow$ &
  Sparsity$\downarrow$ &
  Oracle Calls$\downarrow$ \\ \midrule
                        & iRand         & 26.67\% ± 44.22\% & 0.238 ± 0.445 & \textbf{0.074} ± 0.046 & 29.069 ± 0.046 \\
                        & CF$^\text{2}$ & 68.33\% ± 46.52\% & 0.633 ± 0.502 & 0.657 ± 0.034 & \textbf{0.000} ± 0.004 \\
                        & CLEAR         & 83.17\% ± 37.42\% & 0.797 0.402 & 27.942 ± 24.021 & \textbf{0.000} ± 0.017 \\
                        & RSGG-CE       & \underline{98.33}\% ± 12.80\% & \underline{0.930} ± 0.292\% & \underline{0.396} ± 0.152 & 19.400 ± 2.801 \\
                        & D4Explainer   & 68.00\% ± 46.65\% & 0.632 ± 0.499 & 0.663 ± 0.499 & \textbf{0.000} ± 0.066 \\ \cmidrule(l){2-6} 
                        & CF-GNNExpl    & 68.33\% ± 46.52\% & 0.633 ± 0.502 & 0.657 ± 0.034 & \underline{2.000} ± 0.002 \\
\multirow{-7}{*}{ENZYMES}  & XPlore   &  \textbf{100.00\%} ± 0.00\% & \textbf{0.942} ± 0.291 & 2.394 ± 1.538 & 2.615 ±  \\ \bottomrule
    \end{tabular}
}
    \end{minipage}
    \hfill
\begin{minipage}{.48\linewidth}
\resizebox{\linewidth}{!}{%
        \begin{tabular}{@{}ll
>{\columncolor[HTML]{EFEFEF}}c 
>{\columncolor[HTML]{EFEFEF}}c cc@{}}
\toprule
Dataset &
  Method &
  \cellcolor[HTML]{EFEFEF}Validity$\uparrow$ &
  \cellcolor[HTML]{EFEFEF}Fidelity$\uparrow$ &
  Sparsity$\downarrow$ &
  Oracle Calls$\downarrow$ \\ \midrule
                        & iRand         & 18.87\% ± 39.13\% & 0.183 ± 0.394 & \textbf{0.107} ± 0.078 & 75.524 ± 4.036 \\
                        & CF$^\text{2}$ & 16.35\% ± 36.98\% & 0.151 ± 0.375 & 0.651 ± 0.025 & \textbf{0.000} ± 0.001 \\
                        & CLEAR         & 0.00\% ± 0.000\% & 0.000 ± 0.000& $-$& $-$\\
                        & RSGG-CE       & \underline{58.67}\% 49.24\% & \underline{0.556} ± 0.527 & 0.664 ± 0.477 & 78.757 ± 3.667 \\
                        & D4Explainer   & 0.00\% ± 0.00\% & 0.000 ± 0.000 & $-$& $-$\\ \cmidrule(l){2-6} 
                        & CF-GNNExpl    & 16.35\% ± 36.98\% & 0.151 ± 0.375 & 0.651 ± 0.025 & \underline{2.000} ± 0.002 \\
\multirow{-7}{*}{PROTEINS}  & XPlore   &  \textbf{65.41\%} ± 47.57\% & \textbf{0.627} ± 0.511 & \underline{0.569} ± 0.039 & 2.798 ± 0.006 \\ \bottomrule
    \end{tabular}
}
    \end{minipage}
    \hfill
\begin{minipage}{.48\linewidth}
\resizebox{\linewidth}{!}{%
        \begin{tabular}{@{}ll
>{\columncolor[HTML]{EFEFEF}}c 
>{\columncolor[HTML]{EFEFEF}}c cc@{}}
\toprule
Dataset &
  Method &
  \cellcolor[HTML]{EFEFEF}Validity$\uparrow$ &
  \cellcolor[HTML]{EFEFEF}Fidelity$\uparrow$ &
  Sparsity$\downarrow$ &
  Oracle Calls$\downarrow$ \\ \midrule
                        & iRand         & 0.09\% ± 3.05\% & 0.001 ± 0.030\% & \textbf{0.031} ± 0.000 & 3.500 ± 0.007 \\
                        & CF$^\text{2}$ & 24.52\% ± 43.02\% & 0.133 ± 0.400 & 0.412 ± 0.052 & \textbf{0.000} ± 0.000 \\
                        & CLEAR         & 72.73\% ± 44.53\% & 0.515 ± 0.537 & 15.517 ± 14.400 & \textbf{0.000} ± 0.001 \\
                        & RSGG-CE       & \underline{90.46}\% ± 29.38\% & \underline{0.653} ± 0.566 & \underline{0.329} ± 0.142 & 6.479 ± 0.048 \\
                        & D4Explainer   & 24.52\% ± 43.02\% & 0.133 ± 0.400 & 0.412 ± 0.052 & \textbf{0.000} ± 0.014 \\ \cmidrule(l){2-6} 
                        & CF-GNNExpl    & 24.52\% ± 43.02\% & 0.133 ± 0.400 & 0.412 ± 0.052 & \underline{2.000} ± 0.002 \\
\multirow{-7}{*}{Fingerprint}  & XPlore   &  \textbf{100.00\%} ± 0.00\%& \textbf{0.730} ± 0.487& 0.872 ± 0.591& 4.168 ± 0.052\\ \bottomrule
    \end{tabular}
}
    \end{minipage}
    \hfill
\begin{minipage}{.48\linewidth}
\resizebox{\linewidth}{!}{%
        \begin{tabular}{@{}ll
>{\columncolor[HTML]{EFEFEF}}c 
>{\columncolor[HTML]{EFEFEF}}c cc@{}}
\toprule
Dataset &
  Method &
  \cellcolor[HTML]{EFEFEF}Validity$\uparrow$ &
  \cellcolor[HTML]{EFEFEF}Fidelity$\uparrow$ &
  Sparsity$\downarrow$ &
  Oracle Calls$\downarrow$ \\ \midrule
                        & iRand         & \underline{95.74\%} ± 20.20\% & \underline{0.391} ± 0.656 & \textbf{0.027} ± 0.039 & 24.731 ± 0.383 \\
                        & CF$^\text{2}$ & 90.94\% ± 28.70\% & 0.371 ± 0.523 & 0.719 ± 0.009 & \textbf{0.000} ± 0.000 \\
                        & CLEAR         & 23.98\% ± 42.70\% & 0.012 ± 0.158 & 10.630 ± 1.095 & \textbf{0.000} ± 0.034 \\
                        & RSGG-CE       & \textbf{100.00\%} ± 0.00\% & \textbf{0.423} ± 0.538 & \underline{0.368} ± 0.050 & 7.536 ± 1.666 \\
                        & D4Explainer   & 90.94\% ± 28.70\% & 0.371 ± 0.523 & 0.721 ± 0.009 & \textbf{0.000} ± 0.047 \\ \cmidrule(l){2-6} 
                        & CF-GNNExpl    & 90.94\% ± 28.70\% & 0.371 ± 0.523 & 0.859 ± 0.004 & 46.000 ± 0.078 \\
\multirow{-7}{*}{{MSRC}}  & XPlore   &  \textbf{100.00\%} ± 0.00\% & \textbf{0.423} ± 0.569 & 5.059 ± 0.756 & \underline{2.089} ± 0.006 \\ \bottomrule
    \end{tabular}
}
    \end{minipage}
    \hfill
\begin{minipage}{.48\linewidth}
\resizebox{\linewidth}{!}{%
        \begin{tabular}{@{}ll
>{\columncolor[HTML]{EFEFEF}}c 
>{\columncolor[HTML]{EFEFEF}}c cc@{}}
\toprule
Dataset &
  Method &
  \cellcolor[HTML]{EFEFEF}Validity$\uparrow$ &
  \cellcolor[HTML]{EFEFEF}Fidelity$\uparrow$ &
  Sparsity$\downarrow$ &
  Oracle Calls$\downarrow$ \\ \midrule
                        & iRand         & 4.60\% ± 20.95\% & 0.006 ± 0.213 & \textbf{0.018} ± 0.021 & 58.696 ± 2.505 \\
                        & CF$^\text{2}$ & \underline{52.66\%} ± 49.93\% & \underline{0.262} ± 0.624 & 0.886 ± 0.061 & \textbf{0.000} ± 0.003 \\
                        & CLEAR         & 0.00\% ± 0.00\% & 0.000 ± 0.000 & $-$& $-$\\
                        & RSGG-CE       & 0.00\% ± 0.00\% & 0.000 ± 0.000 & $-$& $-$\\
                        & D4Explainer   & 0.00\% ± 0.00\% & 0.000 ± 0.000 & $-$& $-$\\ \cmidrule(l){2-6} 
                        & CF-GNNExpl    & \underline{52.66\%} ± 49.93\% & \underline{0.262} ± 0.624 & 0.886 ± 0.061 & \underline{2.000} ± 0.004 \\
\multirow{-7}{*}{COLLAB}  & XPlore   &  \textbf{100.00\%} ± 0.00\% & \textbf{0.570} ± 0.709& \underline{0.810} ± 0.068 & 3.604 ± 0.021 \\ \bottomrule
    \end{tabular}
}
    \end{minipage}
    \hfill
\begin{minipage}{.48\linewidth}
\resizebox{\linewidth}{!}{%
        \begin{tabular}{@{}ll
>{\columncolor[HTML]{EFEFEF}}c 
>{\columncolor[HTML]{EFEFEF}}c cc@{}}
\toprule
Dataset &
  Method &
  \cellcolor[HTML]{EFEFEF}Validity$\uparrow$ &
  \cellcolor[HTML]{EFEFEF}Fidelity$\uparrow$ &
  Sparsity$\downarrow$ &
  Oracle Calls$\downarrow$ \\ \midrule
                        & iRand         & 1.17\% ± 10.74\% & 0.006 ± 0.108 & \textbf{0.032} ± 0.019 & \underline{8.780} ± 0.018 \\
                        & CF$^\text{2}$ & 5.76\% ± 23.29\% & 0.027 ± 0.240 & 0.530 ± 0.158 & \textbf{0.000} ± 0.001 \\
                        & CLEAR         & 0.68\% ± 8.21\% & 0.005 ± 0.082 & 3.528 ± 1.037 & \textbf{0.000} ± 0.003 \\
                        & RSGG-CE       & \underline{57.51\%} ± 49.43\% & \underline{0.450} ± 0.686 & \underline{0.419} ± 0.225 & 21.848 ± 0.844 \\
                        & D4Explainer   & 7.02\% ± 25.55\% & 0.031 ± 0.265 & 0.599 ± 0.089 & \textbf{0.000} ± 0.021 \\ \cmidrule(l){2-6} 
                        & CF-GNNExpl    & 5.75\% ± 23.29\% & 0.027 ± 0.240 & 0.765 ± 0.079 & 46.000 ± 0.035 \\
\multirow{-7}{*}{{DBLP}}  & XPlore   &  \textbf{91.84\%} ± 27.37\% & \textbf{0.748} ± 0.765 & 0.569 ± 0.278 & 36.115 ± 0.097 \\ \bottomrule
    \end{tabular}
}
    \end{minipage}
    \hfill

\end{table*}

\begin{table*}[!t]
    \centering
    \caption{Extended comparison of XPlore with SoTA methods on the used datasets. Standard deviation is reported. Bold is best-performing; underline is second-best. Highlighted columns show the key metrics used to evaluate the explainer's faithfulness in generating counterfactuals. When an explainer cannot produce valid counterfactuals in a dataset, it does not make sense to evaluate the other metrics (see $-$).}
    \label{tab:results_extended_new_row}
\begin{minipage}{.48\linewidth}
\resizebox{\linewidth}{!}{%
        \begin{tabular}{@{}ll
>{\columncolor[HTML]{EFEFEF}}c 
>{\columncolor[HTML]{EFEFEF}}c cc@{}}
\toprule
Dataset &
  Method &
  \cellcolor[HTML]{EFEFEF}Validity$\uparrow$ &
  \cellcolor[HTML]{EFEFEF}Fidelity$\uparrow$ &
  Sparsity$\downarrow$ &
  Oracle Calls$\downarrow$ \\ \midrule
                        & iRand         & 3.70\% ± 18.88\% & 0.026 ± 0.192 & \textbf{0.024} ± 0.013 & 10.865 ± 0.048 \\
                        & CF$^\text{2}$ & 50.60\% ± 50.00\% & 0.366 ± 0.674 & 0.803 ± 0.038 & \textbf{0.000} ± 0.001 \\
                        & CLEAR         & 56.20\% ± 49.61\% & 0.420 ± 0.696 & 34.773 ± 15.650 & \textbf{0.000} ± 0.022 \\
                        & RSGG-CE       & \textbf{86.10\%} ± 34.60\% & \textbf{0.664} ± 0.802 & 0.462 ± 0.122 & 34.890 ± 1.795\\
                        & D4Explainer   & 49.60\% ± 50.00\% & 0.358 ± 0.669 & 0.797 ± 0.037 & \textbf{0.000} ± 0.015 \\ \cmidrule(l){2-6} 
                        & CF-GNNExpl    & 50.60\% ± 50.00\% & 0.366 ± 0.674 & 0.902 ± 0.019 & 46.000 ± 0.230 \\
\multirow{-7}{*}{{IMDB}}  & XPlore   &  \underline{78.80} ± 40.87\% & \underline{0.606} ± 0.780 & \underline{0.390} ± 0.398 & \underline{5.042} ± 0.047 \\ \bottomrule
    \end{tabular}
}
    \end{minipage}
    \hfill
\begin{minipage}{.48\linewidth}
\resizebox{\linewidth}{!}{%
        \begin{tabular}{@{}ll
>{\columncolor[HTML]{EFEFEF}}c 
>{\columncolor[HTML]{EFEFEF}}c cc@{}}
\toprule
Dataset &
  Method &
  \cellcolor[HTML]{EFEFEF}Validity$\uparrow$ &
  \cellcolor[HTML]{EFEFEF}Fidelity$\uparrow$ &
  Sparsity$\downarrow$ &
  Oracle Calls$\downarrow$ \\ \midrule
                        & iRand         & 0.00\% ± 0.00\% & 0.000 ± 0.000 & $-$ & $-$ \\
                        & CF$^\text{2}$ & 38.82\% ± 48.74\% & 0.345 ± 0.545 & 0.476 ± 0.111 & \textbf{0.000} ± 0.001 \\
                        & CLEAR         & 7.65\% ± 26.58\% & 0.039 ± 0.277 & 2.352 ± 0.246 & \textbf{0.000} ± 0.003 \\
                        & RSGG-CE       & \underline{48.36\%} ± 49.97\% & \underline{0.434} ± 0.579 & \underline{0.238} ± 0.056 & 2.000 ± 0.003 \\
                        & D4Explainer   & 23.45\% ± 42.37\% & 0.161 ± 0.476 & 0.386 ± 0.157 & \textbf{0.000} ± 0.012 \\ \cmidrule(l){2-6} 
                        & CF-GNNExpl    & 38.82\% ± 48.74\% & 0.345 ± 0.545 & 0.476 ± 0.111 & \underline{1.000} ± 0.061 \\
\multirow{-7}{*}{{TWITTER}}  & XPlore   &  \textbf{50.71\%} ± 50.00\% & \textbf{0.613} ± 0.626 & \textbf{0.217} ± 0.143 & \underline{1.000} ± 0.027 \\ \bottomrule
    \end{tabular}
}
    \end{minipage}
    \hfill
\begin{minipage}{.48\linewidth}
\resizebox{\linewidth}{!}{%
        \begin{tabular}{@{}ll
>{\columncolor[HTML]{EFEFEF}}c 
>{\columncolor[HTML]{EFEFEF}}c cc@{}}
\toprule
Dataset &
  Method &
  \cellcolor[HTML]{EFEFEF}Validity$\uparrow$ &
  \cellcolor[HTML]{EFEFEF}Fidelity$\uparrow$ &
  Sparsity$\downarrow$ &
  Oracle Calls$\downarrow$ \\ \midrule
                        & iRand         & 42.99\% ± 49.51\% & 0.300 ± 0.571 & \textbf{0.076} ± 0.064 & 66.175 ± 62.151 \\
                        & CF$^\text{2}$ & 52.07\% ± 49.96\% & 0.398 ± 0.568 & 0.598 ± 0.046 & \textbf{0.000} ± 0.001 \\
                        & CLEAR         & 0.00\% ± 0.00\% & 0.000 ± 0.000 & $-$& $-$\\
                        & RSGG-CE       & \underline{94.57}\% ± 22.66\% & \underline{0.824} ± 0.471 & \underline{0.312} ± 0.079 & 4.040 ± 0.848 \\
                        & D4Explainer   & 0.00\% ± 0.00\% & 0.000 ± 0.000 & $-$& $-$\\ \cmidrule(l){2-6} 
                        & CF-GNNExpl    & 52.07\% ± 49.96\% & 0.398 ± 0.568 & 0.598 ± 0.046 & \underline{2.000} ±0.007 \\
\multirow{-7}{*}{COLORS-3}  & XPlore   &  \textbf{100.00\%} ± 0.00\% & \textbf{0.871} ± 0.453 & 0.536 ± 0.057 & 12.204 ± 0.155 \\ \bottomrule
    \end{tabular}
}
    \end{minipage}
    \hfill
\begin{minipage}{.48\linewidth}
\resizebox{\linewidth}{!}{%
        \begin{tabular}{@{}ll
>{\columncolor[HTML]{EFEFEF}}c 
>{\columncolor[HTML]{EFEFEF}}c cc@{}}
\toprule
Dataset &
  Method &
  \cellcolor[HTML]{EFEFEF}Validity$\uparrow$ &
  \cellcolor[HTML]{EFEFEF}Fidelity$\uparrow$ &
  Sparsity$\downarrow$ &
  Oracle Calls$\downarrow$ \\ \midrule
                        & iRand         & 6.38\% ± 24.44\% & 0.053 ± 0.247 & \textbf{0.123} ± 0.080 & 71.421 ± 0.346 \\
                        & CF$^\text{2}$ & 37.13\% ± 48.31\% & 0.364 ± 0.488 & 0.621 0.032 & \textbf{0.000} ± 0.001 \\
                        & CLEAR         & 89.99\% ± 30.01\% & 0.892 ± 0.311 & 53.517 ± 36.786 & \textbf{0.000} ± 0.019 \\
                        & RSGG-CE       & \underline{99.84}\% ± 4.00\% & \underline{0.987} ± 0.141 & \underline{0.325} ± 0.080 & 5.161 ± 0.378 \\
                        & D4Explainer   & 31.11\% ± 46.29\% & 0.302 ± 0.468 & 0.619 ± 0.038 & \textbf{0.000} ± 76.890 \\ \cmidrule(l){2-6} 
                        & CF-GNNExpl    & 37.13\% ± 48.31\% & 0.364 ± 0.488 & 0.621 ± 0.032 & 46.000 ± 0.056 \\
\multirow{-7}{*}{TRIANGLES}  & XPlore   &  \textbf{100.00\%} ± 0.00\% & \textbf{0.988} ± 0.138 & 0.532 ± 0.053 & \underline{3.366} ± 0.026 \\ \bottomrule
    \end{tabular}
}

    \end{minipage}
    \hfill
\end{table*}

\begin{figure*}[!t]
    \centering
    \begin{subfigure}{\textwidth}
        \centering
        \includegraphics[width=\linewidth]{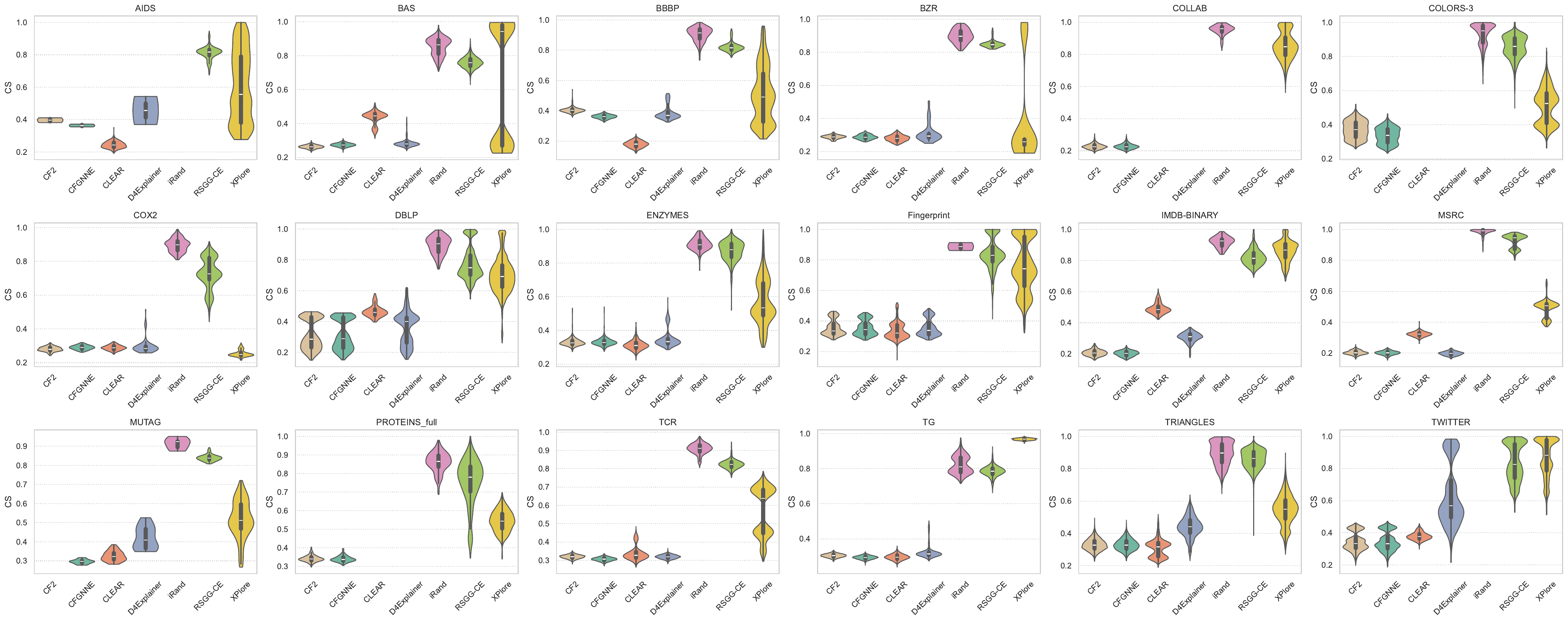}
    \caption{{Cosine similarity distributions comparison.}}
    \label{fig:CS-violin}
    \end{subfigure}
    \begin{subfigure}{\textwidth}
        \centering
        \includegraphics[width=\linewidth]{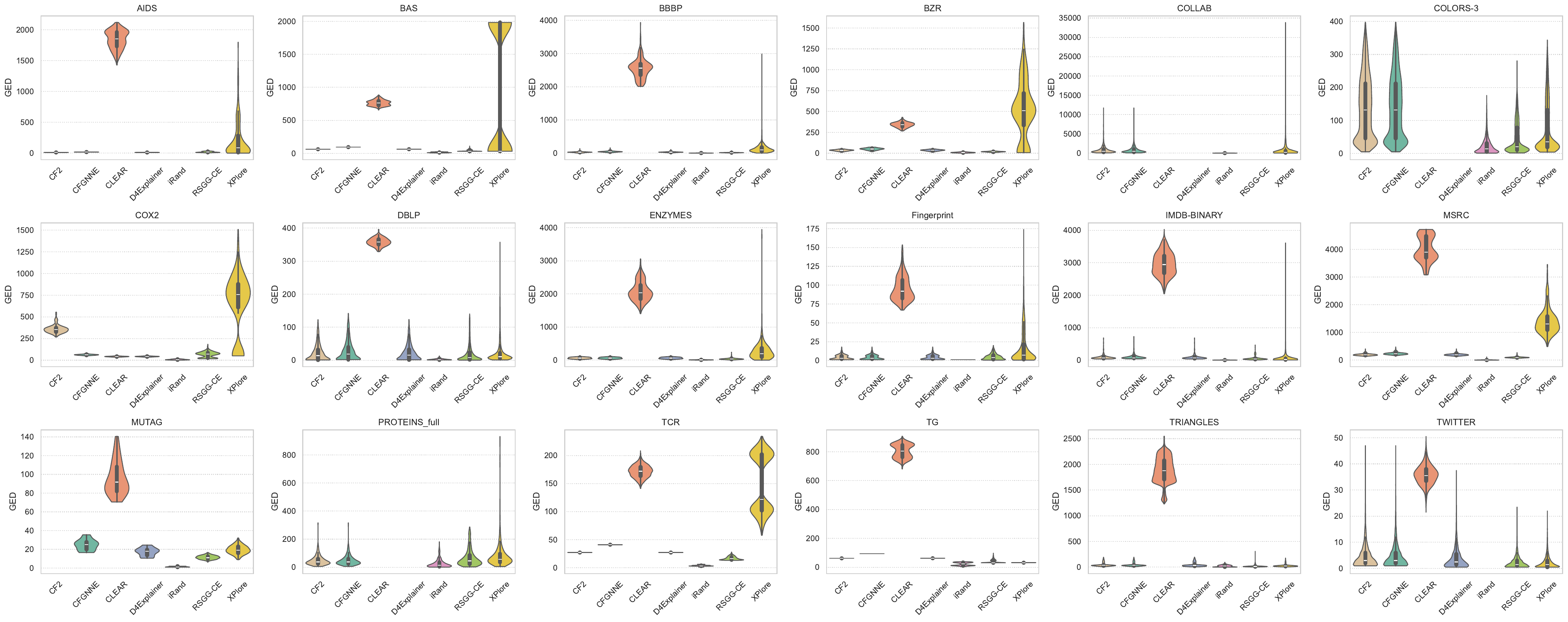}
    \caption{{Graph edit distance distributions comparison.}}
        \label{fig:GED-violin}
    \end{subfigure}
    \caption{{XPlore identifies more CFs, thus capturing harder instances. Elongated and bimodal distributions reflect competitive performance and semantic meaning retention.}}
    \label{fig:violin_plots}
\end{figure*}

\subsection{Out-of-distribution effect analysis}\label{app:OOD}
{
\paragraph{Comparison of XPlore and RSGG: t-SNE projection of Wavelet Characteristic embeddings (Tree-Cycle dataset)}
We can compare XPlore's behaviour with that of the second-best explainer on the Tree-Cycle dataset. Depending on the original instance, we can see distinct behaviours of the two explainers taking place, let's analyze them.

\begin{figure*}[!t]
    \centering
        \includegraphics[width=\linewidth]{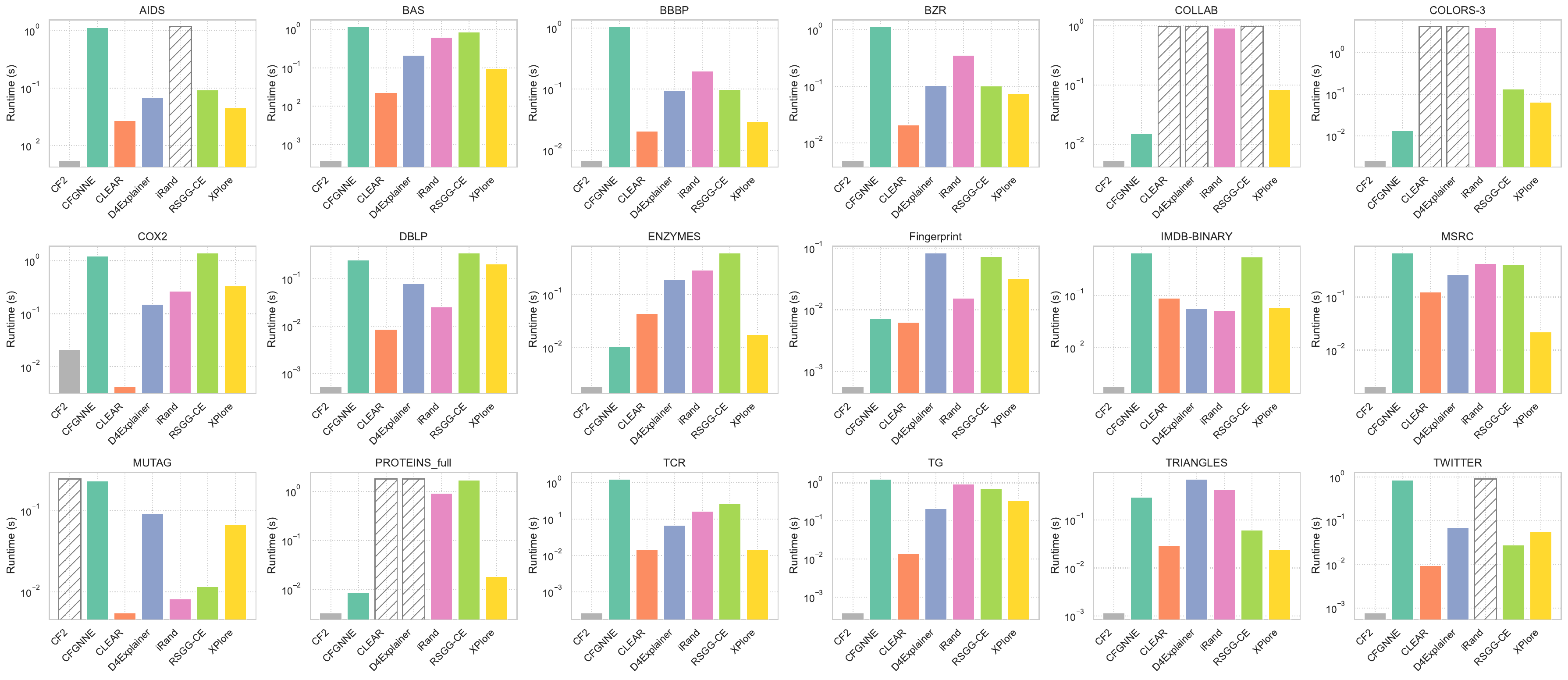}
    \caption{Runtime (s) at inference across multiple datasets. \textbf{XPlore (yellow) maintains a competitive runtime, making it the most reliable explainer in practice.} Placeholders are used for explainers with 0\% accuracy.}
    \label{fig:runtime}
\end{figure*}
\subparagraph{XPlore unable to land on Tree distribution for a Cycle CF}
XPlore sometimes is unable to land on the Tree distribution for a Cycle CF (\Cref{fig:OOD-1a}). \ref{fig:OOD-1b} is a significant error by XPlore. RSGG is sometimes unable to correctly land on the Tree distribution (\Cref{fig:OOD-1c}).

\subparagraph{XPlore strengths, RSGG weaknesses}
Looking at counterfactuals for the Tree class, the results show an \textbf{opposite trend}, XPlore is able to \textbf{correctly land CFs} on the Cycle distribution, while RSGG struggles to generate in-distribution CFs, not always landing on the Cycle distribution: (\Cref{fig:OOD-1d}); yet sometimes, both succeed (\Cref{fig:OOD-1e}).

\paragraph{Comparison of XPlore, CF-GNNExpl, and RSGG: t-SNE projection of Wavelet Characteristic embeddings (Tree-Cycle dataset)}
We now compare together CFs for XPlore, CF-GNNExpl, and RSGG. Here, we show only CFs for the Tree class, as we plot only instances where all 3 explainers were able to find a counterfactual, and CF-GNNExpl has success only in finding CFs for this class.

\subparagraph{CF-GNNExpl does not land on Cycle distribution}
CF-GNNExpl does not land on Cycle distribution (\Cref{fig:OOD-2a}). This is expected, as \textbf{CF-GNNExpl only removes edges}, so it cannot produce Cycles. Note how \textbf{XPlore and RSGG behave similarly}. Sometimes, \textbf{XPlore is the only explainer finding a solution} (\Cref{fig:OOD-2b}). Rarely, the embedder gets tricked (we know CF-GNNExpl cannot produce cycles) (\ref{fig:OOD-2c}).

\begin{figure*}[!t]
    \centering
    \begin{subfigure}{0.4\textwidth}
        \centering
        \includegraphics[width=\linewidth]{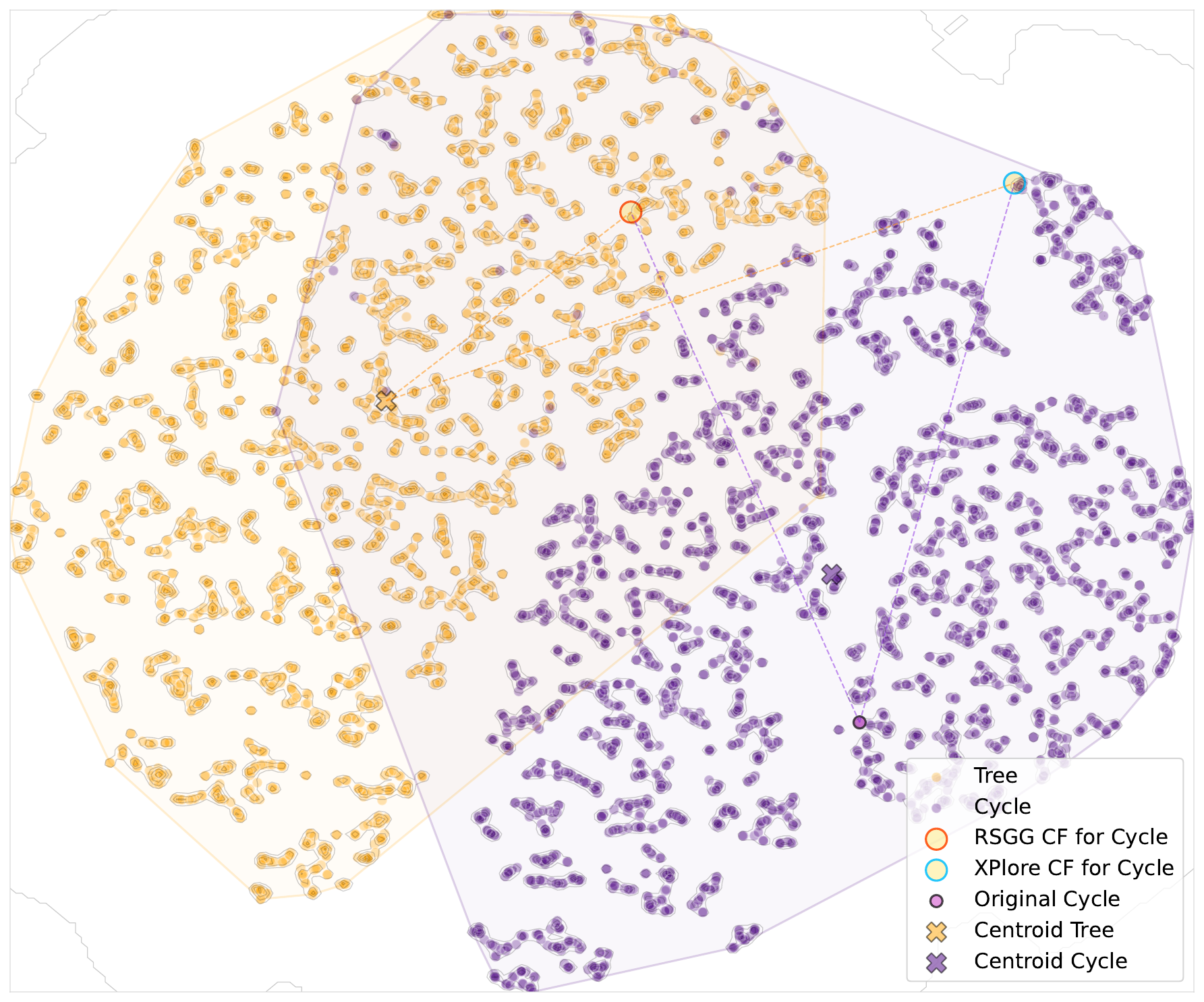}
    \caption{{}}
    \label{fig:OOD-1a}
    \end{subfigure}
    \begin{subfigure}{0.4\textwidth}
        \centering
        \includegraphics[width=\linewidth]{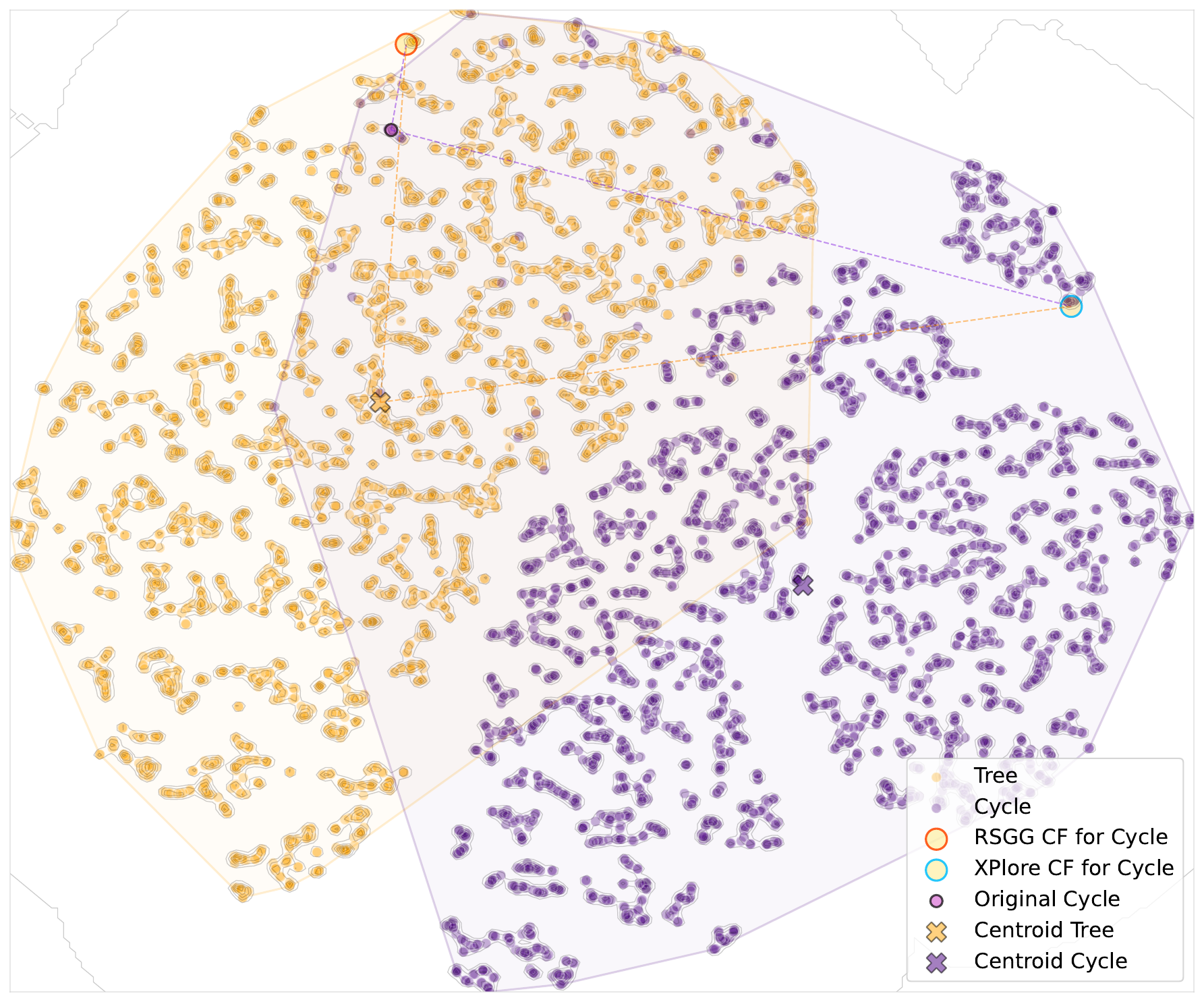}
    \caption{{}}
    \label{fig:OOD-1b}
    \end{subfigure}
    \begin{subfigure}{0.4\textwidth}
        \centering
        \includegraphics[width=\linewidth]{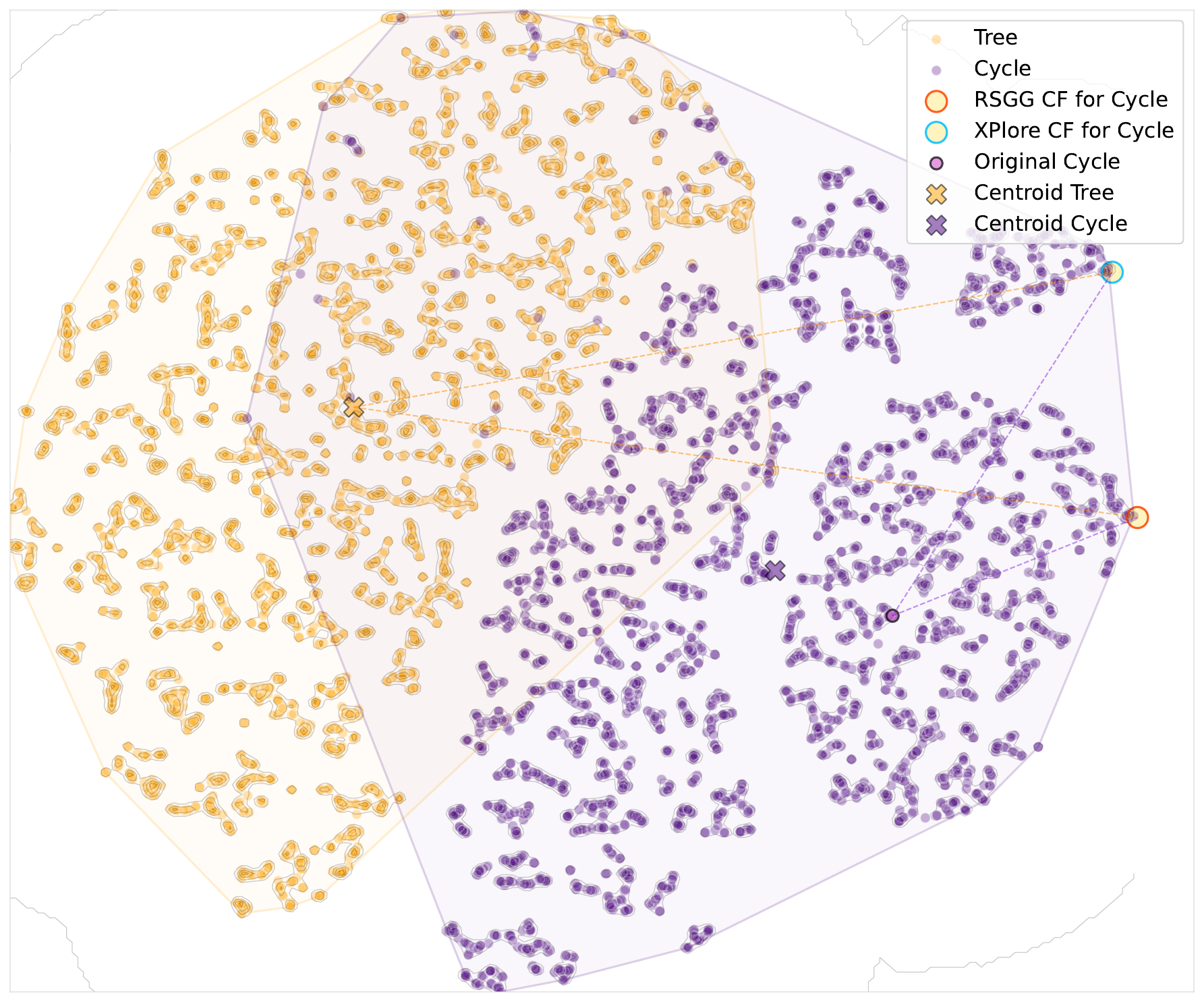}
    \caption{{}}
    \label{fig:OOD-1c}
    \end{subfigure}
    \begin{subfigure}{0.4\textwidth}
        \centering
        \includegraphics[width=\linewidth]{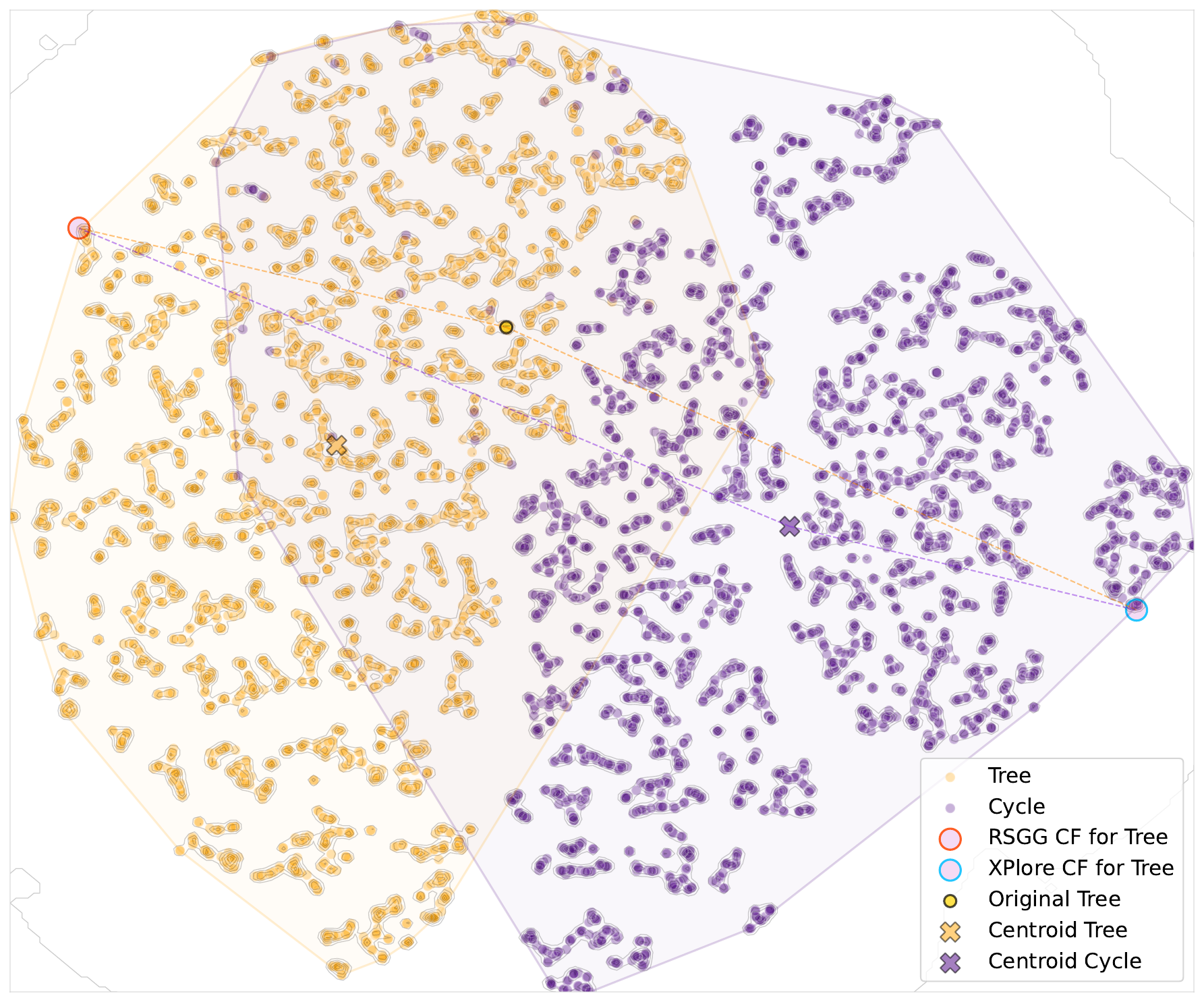}
    \caption{{}}
    \label{fig:OOD-1d}
    \end{subfigure}
    \begin{subfigure}{0.4\textwidth}
        \centering
        \includegraphics[width=\linewidth]{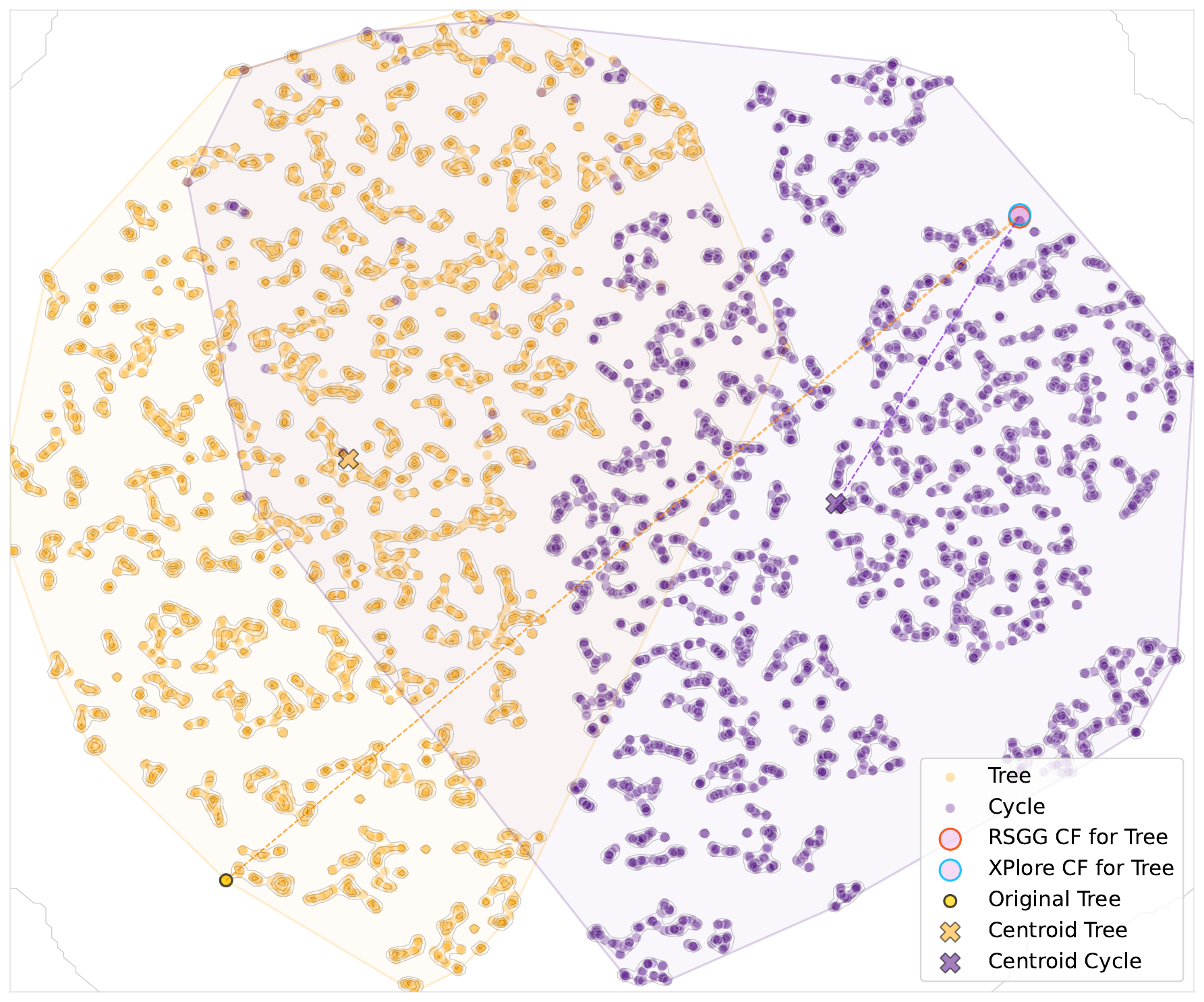}
    \caption{{}}
    \label{fig:OOD-1e}
    \end{subfigure}
    \caption{{Comparison of XPlore and RSGG: t-SNE projection of Wavelet Characteristic embeddings (Tree-Cycle dataset)}}
    \label{fig:OOD-1}
\end{figure*}

\begin{figure*}[!t]
    \centering
    \begin{subfigure}{0.49\textwidth}
        \centering
        \includegraphics[width=\linewidth]{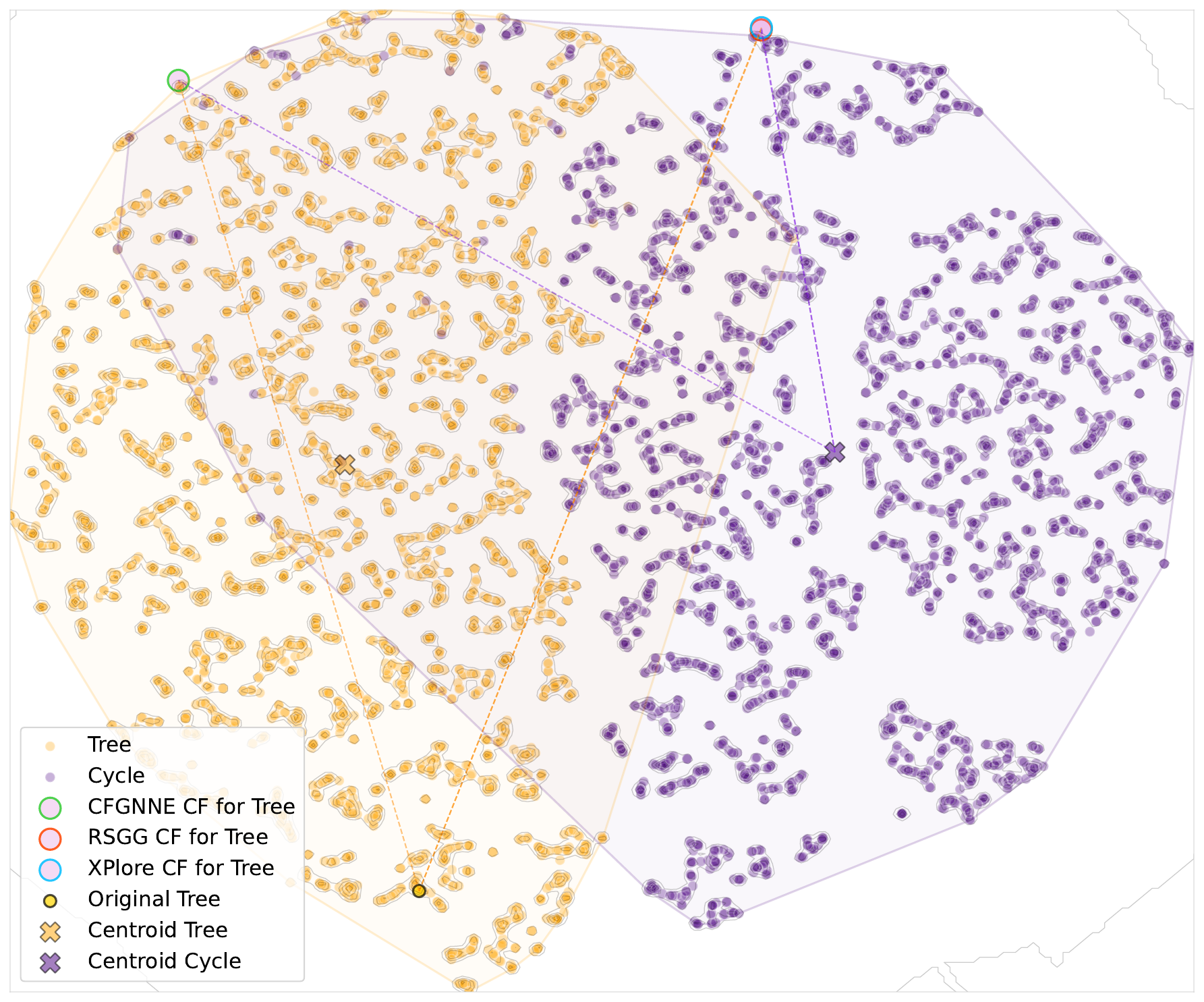}
    \caption{{}}
    \label{fig:OOD-2a}
    \end{subfigure}
    \begin{subfigure}{0.49\textwidth}
        \centering
        \includegraphics[width=\linewidth]{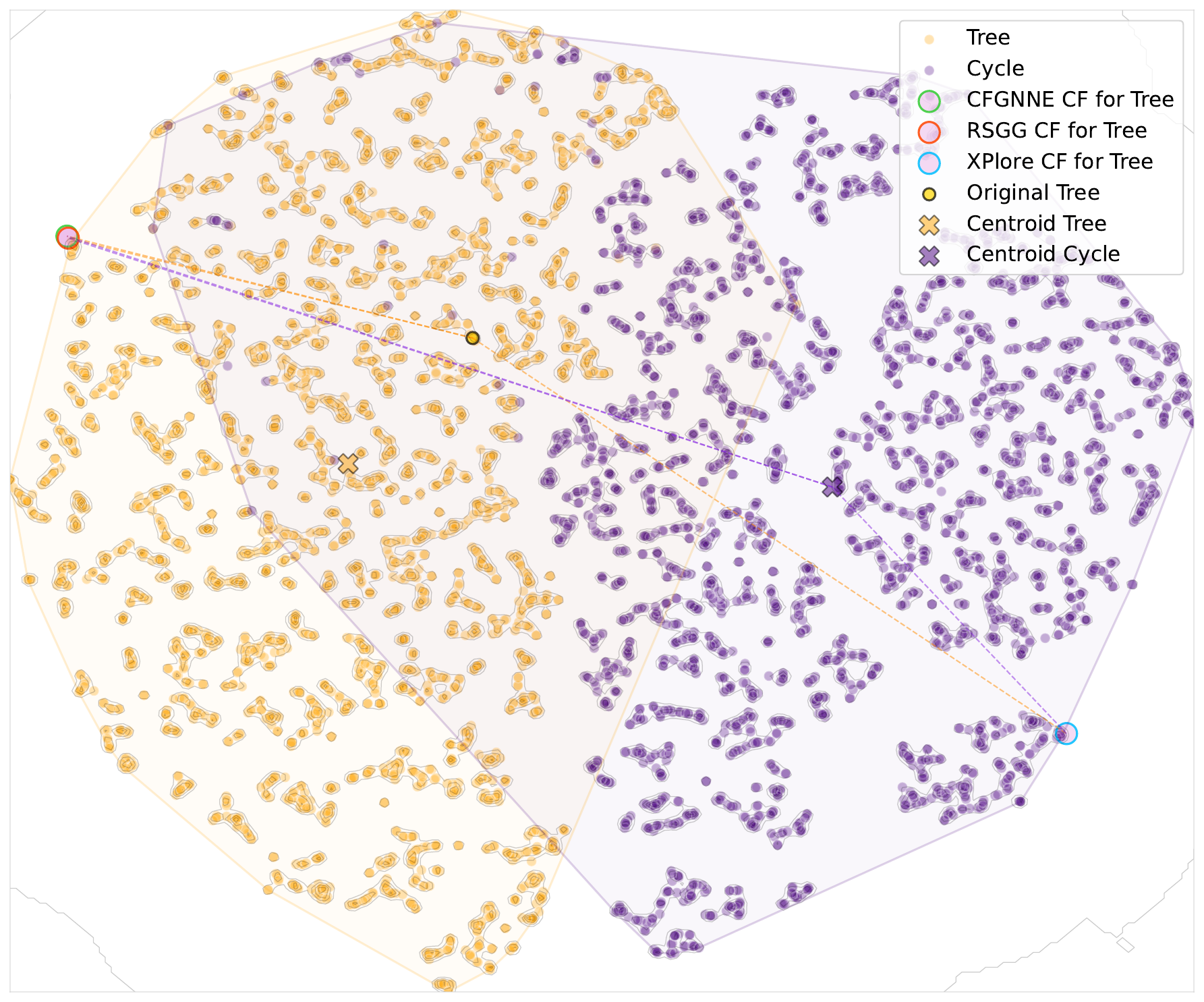}
    \caption{{}}
    \label{fig:OOD-2b}
    \end{subfigure}
    \begin{subfigure}{0.49\textwidth}
        \centering
        \includegraphics[width=\linewidth]{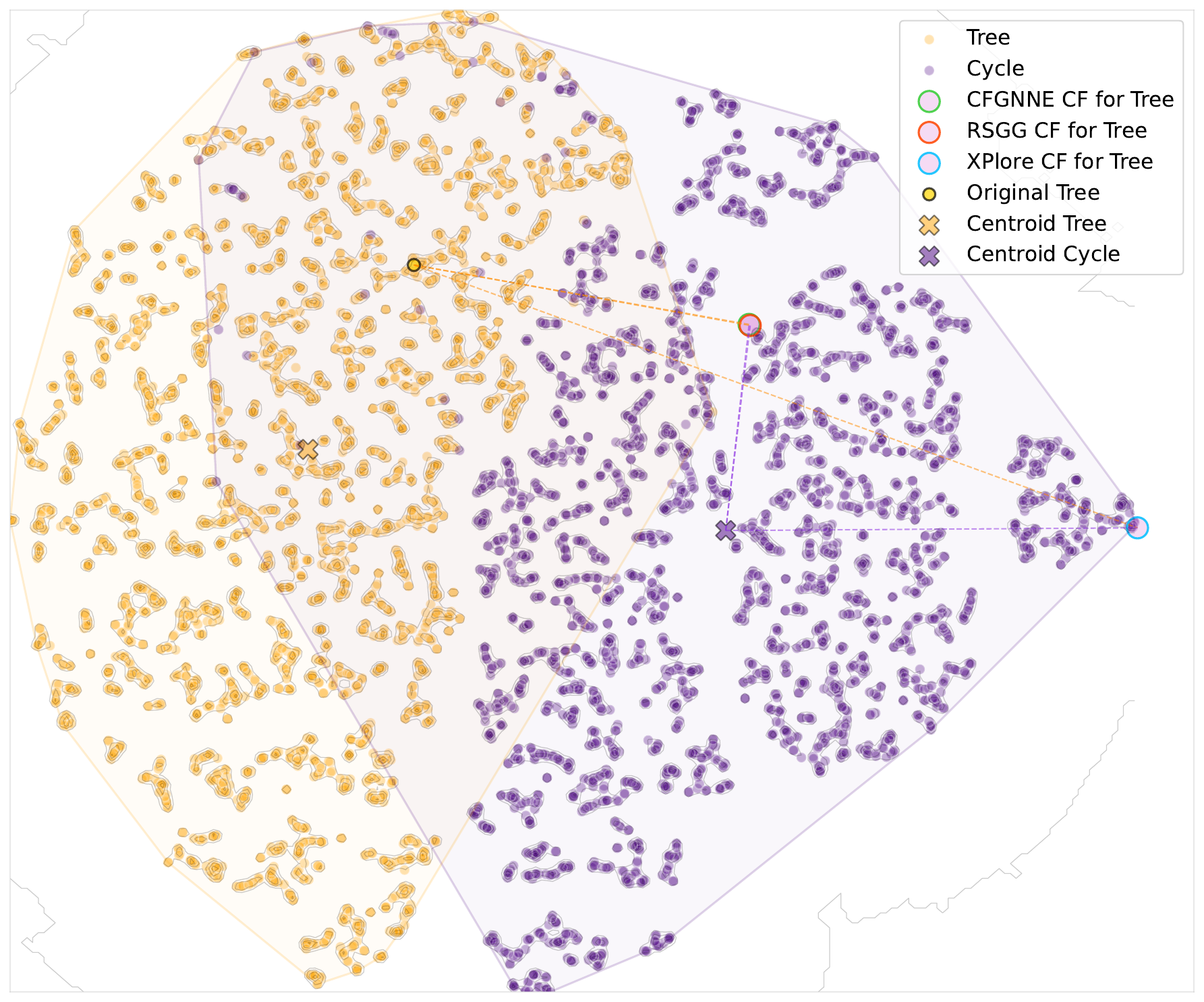}
    \caption{{}}
    \label{fig:OOD-2c}
    \end{subfigure}
    \caption{{Comparison of XPlore, CF-GNNExpl and RSGG: t-SNE projection of Wavelet Characteristic embeddings (Tree-Cycle dataset)}}
    \label{fig:OOD-2}
\end{figure*}

\end{document}

%% file: math_commands.tex
\usepackage{amsmath,amsfonts,bm}

\def\eqref#1{equation~\ref{#1}}

\def\1{\bm{1}}

\DeclareMathAlphabet{\mathsfit}{\encodingdefault}{\sfdefault}{m}{sl}
\SetMathAlphabet{\mathsfit}{bold}{\encodingdefault}{\sfdefault}{bx}{n}



%% file: bibliography.bib
@String{Computing = "Computing" }

@String{Computer = "{IEEE} Computer" }

@String{Springer = "Springer-Verlag" }

@inproceedings{leemann2024towards,
  title={Towards Non-adversarial Algorithmic Recourse},
  author={Leemann, Tobias and Pawelczyk, Martin and Prenkaj, Bardh and Kasneci, Gjergji},
  booktitle={World Conference on Explainable Artificial Intelligence},
  pages={395--419},
  year={2024},
  organization={Springer}
}

@inproceedings{prado2024robust,
  title={Robust Stochastic Graph Generator for Counterfactual Explanations},
  author={Prado-Romero, Mario Alfonso and Prenkaj, Bardh and Stilo, Giovanni},
  booktitle={Proc. of the AAAI Conf. on Artificial Intelligence},
  volume={38},
  number={19},
  pages={21518--21526},
  year={2024}
}

@article{ying2019gnnexplainer,
  title={Gnnexplainer: Generating explanations for graph neural networks},
  author={Ying, Z. and Bourgeois, D. and You, J. and Zitnik, M. and Leskovec, J.},
  journal={Adv. in Neural Inf. Processing Systems},
  volume={32},
  year={2019}
}

@article{10.1145/3554981,
  title={Dual subgraph-based graph neural network for friendship prediction in location-based social networks},
  author={Wei, Xuemei and Liu, Yezheng and Sun, Jianshan and Jiang, Yuanchun and Tang, Qifeng and Yuan, Kun},
  journal={ACM Trans. on Knowl. Disc. from Data},
  year={2022},
  publisher={ACM New York, NY}
}

@article{prado2022survey,
author = {Prado-Romero, Mario Alfonso and Prenkaj, Bardh and Stilo, Giovanni and Giannotti, Fosca},
title = {A Survey on Graph Counterfactual Explanations: Definitions, Methods, Evaluation, and Research Challenges},
year = {2023},
publisher = {Association for Computing Machinery},
address = {New York, NY, USA},
issn = {0360-0300},
journal = {ACM Comput. Surv.},
month = {sep}
}

@article{nguyen2022explaining,
  title={Explaining Black Box Drug Target Prediction through Model Agnostic Counterfactual Samples},
  author={Nguyen, Tri Minh and Quinn, Thomas P and Nguyen, Thin and Tran, Truyen},
  journal={IEEE/ACM Trans. on Compt. Biology and Bioinformatics},
  year={2022},
  publisher={IEEE}
}

@inproceedings{ma2022clear,
title={{CLEAR}: Generative Counterfactual Explanations on Graphs},
author={Jing Ma and Ruocheng Guo and Saumitra Mishra and Aidong Zhang and Jundong Li},
booktitle={Adv. in Neural Inf. Processing Systems},
year={2022}
}

@inproceedings{numeroso2021meg,
  title={Meg: Generating molecular counterfactual explanations for deep graph networks},
  author={Numeroso, Danilo and Bacciu, Davide},
  booktitle={2021 Int. Joint Conf. on Neural Networks},
  pages={1--8},
  year={2021},
  organization={IEEE}
}

@article{guidotti2018survey,
  title={A survey of methods for explaining black box models},
  author={Guidotti, Riccardo and Monreale, Anna and Ruggieri, Salvatore and Turini, Franco and Giannotti, Fosca and Pedreschi, Dino},
  journal={ACM computing surveys (CSUR)},
  volume={51},
  number={5},
  pages={1--42},
  year={2018},
  publisher={ACM New York, NY, USA}
}

@inproceedings{lucic2022cf,
  title={Cf-gnnexplainer: Counterfactual explanations for graph neural networks},
  author={Lucic, A. and Ter Hoeve, M.A. and Tolomei, G. and De Rijke, M. and Silvestri, F.},
  booktitle={Int. Conf. on AI and Statistics},
  pages={4499--4511},
  year={2022},
  organization={PMLR}
}

@inproceedings{ding2019effective,
  title={Effective feature learning with unsupervised learning for improving the predictive models in massive open online courses},
  author={Ding, Mucong and Yang, Kai and Yeung, Dit-Yan and Pong, Ting-Chuen},
  booktitle={Proc. of the 9th Int. Conf. on Learning Analytics \& Knowl.},
  pages={135--144},
  year={2019}
}

@article{vermeire2022explainable,
  title={Explainable image classification with evidence counterfactual},
  author={Vermeire, Tom and Brughmans, Dieter and Goethals, Sofie and de Oliveira, Raphael Mazzine Barbossa and Martens, David},
  journal={Pattern Analysis and Applications},
  volume={25},
  number={2},
  pages={315--335},
  year={2022},
  publisher={Springer}
}

@inproceedings{liu2021multi,
  title={Multi-objective Explanations of GNN Predictions},
  author={Liu, Yifei and Chen, Chao and Liu, Yazheng and Zhang, Xi and Xie, Sihong},
  booktitle={2021 IEEE Int. Conf. on Data Mining (ICDM)},
  pages={409--418},
  year={2021},
  organization={IEEE}
}

@inproceedings{tancf2,
author = {Tan, Juntao and Geng, Shijie and Fu, Zuohui and Ge, Yingqiang and Xu, Shuyuan and Li, Yunqi and Zhang, Yongfeng},
title = {Learning and Evaluating Graph Neural Network Explanations Based on Counterfactual and Factual Reasoning},
year = {2022},
booktitle = {Proc. of the ACM Web Conf. 2022},
pages = {1018–1027},
numpages = {10},
location = {Virtual Event, Lyon, France},
series = {WWW '22}
}

@inproceedings{aragona2021coronna,
  title={CoRoNNa: a deep sequential framework to predict epidemic spread},
  author={Aragona, Dario and Podo, Luca and Prenkaj, Bardh and Velardi, Paola},
  booktitle={Proc. of the 36th Annual ACM Symposium on Applied Computing},
  pages={10--17},
  year={2021}
}

@article{verenich2019predicting,
  title={Predicting process performance: A white-box approach based on process models},
  author={Verenich, I. and Dumas, M. and La Rosa, M. and Nguyen, H.},
  journal={Journal of Software: Evolution and Process},
  volume={31},
  number={6},
  pages={e2170},
  year={2019},
  publisher={Wiley Online Library}
}

@article{8882211,  author={Loyola-González, O.},  journal={IEEE Access},   title={Black-Box vs. White-Box: Understanding Their Advantages and Weaknesses From a Practical Point of View},   year={2019},  volume={7},  number={},  pages={154096-154113}}

@article{petch2021opening,
  title={Opening the black box: the promise and limitations of explainable machine learning in cardiology},
  author={Petch, J. and Di, S. and Nelson, W.},
  journal={Canadian Journal of Cardiology},
  year={2021},
  publisher={Elsevier}
}

@InProceedings{Zemni_2023_CVPR,
    author    = {Zemni, Mehdi and Chen, Micka\"el and Zablocki, \'Eloi and Ben-Younes, H\'edi and P\'erez, Patrick and Cord, Matthieu},
    title     = {OCTET: Object-Aware Counterfactual Explanations},
    booktitle = {Proc. of the IEEE/CVF Conf. on Computer Vision and Pattern Recognition (CVPR)},
    month     = {June},
    year      = {2023},
    pages     = {15062-15071}
}

@article{chen2023d4explainerindistributiongnnexplanations,
  title={D4explainer: In-distribution explanations of graph neural network via discrete denoising diffusion},
  author={Chen, Jialin and Wu, Shirley and Gupta, Abhijit and Ying, Rex},
  journal={Adv. in Neural Inf. Processing Systems},
  volume={36},
  pages={78964--78986},
  year={2023}
}

@inproceedings{riesen2008iam,
  title={IAM graph database repository for graph based pattern recognition and machine learning},
  author={Riesen, Kaspar and Bunke, Horst},
  booktitle={Joint IAPR international workshops on statistical techniques in pattern recognition (SPR) and structural and syntactic pattern recognition (SSPR)},
  pages={287--297},
  year={2008},
  organization={Springer}
}

@inproceedings{prenkaj2024unifying,
  title={Unifying Evolution, Explanation, and Discernment: A Generative Approach for Dynamic Graph Counterfactuals},
  author={Prenkaj, Bardh and Villaiz{\'a}n-Vallelado, Mario and Leemann, Tobias and Kasneci, Gjergji},
  booktitle={Proc. of the 30th ACM SIGKDD Conf. on Knowl. Disc. and Data Mining},
  pages={2420--2431},
  year={2024}
}

@article{doi:10.1021/ci034143r  ,
author = {Sutherland, Jeffrey J. and O'Brien, Lee A. and Weaver, Donald F.},
title = {Spline-Fitting with a Genetic Algorithm:  A Method for Developing Classification Structure-Activity Relationships},
journal = {Journal of Chemical Inf. and Computer Sciences},
volume = {43},
number = {6},
pages = {1906-1915},
year = {2003}
}

@misc{rozemberczki2020characteristicfunctionsgraphsbirds,
      title={Characteristic Functions on Graphs: Birds of a Feather, from Statistical Descriptors to Parametric Models}, 
      author={Benedek Rozemberczki and Rik Sarkar},
      year={2020},
      eprint={2005.07959},
      archivePrefix={arXiv},
      primaryClass={cs.LG},
      url={https://arxiv.org/abs/2005.07959}, 
}

@misc{narayanan2017graph2veclearningdistributedrepresentations,
      title={graph2vec: Learning Distributed Representations of Graphs}, 
      author={Annamalai Narayanan and Mahinthan Chandramohan and Rajasekar Venkatesan and Lihui Chen and Yang Liu and Shantanu Jaiswal},
      year={2017},
      eprint={1707.05005},
      archivePrefix={arXiv},
      url={https://arxiv.org/abs/1707.05005}, 
}

@inproceedings{Tsitsulin_2018,
   title={NetLSD: Hearing the Shape of a Graph},
   booktitle={Proc. of the 24th ACM SIGKDD Int. Conf. on Knowl. Discovery \& Data Mining},
   publisher={ACM},
   author={Tsitsulin, Anton and Mottin, Davide and Karras, Panagiotis and Bronstein, Alexander and Müller, Emmanuel},
   year={2018}
}

@inproceedings{Wang_2021,
  title={Graph embedding via diffusion-wavelets-based node feature distribution characterization},
  author={Wang, Lili and Huang, Chenghan and Ma, Weicheng and Cao, Xinyuan and Vosoughi, Soroush},
  booktitle={Proc. of the 30th ACM Int. Conf. on Information \& Knowl. Management},
  pages={3478--3482},
  year={2021}
}

@inproceedings{prado2023generative,
  title={Are Generative-Based Graph Counterfactual Explainers Worth It?},
  author={Prado-Romero, Mario Alfonso and Prenkaj, Bardh and Stilo, Giovanni},
  booktitle={Joint European Conf. on Machine Learning and Knowl. Disc. in Databases},
  pages={152--170},
  year={2023},
  organization={Springer}
}

@inproceedings{galland:hal-02947290,
  TITLE = {{Invariant embedding for graph classification}},
  AUTHOR = {Galland, Alexis and Lelarge, Marc},
  URL = {https://hal.science/hal-02947290},
  BOOKTITLE = {{ICML 2019 Workshop on Learning and Reasoning with Graph-Structured Data}},
  ADDRESS = {Long Beach, United States},
  YEAR = {2019},
  MONTH = Jun,
  HAL_ID = {hal-02947290},
  HAL_VERSION = {v1},
}

@misc{cai2022simpleeffectivebaselinenonattributed,
      title={A simple yet effective baseline for non-attributed graph classification}, 
      author={Chen Cai and Yusu Wang},
      year={2022},
      eprint={1811.03508},
      archivePrefix={arXiv},
      primaryClass={cs.LG},
      url={https://arxiv.org/abs/1811.03508}, 
}

@InProceedings{pmlr-v97-gao19e,
  title = 	 {Geometric Scattering for Graph Data Analysis},
  author =       {Gao, Feng and Wolf, Guy and Hirn, Matthew},
  booktitle = 	 {Proc. of the 36th Int. Conf. on Machine Learning},
  pages = 	 {2122--2131},
  year = 	 {2019},
  volume = 	 {97},
  publisher =    {PMLR}
}

@inproceedings{10.1007/978-3-030-36718-3_1,
  title={Gl2vec: Graph embedding enriched by line graphs with edge features},
  author={Chen, Hong and Koga, Hisashi},
  booktitle={Neural Information Processing: 26th Int. Conf. ICONIP 2019},
  pages={3--14},
  year={2019},
  organization={Springer}
}

@misc{delara2018simplebaselinealgorithmgraph,
      title={A Simple Baseline Algorithm for Graph Classification}, 
      author={Nathan de Lara and Edouard Pineau},
      year={2018},
      eprint={1810.09155},
      archivePrefix={arXiv},
      primaryClass={cs.LG},
      url={https://arxiv.org/abs/1810.09155}, 
}

@inproceedings{NIPS2017_d2ddea18,
 author = {Verma, Saurabh and Zhang, Zhi-Li},
 booktitle = {Advances in Neural Information Processing Systems},
 title = {Hunt For The Unique, Stable, Sparse And Fast Feature Learning On Graphs},
 volume = {30},
 year = {2017}
}

@misc{knyazev2019understandingattentiongeneralizationgraph,
      title={Understanding Attention and Generalization in Graph Neural Networks}, 
      author={Boris Knyazev and Graham W. Taylor and Mohamed R. Amer},
      year={2019},
      eprint={1905.02850},
      archivePrefix={arXiv},
      primaryClass={cs.LG},
      url={https://arxiv.org/abs/1905.02850}, 
}

@article{martins2012bayesian,
  title={A Bayesian approach to in silico blood-brain barrier penetration modeling},
  author={Martins, Ines Filipa and Teixeira, Ana L and Pinheiro, Luis and Falcao, Andre O},
  journal={Journal of chemical information and modeling},
  volume={52},
  number={6},
  pages={1686--1697},
  year={2012},
  publisher={ACS Publications}
}

@article{debnath1991structure,
  title={Structure-activity relationship of mutagenic aromatic and heteroaromatic nitro compounds. correlation with molecular orbital energies and hydrophobicity},
  author={Debnath, Asim Kumar and Lopez de Compadre, Rosa L and Debnath, Gargi and Shusterman, Alan J and Hansch, Corwin},
  journal={Journal of medicinal chemistry},
  volume={34},
  number={2},
  pages={786--797},
  year={1991},
  publisher={ACS Publications}
}

@article{bertsekas1997nonlinear,
  title={Nonlinear programming},
  author={Bertsekas, Dimitri P},
  journal={Journal of the Operational Research Society},
  volume={48},
  number={3},
  pages={334--334},
  year={1997},
  publisher={Taylor \& Francis}
}

@article{qu2024greedy,
  title={GreeDy and CoDy: Counterfactual Explainers for Dynamic Graphs},
  author={Qu, Zhan and Gomm, Daniel and F{\"a}rber, Michael},
  journal={arXiv preprint arXiv:2403.16846},
  year={2024}
}

@misc{bengio2013estimatingpropagatinggradientsstochastic,
      title={Estimating or Propagating Gradients Through Stochastic Neurons for Conditional Computation}, 
      author={Yoshua Bengio and Nicholas Léonard and Aaron Courville},
      year={2013},
      eprint={1308.3432},
      archivePrefix={arXiv},
      primaryClass={cs.LG},
      url={https://arxiv.org/abs/1308.3432}, 
}

@article{borgwardt2005protein,
  title={Protein function prediction via graph kernels},
  author={Borgwardt, Karsten M and Ong, Cheng Soon and Sch{\"o}nauer, Stefan and Vishwanathan, SVN and Smola, Alex J and Kriegel, Hans-Peter},
  journal={Bioinformatics},
  volume={21},
  number={suppl\_1},
  pages={i47--i56},
  year={2005},
  publisher={Oxford University Press}
}

@article{schomburg2004brenda,
  title={BRENDA, the enzyme database: updates and major new developments},
  author={Schomburg, Ida and Chang, Antje and Ebeling, Christian and Gremse, Marion and Heldt, Christian and Huhn, Gregor and Schomburg, Dietmar},
  journal={Nucleic acids research},
  volume={32},
  number={suppl\_1},
  pages={D431--D433},
  year={2004},
  publisher={Oxford University Press}
}

@article{dobson2003distinguishing,
  title={Distinguishing enzyme structures from non-enzymes without alignments},
  author={Dobson, Paul D and Doig, Andrew J},
  journal={Journal of molecular biology},
  volume={330},
  number={4},
  pages={771--783},
  year={2003},
  publisher={Elsevier}
}

@article{berman2000protein,
  title={The protein data bank},
  author={Berman, Helen M and Westbrook, John and Feng, Zukang and Gilliland, Gary and Bhat, Talapady N and Weissig, Helge and Shindyalov, Ilya N and Bourne, Philip E},
  journal={Nucleic acids research},
  volume={28},
  number={1},
  pages={235--242},
  year={2000},
  publisher={Oxford University Press}
}

@inproceedings{neuhaus2005graph,
  title={A graph matching based approach to fingerprint classification using directional variance},
  author={Neuhaus, Michel and Bunke, Horst},
  booktitle={International Conference on Audio-and Video-Based Biometric Person Authentication},
  pages={191--200},
  year={2005},
  organization={Springer}
}

@inproceedings{leskovec2005graphs,
  title={Graphs over time: densification laws, shrinking diameters and possible explanations},
  author={Leskovec, Jure and Kleinberg, Jon and Faloutsos, Christos},
  booktitle={Proceedings of the eleventh ACM SIGKDD international conference on Knowledge discovery in data mining},
  pages={177--187},
  year={2005}
}

@inproceedings{pan2013graph,
  title={Graph stream classification using labeled and unlabeled graphs},
  author={Pan, Shirui and Zhu, Xingquan and Zhang, Chengqi and Yu, Philip S},
  booktitle={2013 IEEE 29th International Conference on Data Engineering (ICDE)},
  pages={398--409},
  year={2013},
  organization={IEEE}
}

@inproceedings{yanardag2015deep,
  title={Deep graph kernels},
  author={Yanardag, Pinar and Vishwanathan, SVN},
  booktitle={Proceedings of the 21th ACM SIGKDD international conference on knowledge discovery and data mining},
  pages={1365--1374},
  year={2015}
}

@article{neumann2016propagation,
  title={Propagation kernels: efficient graph kernels from propagated information},
  author={Neumann, Marion and Garnett, Roman and Bauckhage, Christian and Kersting, Kristian},
  journal={Machine learning},
  volume={102},
  number={2},
  pages={209--245},
  year={2016},
  publisher={Springer}
}

@inproceedings{winn2005object,
  title={Object categorization by learned universal visual dictionary},
  author={Winn, John and Criminisi, Antonio and Minka, Thomas},
  booktitle={Tenth IEEE International Conference on Computer Vision (ICCV'05) Volume 1},
  volume={2},
  pages={1800--1807},
  year={2005},
  organization={IEEE}
}

@inproceedings{vedaldi2008quick,
  title={Quick shift and kernel methods for mode seeking},
  author={Vedaldi, Andrea and Soatto, Stefano},
  booktitle={European conference on computer vision},
  pages={705--718},
  year={2008},
  organization={Springer}
}

@article{pan2015cogboost,
  title={CogBoost: Boosting for fast cost-sensitive graph classification},
  author={Pan, Shirui and Wu, Jia and Zhu, Xingquan},
  journal={IEEE Transactions on Knowledge and Data Engineering},
  volume={27},
  number={11},
  pages={2933--2946},
  year={2015},
  publisher={IEEE}
}

@article{pan2014graph,
  title={Graph ensemble boosting for imbalanced noisy graph stream classification},
  author={Pan, Shirui and Wu, Jia and Zhu, Xingquan and Zhang, Chengqi},
  journal={IEEE transactions on cybernetics},
  volume={45},
  number={5},
  pages={954--968},
  year={2014},
  publisher={IEEE}
}

@article{verma2024induce,
  title={InduCE: Inductive counterfactual explanations for graph neural networks},
  author={Verma, Samidha and Armgaan, Burouj and Medya, Sourav and Ranu, Sayan},
  journal={Transactions on Machine Learning Research},
  year={2024}
}

@article{giorgi2025combinex,
  title={COMBINEX: A Unified Counterfactual Explainer for Graph Neural Networks via Node Feature and Structural Perturbations},
  author={Giorgi, Flavio and Silvestri, Fabrizio and Tolomei, Gabriele},
  journal={arXiv preprint arXiv:2502.10111},
  year={2025}
}

@inproceedings{ma2025c2explainer,
  title={C2Explainer: Customizable Mask-based Counterfactual Explanation for Graph Neural Networks},
  author={Ma, Jiali and Takigawa, Ichigaku and Yamamoto, Akihiro},
  booktitle={Proceedings of the 2025 ACM Conference on Fairness, Accountability, and Transparency},
  pages={137--149},
  year={2025}
}
